\documentclass{article}

\usepackage{arxiv}

\usepackage[utf8]{inputenc} 
\usepackage[T1]{fontenc}    
\usepackage[colorlinks=true]{hyperref}       
\usepackage{url}            
\usepackage{booktabs}       
\usepackage{amsfonts}       
\usepackage{nicefrac}       
\usepackage{microtype}      
\usepackage{lipsum}
\usepackage{graphicx}
\graphicspath{ {./images/} }

\usepackage{amsthm}
\newtheorem{theorem}{Theorem}
\newtheorem{lemma}{Lemma}
\newtheorem{definition}{Definition}
\newtheorem{proposition}{Proposition}

\usepackage{microtype}
\usepackage{graphicx}
\usepackage{booktabs} 
\usepackage{amsfonts,amsmath,amssymb,amsthm,boxedminipage,color,url}
\usepackage[ruled, vlined,linesnumbered]{algorithm2e} 
\usepackage{algpseudocode}
\usepackage[titletoc,title]{appendix}

\usepackage{amsmath}
\usepackage{amsthm}

\DeclareMathOperator*{\argmin}{\arg\min}

\newcommand{\secref}[1]{Section \ref{#1}}

\newcommand{\sazulay}[1]{\textcolor{cyan}{\bfseries\{SA: #1\}}}
\newcommand{\Xomit}[1]{}
\newcommand\remove[1]{}

\newcommand{\norm}[1]{\left\Vert#1\right\Vert}
\newcommand{\cL}{\mathcal{L}}
\newcommand{\rb}[1]{\left(#1\right)} 
\newcommand{\vect}[1]{\mathbf{#1}} 
\newcommand{\vw}{\mathbf{w}}

\newcommand{\R}{\mathbb{R}}

\newcommand{\bw}[0]{\mathbf{w}}
\newcommand{\commentedout}[1]{}

\title{On the Implicit Bias of Initialization Shape: \\
Beyond Infinitesimal Mirror Descent}

\author{
 Shahar Azulay \\
  Tel Aviv University \\
   \And
 Edward Moroshko \\
  Technion\\
   \And
 Mor Shpigel Nacson \\
  Technion\\
   \And
 Blake Woodworth \\
  Toyota Technological\\
  Institute at Chicago\\
   \And
 Nathan Srebro \\
  Toyota Technological\\
  Institute at Chicago\\
  \And
 Amir Globerson \\
  Tel Aviv University \\
  \And
Daniel Soudry \\
  Technion\\
}

\begin{document}
\hypersetup{%
    ,citecolor=[rgb]{0,0.08,0.45}
    ,linkcolor=[rgb]{0,0.08,0.45}
    }
\maketitle
\begin{abstract}
Recent work has highlighted the role of initialization scale in determining the 
structure of the solutions that gradient methods converge to. In particular, it was shown that large initialization leads to the neural tangent kernel regime solution, whereas small initialization leads to so called ``rich regimes''. However, the initialization structure is richer than the overall scale alone and involves relative magnitudes of different weights and layers in the network. Here we show that these relative scales, which we refer to as initialization shape, play an important role in determining the learned model. We develop a novel technique for deriving the inductive bias of gradient-flow and use it to obtain closed-form implicit regularizers for multiple cases of interest.


\end{abstract}

\section{Introduction}

Gradient descent (GD) is the main optimization tool used in deep learning. A wealth of recent work has highlighted the key role of this specific algorithm in the generalization performance of the learned model, when it is over-parameterized. Namely, the solutions that gradient descent converges to do not merely minimize the training error, but rather reflect the specific implicit biases of the optimization algorithm. 

In light of this role for GD, many works have attempted to precisely characterize the implicit bias of GD in over-parameterized models. Technically, these exact characterizations amount to identifying a function $Q(\vw)$ of the model parameters $\vw$ such that GD converges to a minimizer (or, more generally, a stationary point) of $Q(\vw)$ under the constraint of having zero training error. The form of $Q(\vw)$ can depend on various hyper-parameters (e.g., initialization, architecture, depth) and its dependence sheds light on how these hyper-parameters affect the final solution. This approach worked very well in several regimes. 

The first regime is the "Neural Tangent Kernel" (NTK) regime, which arises in networks that have an unrealistically large width \cite{du2018gradient, jacot2018neural, nguyen2021proof} or initialization scale  \cite{chizat2019lazy}. In this regime, networks converge to a linear predictor where the features are not learned, but determined by the initialization (via the so-called ``Tangent Kernel''), and in this case $Q(\vw)$ is just the RKHS norm for the linear predictor. Therefore, it is not surprising that models trained in this regime typically do not achieve state-of-the-art empirical performance in challenging datasets where deep networks perform well. Accordingly, this regime is typically considered to be less useful for explaining the success of deep learning.

The second regime is the diametrically opposed ``rich'' regime, which was analyzed specifically for classification problems with vanishing loss \cite{lyu2020gradient,chizat2020implicit}. In this regime, the parameters converge to a stationary point (or sometimes a global minimum) of the optimization problem for minimizing $Q(\vw)=||\vw||^2$ subject to margin constraints. This has been shown, under various assumptions, for linear neural networks \cite{gunasekar2018implicit, ji2019gradient} and non-linear neural networks \cite{nacson2019lexicographic, Lyu2020GradientDM, chizat2020implicit}. This regime is arguably more closely related to the performance of practical neural networks but, as \citet{Moroshko2020ImplicitBI} show, reaching  this regime requires unrealistically small loss values, even in toy problems.

Understanding the implicit bias in more realistic and practically relevant regimes remains challenging in models with more than one weight layer. Current results are restricted to very simple models such as diagonal linear neural networks with shared weights in regression \cite{Woodworth2020KernelAR} and classification \cite{Moroshko2020ImplicitBI}, as well as generalized tensor formulations of networks \cite{yun2021a}. These results show exactly how the initialization scale determines the implicit bias of the model. However, these models are quite limited. For example, when the weights in different layers are shared, we cannot understand how the relative scale between layers affects the implicit bias.

Extending these exact results to a more realistic architectures is a considerable technical challenge. In fact, recent work has provided negative results with the square loss, for ReLU networks (even with a single neuron) \cite{Vardi2020ImplicitRI} and for matrix factorization \cite{razin2020implicit, li2021towards}. Thus, finding scenarios where such a characterization of the implicit bias is possible and deriving its exact form is an open question, which we address here, making progress towards more realistic models.  



Previous work \citep{Woodworth2020KernelAR,gunasekar2018characterizing,yun2021a,Vaskevicius2019Optimal,AmidWinnowing,AmidReparameterizing} that analyzes the exact implicit bias in such scenarios mostly focuses on least squares regression. All these analyses can be shown to be equivalent to expressing the dynamics of the predictor (which is induced by gradient flow on the model parameters) as Infinitesimal Mirror Descent (IMD), where the implicit bias then follows from \citet{gunasekar2018characterizing}. This approach severely limits the model class that we can analyze because it is not always clear how to express the predictor dynamics as infinitesimal mirror descent. In fact, we can verify this is impossible to do even for basic models such as linear fully connected networks. 

\paragraph{Our Contributions:} In this work, we sidestep the above difficulty by developing a new method for characterizing the implicit bias and we apply it to obtain several new results:
\begin{itemize}
    \item We identify degrees of freedom that allow us to modify the dynamics of the model so that it can be understood as infinitesimal mirror descent, without changing its implicit bias.  In some cases, we show that this modification is equivalent to a non-linear ``time-warping'' (see Section \ref{sec:warping}).
    \item  Our approach facilitates the analysis of a strictly more general model class. This allows us to investigate the exact implicit bias for models that could not be analyzed using previous techniques. Specific examples include diagonal networks with untied weights, fully connected two-layer\footnote{By "two-layers" we mean two weight layers.} linear networks with vanishing initialization, and a two-layer single leaky ReLU neuron (see Sections \ref{section: two-layer diagonal linear networks}, \ref{sec:multi_neuron}, and \ref{sec: Non Linear Fully Connected Networks} respectively).
\end{itemize}

Our improved methodology is another step in the path toward analyzing the implicit bias in more realistic and complex models.  Also, by being able to handle models with additional complexities, it already allows us to extend the scope of phenomena we can understand, shedding light on the importance of the initialization structure to implicit bias. For example,
\begin{itemize}
    \item We show that the ratio between weights in different layers at initialization (the initialization ``shape'') has a marked effect on the learned model. We find how this property affects the final implicit bias (see Section \ref{sec:shape}).
    \item We prove that balanced initialization in diagonal linear nets improves convergence to the ``rich regime'', when the scale of the initialization vanishes (see Section \ref{sec:diagonal_shape}). 
    \item  For fully connected linear networks, we prove that vanishing initialization results in a simple $\ell_2$-norm implicit bias for the equivalent linear predictor. 
\end{itemize}
Taken together, our analysis and results show the potential of our approach for discovering new implicit biases, and the insights these can provide about the effect of initialization on learned models. 

In what follows, Sections \ref{section: two-layer diagonal linear networks}-\ref{sec:multi_neuron} present derivations of implicit biases for several models of interest, and Section \ref{sec:shape} uses these results to study the effect of initialization shape and scale on the learned models. 

\remove{
1. We are motivated by the question of how different hyperparameters affect implicit bias.

2. Previous works provide unsatisfactory answer - only for extreme (kernel, rich) regimes.

3. The works between kernel and rich regimes discuss only effect of initialization scale and only for simple diagonal linear networks.

4. We develop a new method that can be used to derive closed form implicit bias for different architectures.

5. We apply the method to diagonal (and convolutional?) and fully connected networks and calculate the implicit bias as a function of different hyperparameters (initialization scale, shape, LR).

6. Simulations demonstrate our results ?
}

\section{Preliminaries and Setup}

Given a dataset of $N$ samples $\mathbf{X} = \rb{\vect{x}^{(1)},\cdots,\vect{x}^{(N)}}\in\mathbb{R}^{d\times N}$ with $N$ corresponding scalar labels $\mathbf{y}=\rb{y^{(1)},\cdots,y^{(N)}}^\top\in\R^N$ and a parametric model $f\left(\mathbf{x}; \mathbf{\theta}\right)$ with parameters $\theta$,  we consider the problem of minimizing the square loss\footnote{The analysis in this paper can be extended to classification with the exp-loss along the lines of \citet{Moroshko2020ImplicitBI}.} 
\[
\mathcal{L}\rb{\theta}\triangleq\frac{1}{2N}\sum_{n=1}^{N}\left(y^{\left(n\right)}-f(\mathbf{x}^{(n)};\theta)\right)^2\,,
\]
using gradient descent with infinitesimally
small stepsize (i.e., gradient flow)
\[
\frac{d \theta}{dt}=-\nabla \mathcal{L}(\theta(t))~.
\]
We focus on overparameterized models, where there are many solutions that achieve zero training loss, and assume that the loss is indeed (globally) minimized by gradient flow.

\textbf{Notation} For vectors $\vect{u}, \vect{v}$, we denote by $\vect{u}\circ\vect{v}$ the element-wise multiplication. In addition, $\left\Vert \cdot\right\Vert$ is the $\ell_2$-norm.

\remove{
We follow \cite{Dutta2013ApproximateKP} in the definition of Karush-Kuhn-Tucker (KKT) conditions for non-smooth optimization problem.
\begin{definition} (KKT point)
\label{def: KKT point}
Consider the following optimization problem (P) for $\mathbf{x} \in \mathbb{R}^d$
\begin{align*}
&\min f(\mathbf{x})    \\
&\, \text{s.t.}\,\, g_n(\mathbf{x}) \leq 0\,\,\, \forall n \in [N]
\end{align*}
where $f, g_n: \mathbb{R}^d \xrightarrow{} \mathbb{R}$ are locally Lipschitz functions. We say $\mathbf{x} \in \mathbb{R}^d$ is a feasible point of (P) if $g_n(\mathbf{x}) \leq 0\,\,\, \forall n \in [N]$.

Further, a feasible point $\mathbf{x} \in \mathbb{R}^d$ is a KKT point if $\mathbf{x}$ satisfies the KKT
Current file
Overview
21
 conditions:
\begin{align*}
    \exists \lambda_1,..., \lambda_N &\geq 0 \,\,\, \text{s.t.}\\
    &1.\,\, 0 \in \partial^o f(\mathbf{x}) + \sum_{n \in [N]}\lambda_n\partial^o g_n(\mathbf{x}) \\
    &2.\,\, \forall n \in [N]: \,\, \lambda_n g_n(\mathbf{x}) = 0
\end{align*}
\end{definition}
}

\section{Background: Deriving the Implicit Bias Using Infinitesimal Mirror Descent}
\label{section: general approach for deriving implicit bias}



We begin by describing the crux of current approaches to implicit bias analysis, and in \secref{sec:warping} describe our ``warping'' approach that significantly extends these.

We focus on linear models that can be written as \[f(\mathbf{x};\theta)=\tilde{\mathbf{w}}^{\top}\mathbf{x}\,,\] where $\tilde{\mathbf{w}}=\tilde{\mathbf{w}}(\theta)$ is the equivalent linear predictor. Note that the model is linear in the input $\mathbf{x}$ but \emph{not} in the parameters $\theta$. In Section~\ref{sec: Non Linear Fully Connected Networks}, we show that our method can also be extended to non-linear models.

Our goal is to find a strictly convex function $Q(\tilde{\mathbf{w}})$ that captures the implicit regularization in the sense that the limit point of the gradient flow $\tilde{\mathbf{w}}(\infty)$ is the solution 
to the following optimization problem
\begin{align}
\tilde{\mathbf{w}}(\infty)=\argmin_{\mathbf{w}} 
 Q(\mathbf{w}) ~~~ \text{s.t.} ~~~\mathbf{X}^{\top}\mathbf{w}=\mathbf{y}~.
\label{min_q_problem}
\end{align}
We now describe a method used in \citet{Moroshko2020ImplicitBI,Woodworth2020KernelAR,gunasekar2017implicit,AmidReparameterizing} for obtaining $Q$ (below, we explain that these use essentially the same approach), and in Section \ref{sec:warping} we present our novel approach.
The KKT optimality conditions for 
Eq.~\eqref{min_q_problem}
are that there exists $\boldsymbol{\nu}\in \mathbb{R}^{N}$ such that 
\begin{align}
\nabla Q(\tilde{\mathbf{w}}(\infty)) = \mathbf{X}\boldsymbol{\nu}  ~~~ \text{and} ~~~ \mathbf{X}^{\top}\tilde{\mathbf{w}}(\infty)=\mathbf{y}~.
\label{KKT_cond}
\end{align}
Note that if $Q$ is strictly convex, Eq.~\eqref{KKT_cond} is sufficient to ensure that $\tilde{\mathbf{w}}(\infty)$ is the global minimum of Eq.~\eqref{min_q_problem}. Therefore, our goal is to find a $Q$-function and $\boldsymbol{\nu}\in \mathbb{R}^{N}$ such that the limit point of gradient flow $\tilde{\mathbf{w}}(\infty)$ satisfies \eqref{KKT_cond}. Since we assumed that gradient flow converges to a zero-loss solution, we are only concerned with the stationarity condition 
\[\nabla Q(\tilde{\mathbf{w}}(\infty)) = \mathbf{X}\boldsymbol{\nu}\,.\]
%
%
For the models we consider, the dynamics on $\tilde{\mathbf{w}}(t)$ can be written as
\begin{align}
\frac{d\tilde{\mathbf{w}}(t)}{dt}=\mathbf{H}^{-1}(\tilde{\mathbf{w}}(t))\mathbf{X}\mathbf{r}(t)
\label{w_tilde_dynamics}    
\end{align}
for some $\mathbf{r}(t)\in\mathbb{R}^N$ and ``metric tensor'' $\mathbf{H}: \mathbb{R}^d \to \mathbb{R}^{d\times d}$, which is a positive definite matrix-valued function. 
%
%
%
In this case, we can write
\begin{align}
\mathbf{H}(\tilde{\mathbf{w}}(t))\frac{d\tilde{\mathbf{w}}(t)}{dt}=\mathbf{X}\mathbf{r}(t)
\label{Hw}
\end{align}
and if $\mathbf{H}(\tilde{\mathbf{w}}(t))=\nabla^2 Q(\tilde{\mathbf{w}}(t))$ for some $Q$, we get that 
\begin{align*}
    \frac{d}{dt}(\nabla Q(\tilde{\mathbf{w}}(t)))=\mathbf{X}\mathbf{r}(t).
\end{align*}
Therefore,
\begin{align*}
\nabla Q(\tilde{\mathbf{w}}(t))-\nabla Q(\tilde{\mathbf{w}}(0))=\int_0^t\mathbf{X}\mathbf{r}(t')dt'.
\end{align*}
Denoting $\boldsymbol{\nu}=\int_0^{\infty}\mathbf{r}(t')dt'$, if $\nabla Q(\tilde{\mathbf{w}}(0))=0$ then
\begin{align*}
\nabla Q(\tilde{\mathbf{w}}(\infty))=\mathbf{X}\boldsymbol{\nu}~,
\end{align*}
which is the KKT stationarity condition. Thus, in this case, it is  possible to find the $Q$-function by solving the differential equation
\begin{align}
\mathbf{H}(\tilde{\mathbf{w}}(t))=\nabla^2 Q(\tilde{\mathbf{w}}(t)).
\label{Q_ode}
\end{align}
The aforementioned papers now proceed to solve the differential equation $\mathbf{H} = \nabla^2 Q$ for $Q$.  However, this proof strategy fundamentally relies on this differential equation having a solution, i.e., on $\mathbf{H}$ being a Hessian map. We emphasize that $\mathbf{H}$ being a Hessian map is a very special property, which does not hold for general positive definite matrix-valued functions.\footnote{Indeed, \citet{gunasekar2020mirrorless} show that the innocent-looking $\bw \mapsto I + \bw\bw^\top$ is provably not the Hessian of any function, which can be confirmed by checking the condition Eq.~\eqref{eq: Hessian-map condition}.} Indeed, Eq.~\eqref{Q_ode} only has a solution if $\mathbf{H}$ satisfies the Hessian-map condition \citep[e.g., see][]{gunasekar2020mirrorless}
\begin{align} \label{eq: Hessian-map condition}
    \forall_{i,j,k} : \frac{\partial\mathbf{H}_{i,j}(\mathbf{w})}{\partial \mathbf{w}_k}=\frac{\partial\mathbf{H}_{i,k}(\mathbf{w})}{\partial \mathbf{w}_j}~.
\end{align}
As we discuss in Section \ref{sec:multi_neuron}, this condition is not met for natural models like fully connected linear neural networks, and therefore a new approach is needed.

\subsection{Relation to Infinitesimal Mirror Descent}
The approach described above is a different presentation of the equivalent view of \citet{gunasekar2018characterizing}. They show that when the dynamics on $\tilde{\mathbf{w}}$ can be expressed as ``Infinitesimal Mirror Descent'' (IMD) with respect to a strongly convex potential $\psi$
\begin{align}
\frac{d\tilde{\mathbf{w}}(t)}{dt}=-\nabla^2\psi(\tilde{\mathbf{w}}(t))^{-1}\nabla\mathcal{L}(\tilde{\mathbf{w}}(t))~,
\label{natural}
\end{align}
then the limit point $\tilde{\mathbf{w}}(\infty)$ is described by 
\[
\tilde{\mathbf{w}}(\infty)=\argmin_{\mathbf{w}} D_{\psi}(\mathbf{w},\mathbf{w}(0)) ~~~ \text{s.t.} ~~~\mathbf{X}^{\top}\mathbf{w}=\mathbf{y}~,
\]
where $D_{\psi}(\mathbf{w},\mathbf{w}')=\psi(\mathbf{w})-\psi(\mathbf{w}')-\langle\nabla \psi(\mathbf{w}'),\mathbf{w}-\mathbf{w}'\rangle$ is the Bregman divergence associated with $\psi$. Furthermore, when $\tilde{\mathbf{w}}$ is initialized with $\nabla \psi(\tilde{\mathbf{w}}(0)) = 0$, then 
\[
\tilde{\mathbf{w}}(\infty)=\argmin_{\mathbf{w}} \psi(\mathbf{w}) ~~~ \text{s.t.} ~~~\mathbf{X}^{\top}\mathbf{w}=\mathbf{y}~.
\]
Comparing Eqs.~\eqref{w_tilde_dynamics} and \eqref{natural}, we see that the infinitesimal mirror descent view is equivalent to the approach we have described, with $\psi$ corresponding exactly to $Q$.

Although it may have been presented in different ways, these analysis techniques have formed the basis for all of the existing exact\footnote{There are some statistical (i.e. non-exact) results for matrix factorization with vanishing initialization under certain data assumptions \cite{li2018algorithmic}.} characterizations of implicit bias for linear models with square loss (outside of the NTK regime) that we are aware of (e.g. \citet{gunasekar2017implicit,Woodworth2020KernelAR,AmidWinnowing,Moroshko2020ImplicitBI}). In \secref{sec:warping} we show how to extend this analysis to cases where $\mathbf{H}$ is not a Hessian map.

\remove{For example, \citet{gunasekar2017implicit} used an analogous approach to analyze matrix factorization in the special case of commutative measurements; \citet{Woodworth2020KernelAR} used a related technique to study ``diagonal linear networks'' with tied layers (see Section \ref{section: two-layer diagonal linear networks});  \citep{AmidWinnowing} also used this approach to connect these models to the Exponentiated Gradient algorithm; and \citet{Moroshko2020ImplicitBI} extended this analysis beyond to analyze classification problems (with exponential, not square loss) showing "standard rich" regimes require unrealistically small loss. }

\section{Diagonal Linear Networks}
\label{section: two-layer diagonal linear networks}
All previous analyses of the exact implicit bias for linear models with square loss (outside of the NTK regime) are limited to cases where the different layers share weights. In this section, we will remove this assumption, which allows us to analyze the effect of the relative scales of initialization between different layers in Section \ref{sec:diagonal_shape}. To begin, we examine a two-layer ``diagonal linear network'' with untied weights
\begin{align}
f(\mathbf{x};\mathbf{u}_{+},\mathbf{u}_{-},\mathbf{v}_{+},\mathbf{v}_{-})&=\left(\mathbf{u}_{+}\circ \mathbf{v}_{+}-\mathbf{u}_{-}\circ \mathbf{v}_{-}\right)^\top\mathbf{x} 
=\tilde{\mathbf{w}}^{\top}\mathbf{x}~,
\label{diagonal_linear_model}
\end{align}
where $\tilde{\mathbf{w}}=\mathbf{u}_{+}\circ \mathbf{v}_{+}-\mathbf{u}_{-}\circ \mathbf{v}_{-}$. 

\textbf{Previous Results:}
\citet{Woodworth2020KernelAR, Moroshko2020ImplicitBI}
analyzed these models for the special case of shared weights where $\mathbf{u}_{+}=\mathbf{v}_{+}$ and $\mathbf{u}_{-}=\mathbf{v}_{-}$, corresponding to the model \[f(\mathbf{x};\mathbf{u}_{+},\mathbf{u}_{-})=\left(\mathbf{u}_{+}^2-\mathbf{u}_{-}^2\right)^\top\mathbf{x}\,.\]
Both of these works focused on unbiased initialization, i.e., $\mathbf{u}_{+}(0)=\mathbf{u}_{-}(0)=\alpha\mathbf{u}$ (for some fixed $\mathbf{u}$). In \citet{yun2021a} these results were generalized to a tensor formulation, yet one which does not allow untied weights (as in Eq. \eqref{diagonal_linear_model}).

\remove{
\textbf{Previous Results:}
\citet{Woodworth2020KernelAR, Moroshko2020ImplicitBI}
analyzed these models for the special case of shared weights where $\mathbf{u}_{+}=\mathbf{v}_{+}$ and $\mathbf{u}_{-}=\mathbf{v}_{-}$, corresponding to the model \[f(\mathbf{x};\mathbf{u}_{+},\mathbf{u}_{-})=\left(\mathbf{u}_{+}^2-\mathbf{u}_{-}^2\right)^\top\mathbf{x}\,.\]
Both of these works focused on unbiased initialization, i.e., $\mathbf{u}_{+}(0)=\mathbf{u}_{-}(0)=\alpha\mathbf{u}$ (for some fixed $\mathbf{u}$). In \citet{yun2021a} these results were generalized to a tensor formulation, yet one which does not allow untied weights (as in Eq. \eqref{diagonal_linear_model}).
}

For regression with the square loss, \citet{Woodworth2020KernelAR} showed how the scale of initialization $\alpha$ controls the limit point of gradient flow between two extreme regimes. When $\alpha$ is large, gradient flow is biased towards the minimum $\ell_2$-norm solution \cite{chizat2019lazy}, corresponding to the kernel regime; when $\alpha$ is small, gradient flow is biased towards the minimum $\ell_1$-norm solution, corresponding to the rich regime; and intermediate $\alpha$ leads to some combination of these biases. For classification with the exponential loss, \citet{Moroshko2020ImplicitBI} showed how both the scale of initialization and the optimization accuracy control the implicit bias between the NTK and rich regimes.



\textbf{Our Results:} In this work, we analyze the model \eqref{diagonal_linear_model} for the square loss and show how both the initialization scale and the initialization shape (see Section \ref{sec:diagonal_shape}) affect the implicit bias. To find the implicit bias of this model, we show how to express the training dynamics of this model in the form Eq.~\eqref{w_tilde_dynamics}, which enables the use of the IMD approach (Sec.~\ref{section: general approach for deriving implicit bias}).

To simplify the presentation, we focus on unbiased initialization, where $\mathbf{u}_{+}(0)=\mathbf{u}_{-}(0)$ and $\mathbf{v}_{+}(0)=\mathbf{v}_{-}(0)$, which allows scaling the initialization without scaling the output \cite{chizat2019lazy}. See Appendix~\ref{appendix: proof of Q for unbiased u-v model} for a more general result with any initialization.

\remove{
We examine the following unbiased model:
\[
\nu^{(n)}=\left(\mathbf{u}_{+}\circ \mathbf{v}_{+}-\mathbf{u}_{-}\circ \mathbf{v}_{-}\right)^\top\mathbf{x}^{(n)}
\]
Where representative linear model is given by:
\[
\nu^{(n)} =\mathbf{w}^\top\mathbf{x}^{(n)}=\sum_{i=1}^{d}\mathbf{x}^{(n)}_{i}{w}_{i}
\]
\begin{equation}
\label{eq: def of linear w for unbiased u-v }
\mathbf{w} = \mathbf{u}_{+}\circ \mathbf{v}_{+}-\mathbf{u}_{-}\circ \mathbf{v}_{-}    
\end{equation}

The gradient flow dynamics of parameters is given by:
\[
\frac{\partial \mathcal{L}}{\partial u_{+, i}} = - v_{+, i}(t)\left(\sum_{n=1}^{N}\mathbf{x}_i^{\left(n\right)}r^{\left(n\right)}(t)\right)
\]
\[
\frac{\partial \mathcal{L}}{\partial u_{-, i}} = + v_{-, i}(t)\left(\sum_{n=1}^{N}\mathbf{x}_i^{\left(n\right)}r^{\left(n\right)}(t)\right)
\]
\[
\frac{\partial \mathcal{L}}{\partial v_{+, i}} = - u_{+, i}(t)\left(\sum_{n=1}^{N}\mathbf{x}_i^{\left(n\right)}r^{\left(n\right)}(t)\right)
\]
\[
\frac{\partial \mathcal{L}}{\partial v_{-, i}} = + u_{-, i}(t)\left(\sum_{n=1}^{N}\mathbf{x}_i^{\left(n\right)}r^{\left(n\right)}(t)\right)
\]

We assume $u_{+,i}\left(0\right)=u_{-,i}\left(0\right),v_{+,i}\left(0\right)=v_{-,i}\left(0\right)$, which ensures unbiased initialization $(w_{i}\left(0\right)=0)$.
}

\begin{theorem}
\label{theorem: Q for unbiased u-v model}
For unbiased initialization, if the gradient flow solution $\tilde{\mathbf{w}}(\infty)$  satisfies $\mathbf{X}^\top\tilde{\mathbf{w}}(\infty) = \mathbf{y}$, then:
\[
\tilde{\mathbf{w}}(\infty) = \argmin_\mathbf{w} Q_{\boldsymbol{k}}(\mathbf{w}) \quad \mathrm{s.t.\,\,}\mathbf{X}^\top\mathbf{w} = \mathbf{y}
\]
where
\begin{align}
Q_{\boldsymbol{k}}\left(\mathbf{w}\right)=\sum_{i=1}^{d}q_{k_i}\left(w_{i}\right)~,
\label{q_func_uv}
\end{align}
\begin{align*}
q_k\left(x\right)
&=\frac{1}{2}\int_{0}^{x}\mathrm{arcsinh}\left(\frac{2z}{\sqrt{k}}\right)dz 
=\frac{\sqrt{k}}{4}\left[1-\sqrt{1+\frac{4x^{2}}{k}}+\frac{2x}{\sqrt{k}}\mathrm{arcsinh}\left(\frac{2x}{\sqrt{k}}\right)\right]
\end{align*}
and
$\sqrt{k_i}=2\left(u_{+,i}^{2}\left(0\right)+v_{+,i}^{2}\left(0\right)\right)$.
\end{theorem}
The proof appears in Appendix~\ref{appendix: proof of Q for unbiased u-v model}. 

\remove{
\begin{wrapfigure}{r}{0.25\textwidth}
\begin{center}
\vspace{-8mm}
{\includegraphics[width=\linewidth]{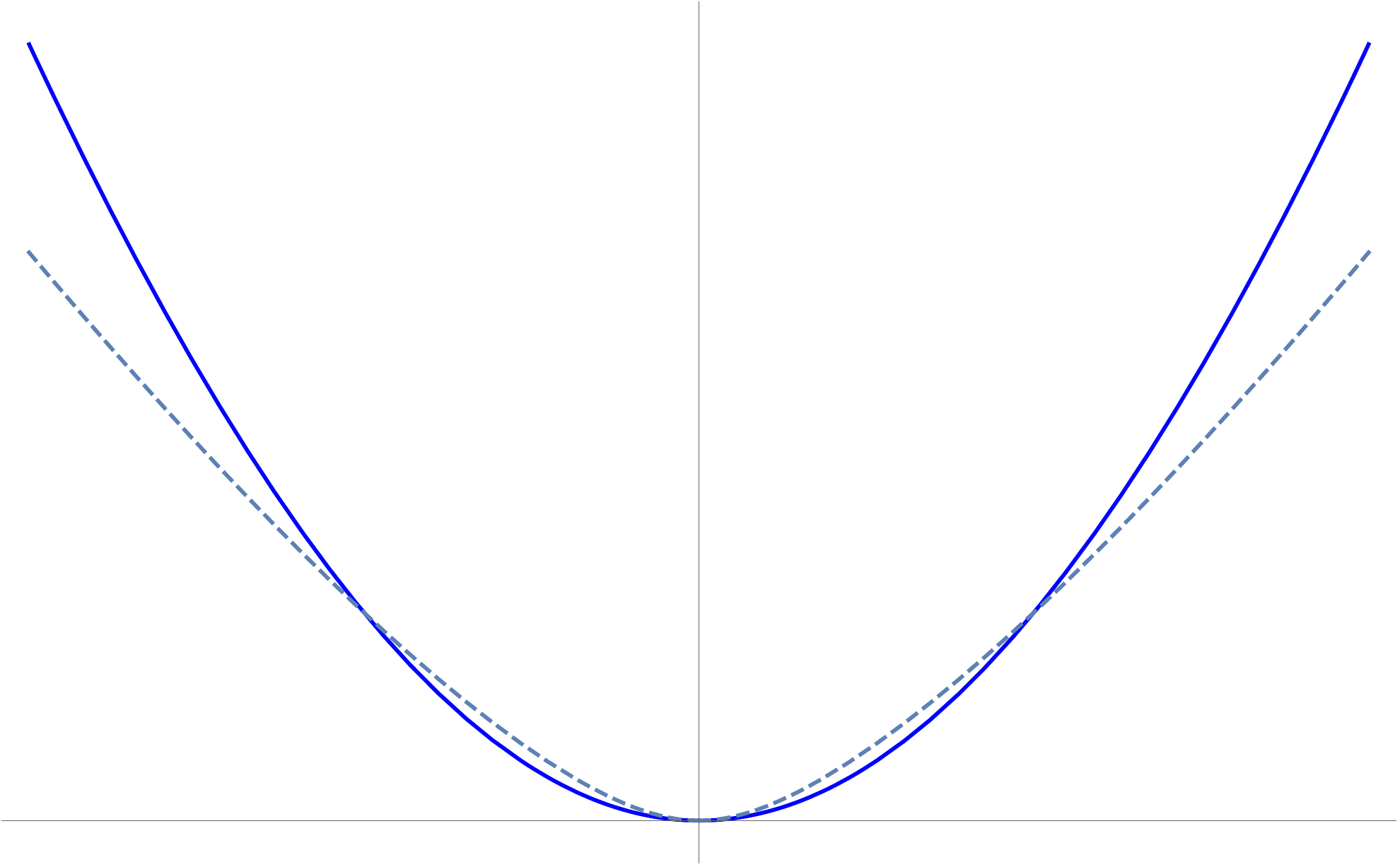}}
\small $q_k$ for $k = 0.1$ (dashed) and $k=10$ (solid).
\end{center}
\vspace{-7mm}
\end{wrapfigure}
}

The function $Q_{\boldsymbol{k}}(\mathbf{w})$ in \eqref{q_func_uv} generalizes the implicit regularizer found by \citet{Woodworth2020KernelAR}
to two layers with untied parameters. As expected, Eq.~\eqref{q_func_uv} reduces to \citet{Woodworth2020KernelAR} when $\mathbf{u}_{+}(0)=\mathbf{v}_{+}(0)$ and $\mathbf{u}_{-}(0)=\mathbf{v}_{-}(0)$. Unlike the previous result, $Q_{\boldsymbol{k}}(\mathbf{w})$ can be used to study how the relative magnitude of $\mathbf{u}$ versus $\mathbf{v}$ at initialization affects the implicit bias. We present this analysis in Section \ref{sec:diagonal_shape}, and highlight how initialization scale and shape have separate effects on the resulting model.
\remove{
The function $Q_k(\mathbf{w})$ in \eqref{q_func_uv} generalizes the implicit regularizer found by \citet{Woodworth2020KernelAR}. This is not surprising since for the special initialization case when $\mathbf{u}_{+}(0)=\mathbf{v}_{+}(0)$ and $\mathbf{u}_{-}(0)=\mathbf{v}_{-}(0)$ we obtain the squared model $f(\mathbf{x};\mathbf{u}_{+},\mathbf{u}_{-})=\left(\mathbf{u}_{+}^2-\mathbf{u}_{-}^2\right)^\top\mathbf{x}$ (see Appendix C in \citet{Woodworth2020KernelAR}). This function interpolates between the $\ell_1$-norm for $k\rightarrow 0$ and the Mahalanobis norm $\sqrt{\mathbf{w}^{\top}\mathbf{M}\mathbf{w}}$, where $\mathbf{M}^{-1}$ is the (diagonal) Gram matrix of the tangent kernel at initialization, for $k\rightarrow \infty$. We defer the full characterization of $Q_k$ to Section \ref{sec:diagonal_shape}..}

\remove{
We can notice the optimization problem described in Theorem~\ref{theorem: Q for unbiased u-v model} is a generalized version of the result of Theorem 1 from \citet{Woodworth2020KernelAR}, describing the implicit bias of a squared diagonal model - a private case of the diagonal two-layered network presented here.
}

\section{Warping Infinitesimal Mirror Descent \label{sec:warping}}
\label{sec: A new technique for deriving the implicit bias}
Our next goal is to go beyond the simplistic ``diagonal'' architecture to a fully connected one. However, deriving the implicit bias for non-diagonal models using the IMD approach (Section \ref{section: general approach for deriving implicit bias}) is not always possible since the $\mathbf{H}$ in Eq.~\eqref{w_tilde_dynamics}  might not be a Hessian map. Indeed, this condition does not hold for linear fully connected neural networks. To sidestep this issue, we next present our new technique for finding the implicit bias when $\mathbf{H}$ is not a Hessian map. 
We begin by multiplying both sides of Eq.~\eqref{Hw} by a smooth, positive function $g:\mathbb{R}^d\rightarrow(0,\infty)$ to get
\begin{align*}
g(\tilde{\mathbf{w}}(t))\mathbf{H}(\tilde{\mathbf{w}}(t))\frac{d\tilde{\mathbf{w}}(t)}{dt}=g(\tilde{\mathbf{w}}(t))\mathbf{X}\mathbf{r}(t)\,.
\end{align*}
Perhaps surprisingly, for the right choice of $g$, the differential equation $g(\bw) \mathbf{H}(\bw) = \nabla^2 Q(\bw)$ can have a solution even when $\mathbf{H}(\bw) = \nabla^2 Q(\bw)$ does not! When such a $g$ can be found, we can continue the analysis just as before,
\begin{align}
g(\tilde{\mathbf{w}}(t))\mathbf{H}(\tilde{\mathbf{w}}(t))=\nabla^2 Q(\tilde{\mathbf{w}}(t))~.
\label{gH}    
\end{align}
We see that
\begin{align}\label{eq:implicit-time-warp}
\frac{d}{dt}(\nabla Q(\tilde{\mathbf{w}}(t)))
=g(\tilde{\mathbf{w}}(t))\mathbf{X}\mathbf{r}(t)~, 
\end{align}
and we conclude
\begin{align*}
\nabla Q(\tilde{\mathbf{w}}(t))-\nabla Q(\tilde{\mathbf{w}}(0))=\int_0^tg(\tilde{\mathbf{w}}(t'))\mathbf{X}\mathbf{r}(t')dt'.
\end{align*}
We require that for our chosen $g$ function $\int_0^{\infty}g(\tilde{\mathbf{w}}(t'))\mathbf{r}(t')dt'$ exists and is finite, in which case, as before, we denote $\boldsymbol{\nu}=\int_0^{\infty}g(\tilde{\mathbf{w}}(t'))\mathbf{r}(t')dt'$
so $\tilde{\bw}(\infty)$ satisfies the stationarity condition when $\nabla Q(\tilde{\mathbf{w}}(0))=0$:
\begin{align*}
\nabla Q(\tilde{\mathbf{w}}(\infty))=\mathbf{X}\boldsymbol{\nu}~.
\end{align*}
This establishes that $Q$ captures 
the implicit bias, and all that remains is to describe how to find a $g$ such that Eq.~\eqref{gH} has a solution. For example, for a two-layer linear fully connected network with single neuron, we begin from the Ansatz that $Q\left(\mathbf{\tilde{w}}(t)\right)$ can be written as
\begin{equation}
    Q\left(\mathbf{\tilde{w}}(t)\right)=\hat{q}\left(\left\Vert \mathbf{\tilde{w}}(t)\right\Vert \right)+\mathbf{z}^{\top}\mathbf{\tilde{w}}(t) \label{eq: general q form}
\end{equation}
for some scalar function $\hat{q}$ and a fixed vector $\mathbf{z}\in\mathbb{R}^d$.

By comparing Eq.~\eqref{gH} with the Hessian of Eq. \eqref{eq: general q form}, we solve for $\hat{q}$ and $g$, and use the condition $\nabla Q(\tilde{\mathbf{w}}(0))=0$ to determine $\mathbf{z}$. For more than one neuron, the analysis becomes more complicated because we will choose different $g$ functions for each neuron.

\paragraph{The $g$ Function as a ``Time Warping''.}
The above approach can also be interpreted as a non-linear warping of the time axis. The key idea is that rescaling ``time'' for an ODE affects neither the set of points visited by the solution nor the eventual limit point. Our approach essentially finds a rescaling that yields dynamics that allow solving for $Q$.

Specifically, if $\bw(t)\in\mathbb{R}^d$ is a solution to the ODE
\begin{equation}\label{eq:generic-un-warped}
\frac{d}{dt} \bw(t) = f(\bw(t))
\end{equation}
for any ``time warping'' $\tau: \mathbb{R} \to \mathbb{R}$ such that $\tau(0) = 0$, $\lim_{t\to\infty} \tau(t) = \infty$, and $\exists c>0:  c < \tau'(t) < \infty$, then $\bw(\tau(t))$ is a solution to the ODE
\begin{equation}\label{eq:generic-warped}
\frac{d}{dt} \bw(\tau(t)) = \tau'(t)f(\bw(\tau(t)))~.
\end{equation}
Therefore, the set of points visited by $\bw(t)$ and $\bw\rb{\tau(t)}$ 
are the same, and so are their limit points $\bw(\infty) = \bw(\tau(\infty))$. All that changes is the time at which these points are reached. Furthermore, since $\tau' > 0$, $\tau$ is invertible so, conversely, a solution for Eq.~\eqref{eq:generic-warped} can also be converted into a solution for Eq.~\eqref{eq:generic-un-warped} via the warping $\tau^{-1}$. In this way, we can interpret $g$ as a time warping function which transforms the ODE
\begin{equation}\label{eq:un-warped-Q-ode}
\frac{d}{d\tau}\nabla Q(\tilde{\bw}(\tau)) = \mathbf{X}\mathbf{r}(\tau)
\end{equation}
into Eq.~\eqref{eq:implicit-time-warp}, which is equivalent in the sense that it does not affect the set of models visited by gradient flow (it only affects the time they are visited). 
In particular, let $\tilde{\bw}(\tau)$ be a solution to Eq.~\eqref{eq:un-warped-Q-ode}, then $\tilde{\bw}(\tau(t))$ is a solution for Eq.~\eqref{eq:implicit-time-warp}
for $\tau(t) = \int_0^t g(\tilde{\bw}(t')) dt'$. So long as $\tau(\infty) = \int_0^\infty g(\tilde{\bw}(t')) dt' = \infty$ so that $\tilde{\bw}(\tau(t))$ does not ``stall out,'' we conclude that the limit points of Eqs.~\eqref{eq:implicit-time-warp} and \eqref{eq:un-warped-Q-ode} are the same.

\section{Fully Connected Linear Networks}\label{sec:multi_neuron}
In this section we examine the class of fully connected linear networks of depth $2$, defined as
\[
f(\mathbf{x};\{a_i\},\{\vw_i\})	=\sum_{i=1}^m a_{i}\mathbf{w}_{i}^{\top}\mathbf{x} = \tilde{\mathbf{w}}^\top\mathbf{x}~,
\]
where $\tilde{\mathbf{w}}\triangleq \sum_{i=1}^m\tilde{\mathbf{w}}_{i}$, and $\tilde{\mathbf{w}}_{i}\triangleq a_{i}\mathbf{w}_{i}$.

\remove{
The parameter dynamics are:
\[\dot{a}_{i}=-\partial_{a_{i}}\mathcal{L}=\mathbf{w}_{i}^{\top}\left(\sum_{n=1}^{N}\mathbf{x}^{\left(n\right)}r^{\left(n\right)}\right)
\]
\begin{equation}
\dot{\mathbf{w}}_{i}=-\partial_{\mathbf{w}_{i}}\mathcal{L}=a_{i}\left(\sum_{n=1}^{N}\mathbf{x}^{\left(n\right)}r^{\left(n\right)}\right)
\label{eq: paramter dynamic fully connected}
\end{equation}

Defining
\[
\tilde{\mathbf{w}_{i}}\triangleq a_{i}\mathbf{w}_{i}
\]
And:
\[
\tilde{\mathbf{w}}(t)\triangleq \sum_{i=1}^m\tilde{\mathbf{w}_{i}}(t)
\]
We can write the linear model as:
\[
v^{\left(n\right)} = \tilde{\mathbf{w}}(t)^\top \mathbf{x}^{(n)}
\]
And the dynamics:
\[
\frac{d}{dt}\tilde{\mathbf{w}_{i}}=\dot{a}_{i}\mathbf{w}_{i}+a_{i}\dot{\mathbf{w}}_{i}=\left(a_{i}^{2}\boldsymbol{I}+\mathbf{w}_{i}\mathbf{w}_{i}^{\top}\right)\left(\sum_{n=1}^{N}\mathbf{x}^{\left(n\right)}r^{\left(n\right)}\right)
\]
}
For this model, the Hessian-map condition (Eq.~\eqref{eq: Hessian-map condition}) does not hold and thus our analysis uses the ``warped IMD'' technique described in Section \ref{sec: A new technique for deriving the implicit bias}. In addition,
our analysis of the implicit bias employs the following \textit{balancedness} properties for gradient flow shown by \citet{Du2018AlgorithmicRI}:

Theorem 2.1 of \citet{Du2018AlgorithmicRI} states that
\[
\forall t:~~a_{i}^{2}(t)-\left\Vert \mathbf{w}_{i}(t)\right\Vert ^{2}=a_{i}^{2}(0)-\left\Vert \mathbf{w}_{i}(0)\right\Vert ^{2}\triangleq \delta_{i}~.
\]

\remove{
So we can write:
\begin{align}
\label{eq: w dynamics for fully connected linear networks}
&\frac{d}{dt}\tilde{\mathbf{w}_{i}}(t)=\nonumber\\
&=\left(\left(\delta_{i}+\left\Vert \mathbf{w}_{i}(t)\right\Vert ^{2}\right)\mathbf{I}+\mathbf{w}_{i}(t)\mathbf{w}_{i}^{\top}(t)\right)\left(\sum_{n=1}^{N}\mathbf{x}^{\left(n\right)}r^{\left(n\right)}\right)
\end{align}
}

In addition, Theorem 2.2 (a stronger balancedness property for linear activations) of \citet{Du2018AlgorithmicRI} states that
\begin{align*}
\forall t:~~ &\mathbf{a}(t)\mathbf{a}(t)^T -\mathbf{W}(t)^\top\mathbf{W}(t)  =\mathbf{a}(0)\mathbf{a}(0)^T -\mathbf{W}(0)^\top\mathbf{W}(0)\triangleq \mathbf{\Delta}~,
\end{align*}
where $\mathbf{\Delta} \in \mathbb{R}^{m \times m}$, $\mathbf{a}=\rb{a_1,...,a_m}^{\top}$ and $\mathbf{W}=\rb{\mathbf{w}_1,...,\mathbf{w}_m}\in\R^{d\times m}$.

First, we derive the implicit bias for a fully connected single-neuron assuming $\delta_i \ge 0$ (which ensures that we can write the dynamics in the form \eqref{Hw} for invertible $\mathbf{H}$), and then expand our results to multi-neuron networks under more specific settings.
\begin{theorem}
\label{theorem: Q for single-neuron linear network}
For a depth $2$ fully connected network with a single hidden neuron ($m=1$), any $\delta \ge 0$, and initialization $\tilde{\vw}(0)=a(0)\vw(0)\neq\mathbf{0}$, if the gradient flow solution $\tilde{\mathbf{w}}(\infty)$  satisfies $\mathbf{X}^\top\tilde{\mathbf{w}}(\infty) = \mathbf{y}$, then:
\[
\tilde{\mathbf{w}}(\infty) = \argmin_{\mathbf{w}} q_{\delta, \tilde{\vw}(0)}(\mathbf{w}) \quad \mathrm{s.t.\,\,\,\,}\mathbf{X}^\top\mathbf{w} = \mathbf{y}
\]
where $q_{\delta, \tilde{\vw}(0)}(\mathbf{w}) =\hat{q}_\delta\left(\left\Vert \mathbf{w}\right\Vert \right)+\mathbf{z}^{\top}\mathbf{w}$
for
\[
\hat{q}_\delta(x) = \frac{\left(x^2 -\frac{\delta}{2}\left(\frac{\delta}{2} + \sqrt{x^2 + \frac{\delta^2}{4}}\right)\right)\sqrt{\sqrt{x^2 + \frac{\delta^2}{4}} - \frac{\delta}{2}}}{x}
\]
\[
\mathbf{z} = -\frac{3}{2}\sqrt{\sqrt{\left\Vert \mathbf{\tilde{w}}(0)\right\Vert ^{2}+\frac{\delta^{2}}{4}}-\frac{\delta}{2}}\frac{\mathbf{\tilde{w}}(0)}{\left\Vert \mathbf{\tilde{w}}(0)\right\Vert}~.
\]
\end{theorem}
The proof appears in Appendix~\ref{appendix: proof of Q for single-neuron linear network}.

The function $q_{\delta, \tilde{\vw}(0)}(\mathbf{w})$ above again reveals interesting tradeoffs between initialization scale and shape, which we discuss in Section~\ref{sec:full_connected_shape}.

\remove{
\begin{wrapfigure}{r}{0.25\textwidth}
\begin{center}
\vspace{-8mm}
{\includegraphics[width=\linewidth]{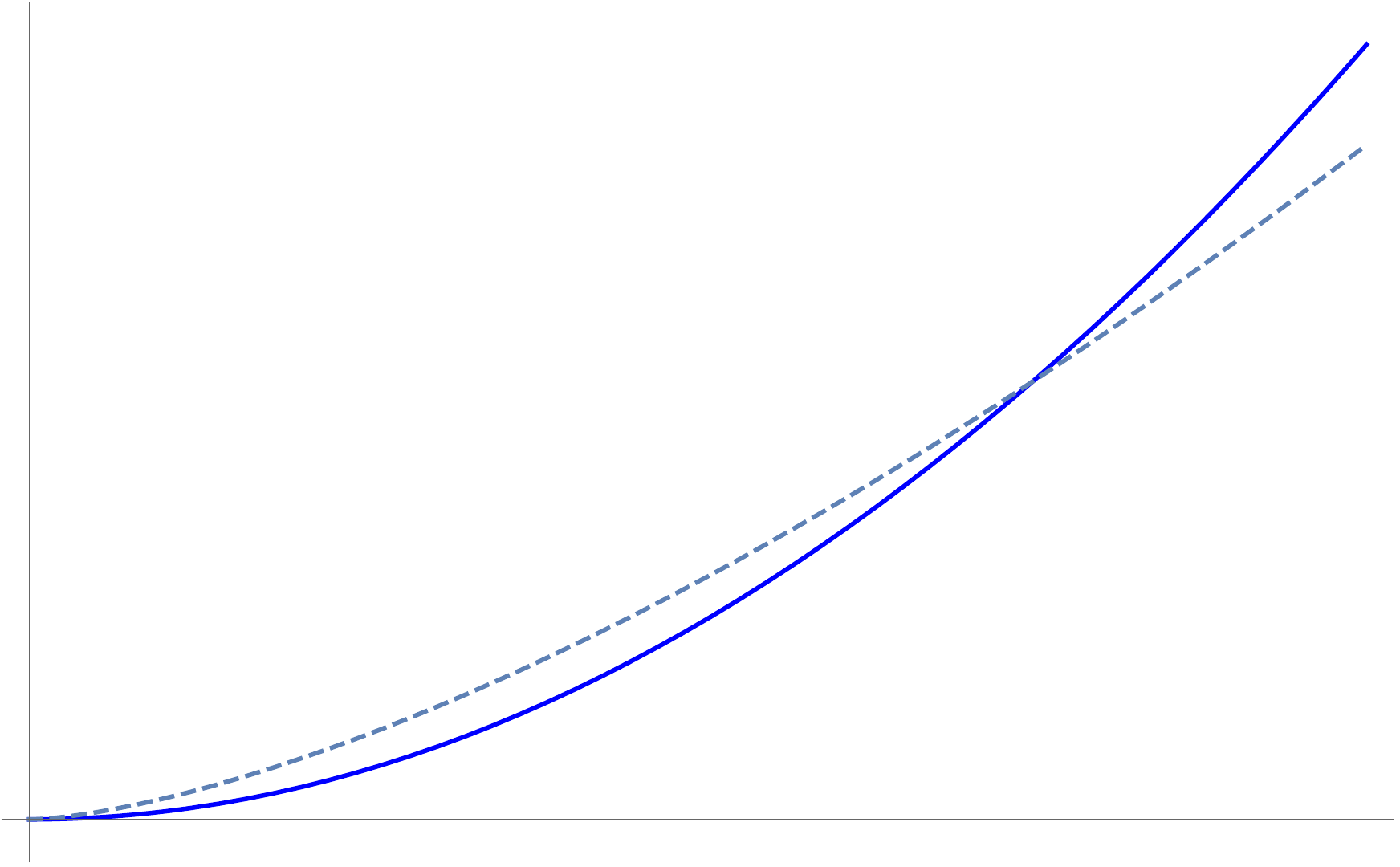}}
\small $q_\delta$ for $\delta = 0.1$ (dashed) and $\delta=10$ (solid).
\end{center}
\vspace{-7mm}
\end{wrapfigure}
}

In order to extend this result beyond a single neuron we require additional conditions to be met. For a multi-neuron network, in contrast to the single neuron case, we cannot use globally the ``time warping'' technique since it requires multiplying each neuron by a different $g$ function. However, for the special case of strictly balanced initialization, $\mathbf{\Delta} = 0$, we can extend this result to $m>1$.
\begin{proposition}
\label{corollary: Q for strictly balanced multi-neuron linear network}
For a multi-neuron network ($m>1$) with strictly balanced initialization ($\mathbf{\Delta} = 0$), assume $\tilde{\vw}(0)\neq\mathbf{0}$. If the gradient flow solution $\tilde{\mathbf{w}}(\infty)$  satisfies $\mathbf{X}^\top\tilde{\mathbf{w}}(\infty) = \mathbf{y}$, then:
\[
\tilde{\mathbf{w}}(\infty) = \argmin_{\mathbf{w}}\left[\left\Vert \mathbf{w}\right\Vert ^{3/2}-\frac{3}{2}\left\Vert \mathbf{\tilde{w}}(0)\right\Vert ^{-1/2}\mathbf{\tilde{w}}(0)^{\top}\mathbf{w}\right] \quad \mathrm{s.t.\,\,\,\,}\mathbf{X}^\top\mathbf{w} = \mathbf{y}
\]

\end{proposition}
The proof appears in Appendix~\ref{appendix: proof of Q for strictly balanced multi-neuron linear network}.

Next, we show that for infinitesimal nonzero initialization, the equivalent linear predictor of the multi-neuron linear network is biased towards the minimum $\ell_2$-norm.
\begin{theorem}
\label{theorem: Q for small initialization multi-neuron linear network}
For a multi-neuron network, and for nonzero infinitesimal initialization, i.e. $\forall i: \mathbf{0}\neq\|\tilde{\mathbf{w}}_i(0)\| \rightarrow 0$, if the gradient flow solution $\tilde{\mathbf{w}}(\infty)$  satisfies $\mathbf{X}^\top\tilde{\mathbf{w}}(\infty) = \mathbf{y}$, then:
\[
\mathbf{\tilde{w}}(\infty)=\mathrm{argmin}_{\mathbf{w}}\left\Vert \mathbf{w}\right\Vert \,\,\,\mathrm{s.t.}\,\,\mathbf{X}^\top\mathbf{w} = \mathbf{y} \,.
\]
\end{theorem}
The proof appears in Appendix~\ref{appendix: rich regime for fully connected multi-neuron network}.

Note that for infinitesimal initialization, as above, the training dynamics of fully connected linear networks is not captured by the neural tangent kernel \cite{jacot2018neural}, i.e., the tangent kernel is not fixed during training, so that we are not in the NTK regime \citep{chizat2019lazy,Woodworth2020KernelAR}. Yet, the implicit bias is towards a solution that can be captured by a kernel ($\ell_2$-norm). Though in other models, this limit coincides with the "rich" regime \citep{Woodworth2020KernelAR}, in these cases the $Q$ function is not an RKHS. Since in our case the $Q$ function is an RKHS, calling this regime "rich" is problematic. Therefore, we propose to call this vanishing initialization regime as the \textit{Anti-NTK} regime --- since this limit is diametrically opposed to the NTK regime, which is reached at the limit of infinite initialization \citep{chizat2019lazy,Woodworth2020KernelAR}. This regime coincides with the "rich" regimes in models where the $Q$ function is not an RKHS norm in that limit.

To the best of our knowledge such an  $\ell_2$ minimization result as in Theorem \ref{theorem: Q for small initialization multi-neuron linear network} was not proven for fully connected linear nets in a regression setting, even under vanishing initialization. However, for classification problems (e.g. with exponential or logistic loss) it was proven that the predictor of fully connected linear nets converges to the max-margin solution with the minimum $\ell_2$ norm \cite{ji2019gradient}, in the regime where the loss vanishes. This regime is closely related to the Anti-NTK regime since in a classification setting, vanishing loss and vanishing initialization can yield similar $Q$ function \cite{Moroshko2020ImplicitBI}.

\section{The Effect of Initialization Shape and Scale}\label{sec:shape}

\citet{chizat2019lazy} identified the scale of the initialization as the crucial parameter for entering
the NTK regime, and \citet{Woodworth2020KernelAR} further characterized the transition between the NTK and rich regimes as a function of the initialization scale, and how this affects the generalization properties of the model. Both showed the close relation between the initialization scale and the model width.

However, we identify another hyper-parameter that controls this transition between NTK and rich regimes for two-layer models, the\textit{\textbf{ shape of the initialization}}, which describes the relative scale between different layers.

We first demonstrate this by using the example of two-layer diagonal linear networks described in Section~\ref{section: two-layer diagonal linear networks}.

\begin{figure}[t]
\counterwithin{figure}{section}
\vskip 0.2in
\begin{center}
\centerline{\includegraphics[width=.75\columnwidth]{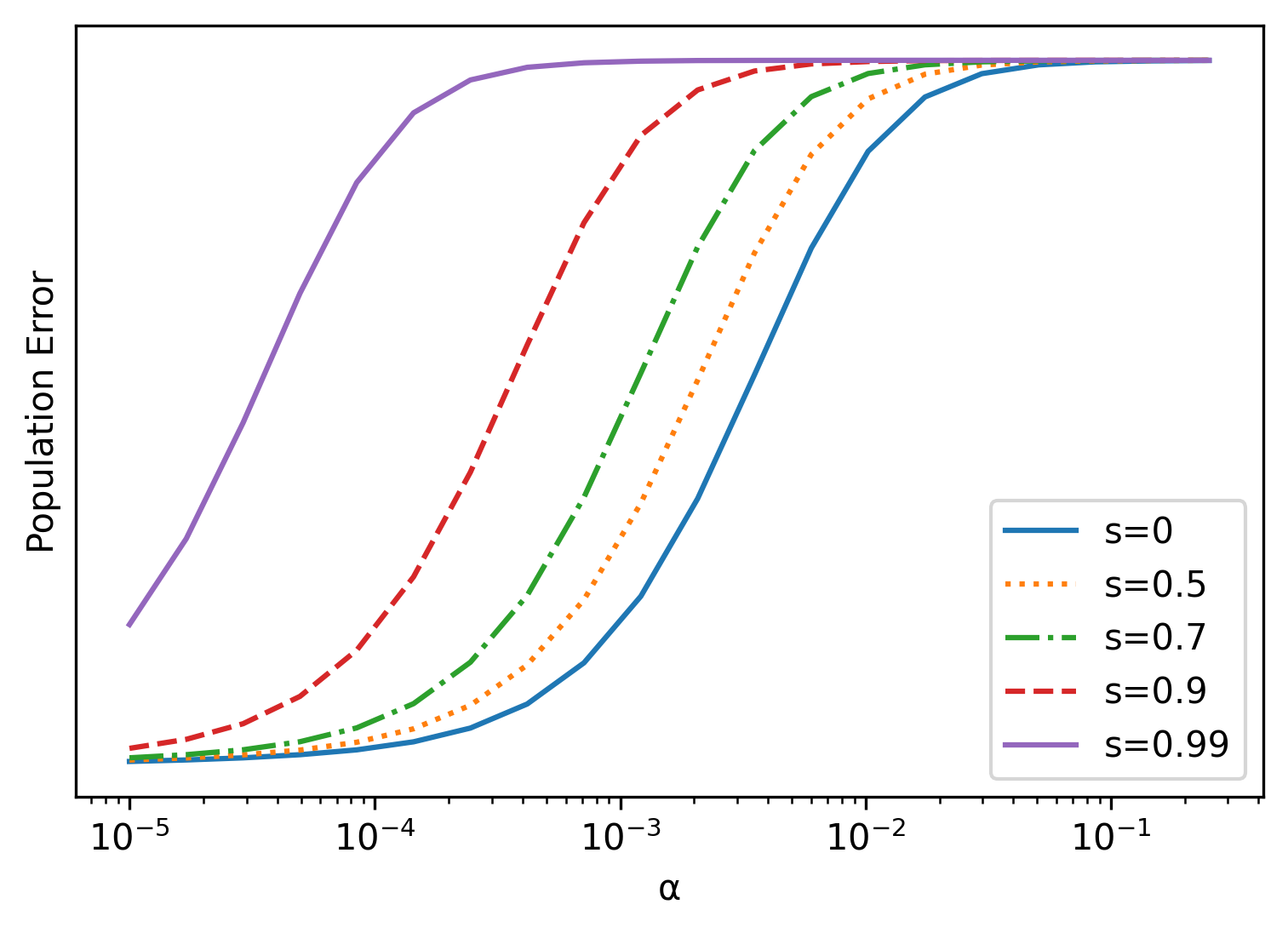}}
\caption{The population error of the gradient flow solution for a diagonal linear network as a function of initialization scale $\alpha$ and shape $s$, in the sparse regression problem described in Section~\ref{section: numerical simuations}.}\label{figure: the effect of shape in unbiased u-v model}
\end{center}
\vskip -0.2in
\end{figure}

\subsection{Diagonal Linear Networks}\label{sec:diagonal_shape}
We denote the per-neuron initialization shape $s_i$ and scale $\alpha_i$ as
\begin{align*}
s_{i}=\frac{\frac{\left|v_{+,i}\left(0\right)\right|}{\left|u_{+,i}\left(0\right)\right|}-1}{\frac{\left|v_{+,i}\left(0\right)\right|}{\left|u_{+,i}\left(0\right)\right|}+1} ~~~,~~~
\alpha_{i}=\left|u_{+,i}\left(0\right)\right|\left|v_{+,i}\left(0\right)\right|~.
\end{align*}

We can notice from Theorem~\ref{theorem: Q for unbiased u-v model} that $\sqrt{k_i}=2\left(u_{+,i}^{2}\left(0\right)+v_{+,i}^{2}\left(0\right)\right)$ controls the transition between the NTK and rich regimes.
Using the definitions of the initialization shape and scale we write \[\sqrt{k_i}  =4\alpha_{i}\frac{1+s_{i}^{2}}{1-s_{i}^{2}}\,.\]
Since $-1 < s_i < 1$, we can more accurately say that $\hat{k}_i = \frac{\alpha_i}{1 - s_i^2}$ is the factor controlling the transition.

For simplicity, we next assume that $\alpha_i = \alpha, \,\,\ s_i = s\,\,\, \forall i \in [d]$.
We can notice that for $k \rightarrow\infty$, i.e. $\frac{\alpha}{1 - s^2} \rightarrow\infty$ we get that
\begin{align*}
Q\left(w\right)=\sum_{i=1}^{d}q\left(w_{i}\right)=\sum_{i=1}^{d}\frac{1}{2\left(u_{+,i}^{2}\left(0\right)+v_{+,i}^{2}\left(0\right)\right)}w_{i}^{2}~,
\end{align*}
which is exactly the minimum RKHS norm with respect to the NTK at initialization.
Therefore, $k \rightarrow\infty$ leads to the NTK regime.
However, for $k \rightarrow 0$, i.e. $\frac{\alpha}{1 - s^2} \rightarrow 0$ we get that 
\[
Q\left(\mathbf{w}\right)=\sum_{i=1}^{d}\left|w_{i}\right|=\left\Vert \mathbf{w}\right\Vert _{1}~,
\]
which describes the rich regime.
The proof for the above two claims appears in Appendix~\ref{appendix: charachtersation of q for unbiased u-v model}.

Therefore, both the initialization scale $\alpha$ and the initialization shape $s$ affect the transition between NTK and rich regimes.
While $\alpha \rightarrow 0$ pushes to the rich regime, $|s| \rightarrow 1$ pushes towards the NTK regime.
Since both limits can take place simultaneously, the regime we will converge to in this case is captured by the joint limit
\[
\hat{k}^{\star} = \lim_{\alpha \rightarrow 0, |s| \rightarrow 1} \frac{\alpha}{1 - s^2}~.
\]
Intuitively, when $\alpha\rightarrow 0$ faster than $s\rightarrow 1$ we will be in the rich regime, corresponding to $\hat{k}^{\star}=0$. However, when $s\rightarrow 1$ faster than $\alpha\rightarrow 0$ we will be in the NTK regime, corresponding to $\hat{k}^{\star}=\infty$. For any $0<\hat{k}^{\star}<\infty$ the $Q$-function in Eq. \eqref{q_func_uv} captures the implicit bias.

Figure~\ref{figure: the effect of shape in unbiased u-v model} demonstrates the interplay between the scale and the shape of initialization. See Section \ref{section: numerical simuations} for details.
The figure shows the population error (i.e., test error) of the learned model for different choices of scale $\alpha$ and shape $s$. Since in this case the ground truth is a sparse regressor, low error corresponds to the rich regime whereas high error corresponds to the NTK regime.
It can be seen that as the shape $s$ approaches 1, the model tends to converge to a solution in the NTK regime, or an intermediate regime even for very small initialization scales. These results give further credence to the idea that the learned model will perform best when trained with balanced initialization ($s=0$).


\subsection{Fully Connected Linear Networks}\label{sec:full_connected_shape}
We begin by characterizing the effect of the initialization scale and shape for a single linear neuron with two layers, analyzed in Section \ref{sec:multi_neuron}.  Our characterization is based on the $q_{\delta,\tilde{\vw}(0)}(\mathbf{w})$ function in Theorem \ref{theorem: Q for single-neuron linear network}.
Due to the lack of space we defer the detailed analysis to Appendix \ref{appendix: charachtersation of q for fully connected single-neuron network} and provide here a summary of the results.

Similarly to the diagonal model, we again define the initialization shape parameter $s$ and scale parameter $\alpha$ as
\begin{align*}
s=\frac{\frac{\left|a\left(0\right)\right|}{\|\mathbf{w}(0)\|}-1}{\frac{\left|a\left(0\right)\right|}{\|\mathbf{w}(0)\|}+1}~~~~~,~~~~~\alpha=\left|a\left(0\right)\right|\|\mathbf{w}(0)\|~.
\end{align*}
Note that Theorem \ref{theorem: Q for single-neuron linear network} is correct for $0\leq s<1$ and any $\alpha>0$. We also employ the initialization orientation, defined as $\mathbf{u}=\frac{\vw(0)}{\|\vw(0)\|}$. Given $\alpha,s,\mathbf{u}$ we identify a few limit cases.

First, consider some fixed shape $0\leq s<1$. When $\alpha\rightarrow 0$ we will be in the Anti-NTK regime, where we obtain the minimum $\ell_2$-norm predictor. However, when $\alpha\rightarrow\infty$ we will be in the NTK regime, where the tangent kernel is fixed during training, and the implicit bias is given by the minimum RKHS norm predictor. Indeed, in this case we show in Appendix \ref{appendix: charachtersation of q for fully connected single-neuron network} that
\[
q(\tilde{\vw})\propto\left(\mathbf{\tilde{w}}-\mathbf{\tilde{w}}\left(0\right)\right)^{\top}\mathbf{B}\left(\mathbf{\tilde{w}}-\mathbf{\tilde{w}}\left(0\right)\right)~,
\]
where
\[
\mathbf{B}=\mathbf{I}-\frac{\left(1-s\right)^{2}}{2\left(1+s^{2}\right)}\mathbf{u}\mathbf{u}^{\top}
\]
and it is easy to verify that the tangent kernel is given by $K(\mathbf{x},\mathbf{x}')=\mathbf{x}^{\top}\mathbf{B}^{-1}\mathbf{x}'$.

Therefore, for any fixed shape, taking $\alpha$ from $0$ to $\infty$ we move from the Anti-NTK regime (with $\ell_2$ implicit bias) to the NTK regime where the bias is given by a Mahalanobis norm that depends on the shape and initialization orientation. Note that when $s\approx 1$, we have $\mathbf{B}\approx \mathbf{I}$, and thus we obtain the $\ell_2$ bias about the initialization, namely $\argmin_{\tilde{\vw}}\|\tilde{\vw}-\tilde{\vw}(0)\|$.
In the bottom row of Figure \ref{figure: q_delta contour for fully connected} we illustrate the $q$ function for $s=0.1$ and different values of $\alpha$.
Note that for intermediate $\alpha$ we obtain non-kernel implicit bias.

On the other hand, for any fixed scale $\alpha$, taking $s\rightarrow 1$ we will be in the NTK regime. This is because in this case the gradients of $a$ are much smaller that the gradients of $\vw$, and thus effectively,
only the $\vw$ parameters will optimize. Therefore, in this case we obtain a linear model (linear in the parameters) and the $\ell_2$ bias about the initialization, $\argmin_{\tilde{\vw}}\|\tilde{\vw}-\tilde{\vw}(0)\|$. This phenomenon is illustrated in the top row of Figure \ref{figure: q_delta contour for fully connected}.

To sum-up, in order to achieve non-kernel bias for fully connected networks we prefer balanced initialization ($s\approx 0$). This observation is in line with our observation for diagonal models in Section \ref{sec:diagonal_shape}.

\begin{figure}[t]
 \counterwithin{figure}{section}
\vskip 0.2in
\begin{center}
\counterwithin{figure}{section}
\centerline{\includegraphics[width=0.75\columnwidth]{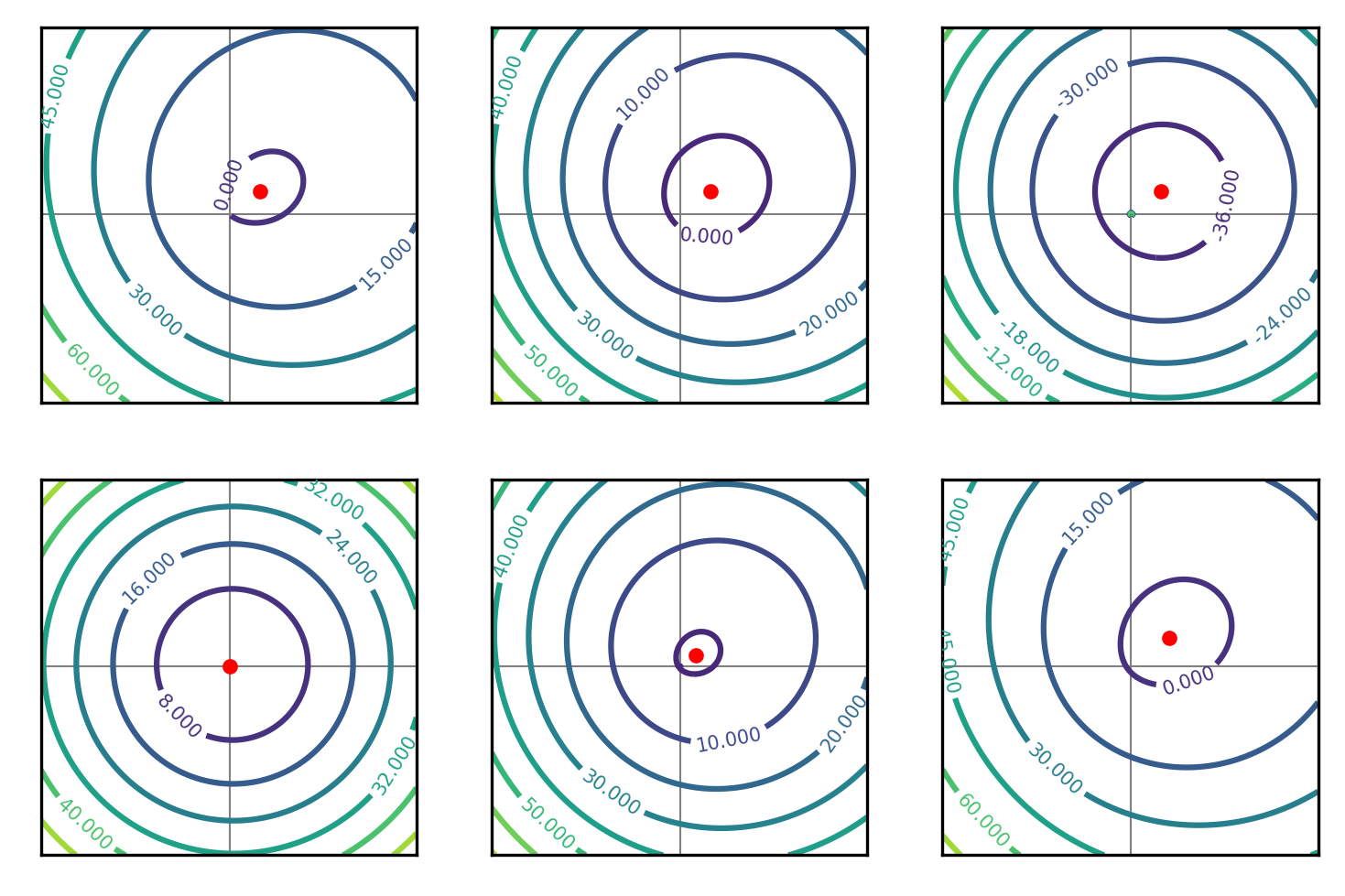}}
\caption{Contour plots of $q_{\delta,\tilde{\vw}(0)}(\tilde{\mathbf{w}})$ presented in Theorem~\ref{theorem: Q for single-neuron linear network} for the case of $d=2$, $\tilde{\mathbf{w}}(0) = \alpha \cdot [0.6, 0.8]$. Top row: $\alpha = 2$ and $s = [0, 0.2, 0.8]$ (left to right in order). Bottom row: $s=0.1$ and $\alpha = [0.01, 1, 2.5]$ (left to right in order). The red dot marks the vector $\tilde{\mathbf{w}}(0)$.
}\label{figure: q_delta contour for fully connected}
\end{center}
\vskip -0.2in
\end{figure}

\section{Two-Layer Single Leaky ReLU Neuron} \label{sec: Non Linear Fully Connected Networks}
We further extend our analysis to the class of fully connected two-layer single neuron with Leaky ReLU activations, $\sigma(x)=\max(x,\rho x)$ for $\rho>0$. This is a first step in analyzing the implicit bias of practical non-linear fully connected models for regression with the square loss.

We follow \citet{Dutta2013ApproximateKP} in the definition of KKT conditions for non-smooth optimization problem (see the definition in Appendix \ref{appendix: proof of Q for single-neuron with leaky relu activation}).

\begin{theorem}
\label{theorem: Q for single-neuron with leaky relu activation}
For a single-neuron network with Leaky ReLU activation $\sigma$ of any slope $\rho > 0$, and for any $\delta \ge 0$, assume $a(0)\vw(0)\neq\mathbf{0}$. If the gradient flow solution $(a(\infty),\mathbf{w}(\infty))$ satisfies $a(\infty)\sigma(\mathbf{X}^{\top}\mathbf{w}(\infty)) = \mathbf{y}$, then $(a(\infty),\mathbf{w}(\infty))$ satisfies the KKT conditions (according to definition~\ref{def: KKT point}) of the following optimization problem:
\[
 (a(\infty),\mathbf{w}(\infty)) = \argmin_{a,\mathbf{w}} q_{\delta}(a\mathbf{w}) \quad \mathrm{s.t.\,\,}
 a\sigma(\mathbf{X}^{\top}\mathbf{w}) = \mathbf{y}
\]
and $q_\delta(\mathbf{w})$ is identical to the definition given in Theorem~\ref{theorem: Q for single-neuron linear network}.
\end{theorem}
The proof appears in Appendix~\ref{appendix: proof of Q for single-neuron with leaky relu activation}.

Recently, \citet{Vardi2020ImplicitRI} proved a negative result for depth 2 single ReLU neuron with the square loss. They showed that it is impossible to 
characterize the implicit regularization by any explicit function of the model parameters. We note that Theorem \ref{theorem: Q for single-neuron with leaky relu activation} does not contradict the result of \citet{Vardi2020ImplicitRI} since it does not include the ReLU case ($\rho=0$).

\section{Numerical Simulations Details}
\label{section: numerical simuations}

In order to study the effect of initialization over the implicit bias of gradient flow, we follow the sparse regression problem suggested by \citet{Woodworth2020KernelAR}, where $\mathbf{x}^{(1)}, . . . , \mathbf{x}^{(N)} \sim \mathcal{N}(0, I)$ and $y^{(n)} \sim \mathcal{N}(\langle \beta^*, \mathbf{x}^{(n)}\rangle, 0.01)$ and $\beta^*$ is $r^*$-sparse, 
with non-zero entries equal to $1/\sqrt{r^*}$.
For every $N \leq d$, gradient flow will generally reach a zero training error solution, however not all of these solutions will be the same, allowing us to explore the effect of initialization over the implicit bias.

This setting was also shown by \citet{Woodworth2020KernelAR} to be tightly linked to generalization in certain settings, since the minimal $\ell_1$ solution has a sample complexity of $N = \Omega(r^*\log d)$, while the minimal $\ell_2$ solution has a much higher sample complexity of $N = \Omega(d)$. Throughout all the simulations, unless stated otherwise, we have used $N=100$, $d=1000$, $r^*=5$. 

See Figure \ref{figure: the effect of shape in unbiased u-v model} for results, and Section \ref{sec:diagonal_shape} for discussion.

\section{Conclusion}
Understanding generalization in deep learning requires understanding the implicit biases of gradient methods. Much remains to be understood about these, and even a complete understanding of linear networks is yet to be attained. Here we make progress in this direction by developing a new technique, which we apply to derive biases for diagonal and fully connected networks with independently trained layers (i.e., without shared weights). This allows us to study the effect of the initialization shape on implicit bias.

From a practical perspective it has been previously observed that balance plays an important role in initialization. For example, Xavier initialization \cite{Glorot2010UnderstandingTD} is roughly balanced by construction, and our results now provide additional theoretical support for the practical utility of this commonly used approach. We believe it is likely that further theoretical results like those presented here, can lead to improved initialization methods that lead to more effective convergence to rich regime solutions.

\section{Acknowledgements}
This research is supported by the European Research Council
(ERC) under the European Unions Horizon 2020 research
and innovation programme (grant ERC HOLI 819080) and by the Yandex Initiative in Machine Learning. 
The research of DS was supported by the Israel Science Foundation (grant No. 31/1031), by the Israel Inovation Authority (the Avatar Consortium), and by the Taub Foundation.
BW is supported by a Google Research PhD Fellowship.

\bibliography{arxiv}

\begin{thebibliography}{26}
\providecommand{\natexlab}[1]{#1}
\providecommand{\url}[1]{\texttt{#1}}
\expandafter\ifx\csname urlstyle\endcsname\relax
  \providecommand{\doi}[1]{doi: #1}\else
  \providecommand{\doi}{doi: \begingroup \urlstyle{rm}\Url}\fi

\bibitem[Amid \& Warmuth(2020{\natexlab{a}})Amid and Warmuth]{AmidWinnowing}
Amid, E. and Warmuth, M.~K.
\newblock Winnowing with gradient descent.
\newblock In \emph{Proceedings of Thirty Third Conference on Learning Theory},
  pp.\  163--182, 2020{\natexlab{a}}.

\bibitem[Amid \& Warmuth(2020{\natexlab{b}})Amid and
  Warmuth]{AmidReparameterizing}
Amid, E. and Warmuth, M. K.~K.
\newblock Reparameterizing mirror descent as gradient descent.
\newblock In \emph{Advances in Neural Information Processing Systems},
  volume~33, pp.\  8420--8429, 2020{\natexlab{b}}.

\bibitem[Chizat \& Bach(2020)Chizat and Bach]{chizat2020implicit}
Chizat, L. and Bach, F.
\newblock Implicit bias of gradient descent for wide two-layer neural networks
  trained with the logistic loss.
\newblock In \emph{Conference on Learning Theory}, pp.\  1305--1338, 2020.

\bibitem[Chizat et~al.(2019)Chizat, Oyallon, and Bach]{chizat2019lazy}
Chizat, L., Oyallon, E., and Bach, F.
\newblock On lazy training in differentiable programming.
\newblock In \emph{Advances in Neural Information Processing Systems}, pp.\
  2937--2947, 2019.

\bibitem[Du et~al.(2018)Du, Hu, and Lee]{Du2018AlgorithmicRI}
Du, S.~S., Hu, W., and Lee, J.~D.
\newblock Algorithmic regularization in learning deep homogeneous models:
  Layers are automatically balanced.
\newblock In \emph{NeurIPS}, 2018.

\bibitem[Du et~al.(2019)Du, Zhai, Poczos, and Singh]{du2018gradient}
Du, S.~S., Zhai, X., Poczos, B., and Singh, A.
\newblock Gradient descent provably optimizes over-parameterized neural
  networks.
\newblock In \emph{International Conference on Learning Representations}, 2019.

\bibitem[Dutta et~al.(2013)Dutta, Deb, Tulshyan, and
  Arora]{Dutta2013ApproximateKP}
Dutta, J., Deb, K., Tulshyan, R., and Arora, R.
\newblock Approximate {KKT} points and a proximity measure for termination.
\newblock \emph{Journal of Global Optimization}, 56:\penalty0 1463--1499, 2013.

\bibitem[Glorot \& Bengio(2010)Glorot and Bengio]{Glorot2010UnderstandingTD}
Glorot, X. and Bengio, Y.
\newblock Understanding the difficulty of training deep feedforward neural
  networks.
\newblock In \emph{AISTATS}, 2010.

\bibitem[Gunasekar et~al.(2017)Gunasekar, Woodworth, Bhojanapalli, Neyshabur,
  and Srebro]{gunasekar2017implicit}
Gunasekar, S., Woodworth, B.~E., Bhojanapalli, S., Neyshabur, B., and Srebro,
  N.
\newblock Implicit regularization in matrix factorization.
\newblock In \emph{Advances in Neural Information Processing Systems}, pp.\
  6151--6159, 2017.

\bibitem[Gunasekar et~al.(2018{\natexlab{a}})Gunasekar, Lee, Soudry, and
  Srebro]{gunasekar2018characterizing}
Gunasekar, S., Lee, J.~D., Soudry, D., and Srebro, N.
\newblock Characterizing implicit bias in terms of optimization geometry.
\newblock In \emph{International Conference on Machine Learning}, pp.\
  1827--1836, 2018{\natexlab{a}}.

\bibitem[Gunasekar et~al.(2018{\natexlab{b}})Gunasekar, Lee, Soudry, and
  Srebro]{gunasekar2018implicit}
Gunasekar, S., Lee, J.~D., Soudry, D., and Srebro, N.
\newblock Implicit bias of gradient descent on linear convolutional networks.
\newblock In \emph{Advances in Neural Information Processing Systems}, pp.\
  9461--9471, 2018{\natexlab{b}}.

\bibitem[Gunasekar et~al.(2020)Gunasekar, Woodworth, and
  Srebro]{gunasekar2020mirrorless}
Gunasekar, S., Woodworth, B., and Srebro, N.
\newblock Mirrorless mirror descent: A more natural discretization of
  riemannian gradient flow, 2020.

\bibitem[Jacot et~al.(2018)Jacot, Gabriel, and Hongler]{jacot2018neural}
Jacot, A., Gabriel, F., and Hongler, C.
\newblock Neural tangent kernel: Convergence and generalization in neural
  networks.
\newblock In \emph{Advances in Neural Information Processing Systems}, pp.\
  8571--8580, 2018.

\bibitem[Ji \& Telgarsky(2019)Ji and Telgarsky]{ji2019gradient}
Ji, Z. and Telgarsky, M.~J.
\newblock Gradient descent aligns the layers of deep linear networks.
\newblock In \emph{International Conference on Learning Representations}, 2019.

\bibitem[Li et~al.(2018)Li, Ma, and Zhang]{li2018algorithmic}
Li, Y., Ma, T., and Zhang, H.
\newblock Algorithmic regularization in over-parameterized matrix sensing and
  neural networks with quadratic activations.
\newblock In \emph{Conference On Learning Theory}, pp.\  2--47, 2018.

\bibitem[Li et~al.(2021)Li, Luo, and Lyu]{li2021towards}
Li, Z., Luo, Y., and Lyu, K.
\newblock Towards resolving the implicit bias of gradient descent for matrix
  factorization: Greedy low-rank learning.
\newblock In \emph{International Conference on Learning Representations}, 2021.

\bibitem[Lyu \& Li(2020{\natexlab{a}})Lyu and Li]{Lyu2020GradientDM}
Lyu, K. and Li, J.
\newblock Gradient descent maximizes the margin of homogeneous neural networks.
\newblock \emph{ArXiv}, abs/1906.05890, 2020{\natexlab{a}}.

\bibitem[Lyu \& Li(2020{\natexlab{b}})Lyu and Li]{lyu2020gradient}
Lyu, K. and Li, J.
\newblock Gradient descent maximizes the margin of homogeneous neural networks.
\newblock In \emph{International Conference on Learning Representations},
  2020{\natexlab{b}}.

\bibitem[Moroshko et~al.(2020)Moroshko, Woodworth, Gunasekar, Lee, Srebro, and
  Soudry]{Moroshko2020ImplicitBI}
Moroshko, E., Woodworth, B.~E., Gunasekar, S., Lee, J.~D., Srebro, N., and
  Soudry, D.
\newblock Implicit bias in deep linear classification: Initialization scale vs
  training accuracy.
\newblock In \emph{Advances in Neural Information Processing Systems},
  volume~33, pp.\  22182--22193, 2020.

\bibitem[Nacson et~al.(2019)Nacson, Gunasekar, Lee, Srebro, and
  Soudry]{nacson2019lexicographic}
Nacson, M.~S., Gunasekar, S., Lee, J., Srebro, N., and Soudry, D.
\newblock Lexicographic and depth-sensitive margins in homogeneous and
  non-homogeneous deep models.
\newblock In \emph{International Conference on Machine Learning}, pp.\
  4683--4692, 2019.

\bibitem[Nguyen(2021)]{nguyen2021proof}
Nguyen, Q.
\newblock On the proof of global convergence of gradient descent for deep relu
  networks with linear widths.
\newblock \emph{arXiv preprint arXiv:2101.09612}, 2021.

\bibitem[Razin \& Cohen(2020)Razin and Cohen]{razin2020implicit}
Razin, N. and Cohen, N.
\newblock Implicit regularization in deep learning may not be explainable by
  norms.
\newblock \emph{arXiv preprint arXiv:2005.06398}, 2020.

\bibitem[Vardi \& Shamir(2020)Vardi and Shamir]{Vardi2020ImplicitRI}
Vardi, G. and Shamir, O.
\newblock Implicit regularization in relu networks with the square loss.
\newblock \emph{ArXiv}, abs/2012.05156, 2020.

\bibitem[Vaskevicius et~al.(2019)Vaskevicius, Kanade, and
  Rebeschini]{Vaskevicius2019Optimal}
Vaskevicius, T., Kanade, V., and Rebeschini, P.
\newblock Implicit regularization for optimal sparse recovery.
\newblock In \emph{Advances in Neural Information Processing Systems},
  volume~32, pp.\  2972--2983, 2019.

\bibitem[Woodworth et~al.(2020)Woodworth, Gunasekar, Lee, Moroshko, Savarese,
  Golan, Soudry, and Srebro]{Woodworth2020KernelAR}
Woodworth, B., Gunasekar, S., Lee, J.~D., Moroshko, E., Savarese, P., Golan,
  I., Soudry, D., and Srebro, N.
\newblock Kernel and rich regimes in overparametrized models.
\newblock In \emph{Conference on Learning Theory}, pp.\  3635--3673, 2020.

\bibitem[Yun et~al.(2021)Yun, Krishnan, and Mobahi]{yun2021a}
Yun, C., Krishnan, S., and Mobahi, H.
\newblock A unifying view on implicit bias in training linear neural networks.
\newblock In \emph{International Conference on Learning Representations}, 2021.

\end{thebibliography}

\onecolumn
\appendixpage
\appendix

\section{Proof of Theorem ~\ref{theorem: Q for unbiased u-v model}}
\label{appendix: proof of Q for unbiased u-v model}
\begin{proof}
We examine a two-layer ``diagonal linear network'' with untied weights
\begin{align*}
f(\mathbf{x};\mathbf{u}_{+},\mathbf{u}_{-},\mathbf{v}_{+},\mathbf{v}_{-})&=\left(\mathbf{u}_{+}\circ \mathbf{v}_{+}-\mathbf{u}_{-}\circ \mathbf{v}_{-}\right)^\top\mathbf{x} = \tilde{\mathbf{w}}^{\top}\mathbf{x}~,
\end{align*}
where
\begin{equation}
\label{eq: def of linear w for unbiased u-v }
\tilde{\mathbf{w}} = \mathbf{u}_{+}\circ \mathbf{v}_{+}-\mathbf{u}_{-}\circ \mathbf{v}_{-}~.
\end{equation}

The gradient flow dynamics of the parameters is given by:
\[
\frac{du_{+,i}}{dt} = -\frac{\partial \mathcal{L}}{\partial u_{+, i}} = v_{+, i}(t)\left(\sum_{n=1}^{N}{x}_i^{\left(n\right)}r^{\left(n\right)}(t)\right)
\]
\[
\frac{du_{-,i}}{dt} = -\frac{\partial \mathcal{L}}{\partial u_{-, i}} = - v_{-, i}(t)\left(\sum_{n=1}^{N}{x}_i^{\left(n\right)}r^{\left(n\right)}(t)\right)
\]
\[
\frac{dv_{+,i}}{dt} = -\frac{\partial \mathcal{L}}{\partial v_{+, i}} = u_{+, i}(t)\left(\sum_{n=1}^{N}{x}_i^{\left(n\right)}r^{\left(n\right)}(t)\right)
\]
\[
\frac{dv_{-,i}}{dt} = -\frac{\partial \mathcal{L}}{\partial v_{-, i}} = - u_{-, i}(t)\left(\sum_{n=1}^{N}{x}_i^{\left(n\right)}r^{\left(n\right)}(t)\right)
\]
where we denote the residual
\[
r^{\left(n\right)}(t) \triangleq y^{(n)} - \tilde{\mathbf{w}}^{\top}(t)\mathbf{x}^{(n)}~.
\]

From Eq.~\ref{eq: def of linear w for unbiased u-v } we can write:
\begin{align*}
\frac{d\tilde{w}_{i}}{dt}&=\frac{du_{+,i}}{dt}v_{+,i}+u_{+,i}\frac{dv_{+,i}}{dt}-\frac{du_{-,i}}{dt}v_{-,i}-u_{-,i}\frac{dv_{-,i}}{dt}\\&=v_{+,i}^{2}\sum_{n=1}^{N}x_{i}^{\left(n\right)}r^{\left(n\right)}+u_{+,i}^{2}\sum_{n=1}^{N}x_{i}^{\left(n\right)}r^{\left(n\right)}+v_{-,i}^{2}\sum_{n=1}^{N}x_{i}^{\left(n\right)}r^{\left(n\right)}+u_{-,i}^{2}\sum_{n=1}^{N}x_{i}^{\left(n\right)}r^{\left(n\right)}\\&=\left(u_{+,i}^{2}+v_{+,i}^{2}+u_{-,i}^{2}+v_{-,i}^{2}\right)\sum_{n=1}^{N}x_{i}^{\left(n\right)}r^{\left(n\right)}~.
\end{align*}

Thus,
\[
\frac{1}{u_{+,i}^{2}+v_{+,i}^{2}+u_{-,i}^{2}+v_{-,i}^{2}}\frac{d\tilde{w}_{i}}{dt}=\sum_{n=1}^{N}x_{i}^{\left(n\right)}r^{\left(n\right)}~.
\]

We note that the quantity $u_{+,i}u_{-,i}+v_{+,i}v_{-,i}$ is conserved during training, since
\begin{align*}
\frac{d}{dt}\left(u_{+,i}u_{-,i}+v_{+,i}v_{-,i}\right)&=\frac{du_{+,i}}{dt}u_{-,i}+u_{+,i}\frac{du_{-,i}}{dt}+\frac{dv_{+,i}}{dt}v_{-,i}+v_{+,i}\frac{dv_{-,i}}{dt}\\&=u_{-,i}v_{+,i}\sum_{n=1}^{N}x_{i}^{\left(n\right)}r^{\left(n\right)}-u_{+,i}v_{-,i}\sum_{n=1}^{N}x_{i}^{\left(n\right)}r^{\left(n\right)}+u_{+,i}v_{-,i}\sum_{n=1}^{N}x_{i}^{\left(n\right)}r^{\left(n\right)}-v_{+,i}u_{-,i}\sum_{n=1}^{N}x_{i}^{\left(n\right)}r^{\left(n\right)}\\&=0~.
\end{align*}
So
\begin{equation}
\label{eq: training invariant unbiased u-v model}
u_{+,i}u_{-,i}+v_{+,i}v_{-,i}=u_{+,i}\left(0\right)u_{-,i}\left(0\right)+v_{+,i}\left(0\right)v_{-,i}\left(0\right)\triangleq c_i~.
\end{equation}

Combining Eq.~\ref{eq: def of linear w for unbiased u-v } and Eq.~\ref{eq: training invariant unbiased u-v model} we can write:
\[
\begin{cases}
\tilde{w}_{i}=u_{+,i}v_{+,i}-u_{-,i}v_{-,i}\\
u_{+,i}u_{-,i}+v_{+,i}v_{-,i}=c_i
\end{cases}\Rightarrow\begin{cases}
\tilde{w}_{i}^{2}=u_{+,i}^{2}v_{+,i}^{2}+u_{-,i}^{2}v_{-,i}^{2}-2u_{+,i}v_{+,i}u_{-,i}v_{-,i}\\
u_{+,i}^{2}u_{-,i}^{2}+v_{+,i}^{2}v_{-,i}^{2}+2u_{+,i}u_{-,i}v_{+,i}v_{-,i}=c_i^{2}
\end{cases}
\]
\begin{equation}
\label{eq: squared training invariant unbiased u-v model}
\Rightarrow u_{+,i}^{2}u_{-,i}^{2}+v_{+,i}^{2}v_{-,i}^{2}+u_{+,i}^{2}v_{+,i}^{2}+u_{-,i}^{2}v_{-,i}^{2}-\tilde{w}_{i}^{2}=c_i^{2}~.
\end{equation}

We also know that:
\[
v_{+,i}^{2}-u_{+,i}^{2}=v_{+,i}^{2}\left(0\right)-u_{+,i}^{2}\left(0\right)\triangleq\delta_{+,i}
\]
\[
v_{-,i}^{2}-u_{-,i}^{2}=v_{-,i}^{2}\left(0\right)-u_{-,i}^{2}\left(0\right)\triangleq\delta_{-,i}
\]
which can be easily shown since $\frac{d}{dt}\left(v_{+,i}^{2}-u_{+,i}^{2}\right)=0$ and $\frac{d}{dt}\left(v_{-,i}^{2}-u_{-,i}^{2}\right)=0$.
So using Eq.~\ref{eq: squared training invariant unbiased u-v model} we can write:
\[
u_{+,i}^{2}u_{-,i}^{2}+\left(\delta_{+,i}+u_{+,i}^{2}\right)\left(\delta_{-,i}+u_{-,i}^{2}\right)+u_{+,i}^{2}\left(\delta_{+,i}+u_{+,i}^{2}\right)+u_{-,i}^{2}\left(\delta_{-,i}+u_{-,i}^{2}\right)-\tilde{w}_{i}^{2}=c_i^{2}\\
\]
\[
\Rightarrow\left(u_{+,i}^{2}+u_{-,i}^{2}\right)^{2}+\left(\delta_{+,i}+\delta_{-,i}\right)\left(u_{+,i}^{2}+u_{-,i}^{2}\right)+\delta_{+,i}\delta_{-,i}-\tilde{w}_{i}^{2}-c_i^{2}=0
\]
\begin{align}
\label{eq: solution to u^2 + v^2 unbiased u-v model}
\Rightarrow u_{+,i}^{2}+u_{-,i}^{2}	&=\frac{-\left(\delta_{+,i}+\delta_{-,i}\right)+\sqrt{\left(\delta_{+,i}+\delta_{-,i}\right)^{2}-4\left(\delta_{+,i}\delta_{-,i}-\tilde{w}_{i}^{2}-c_i^{2}\right)}}{2} \nonumber\\
	&=\frac{-\left(\delta_{+,i}+\delta_{-,i}\right)+\sqrt{\left(\delta_{+,i}-\delta_{-,i}\right)^{2}+4c_i^{2}+4\tilde{w}_{i}^{2}}}{2}~.
\end{align}

Coming back to $u_{+,i}^{2}+v_{+,i}^{2}+u_{-,i}^{2}+v_{-,i}^{2}$ we have using Eq.~\ref{eq: solution to u^2 + v^2 unbiased u-v model} that:
\begin{align*}
u_{+,i}^{2}+v_{+,i}^{2}+u_{-,i}^{2}+v_{-,i}^{2}&=2\left(u_{+,i}^{2}+u_{-,i}^{2}\right)+\delta_{+,i}+\delta_{-,i}\\&=\sqrt{\left(\delta_{+,i}-\delta_{-,i}\right)^{2}+4c_i^{2}+4\tilde{w}_{i}^{2}}~.
\end{align*}
Therefore,
\[
\frac{1}{\sqrt{\left(\delta_{+,i}-\delta_{-,i}\right)^{2}+4c_i^{2}+4\tilde{w}_{i}^{2}}}\frac{d\tilde{w}_{i}}{dt}=\sum_{n=1}^{N}x_{i}^{\left(n\right)}r^{\left(n\right)}~.
\]

We follow the IMD approach for deriving the implicit bias (presented in detail in Section~\ref{section: general approach for deriving implicit bias} of the main paper) and try and find a function $q({\tilde w}_i)$ such that:
\begin{equation}
\nabla^{2}q\left(\tilde{w}_i(t)\right) = \frac{1}{\sqrt{\left(\delta_{+,i}-\delta_{-,i}\right)^{2}+4c_i^{2}+4\tilde{w}_{i}^{2}}}~,
\label{eq:secord_cond_uv}
\end{equation}
which will then give us that
\[
\nabla^{2}q\left(\tilde{w}_i(t)\right)\frac{d}{dt}\tilde{w}_i(t) = \sum_{n=1}^{N}x_{i}^{\left(n\right)}r^{\left(n\right)}
\]
or
\[
\frac{d}{dt}\left(\nabla q\left(\tilde{w}_i(t)\right)\right) = \sum_{n=1}^{N}x_{i}^{\left(n\right)}r^{\left(n\right)}~.
\]
Integrating the above, we get
\[
\nabla q\left(\tilde{w}_i(t)\right)-\nabla q\left(\tilde{w}_i(0)\right)=\sum_{n=1}^{N}x_{i}^{\left(n\right)}\int_{0}^{t}r^{\left(n\right)}(t')dt'~.
\]
Denoting $\nu^{(n)} = \int_{0}^{\infty}r^{\left(n\right)}(t')dt'$, and assuming $q$ also satisfies $\nabla q\left(\tilde{w}_i(0)\right) = 0$, will in turn give us the KKT stationarity condition
\[
\nabla q\left(\tilde{w}_i(\infty)\right)=\sum_{n=1}^{N}{x}_i^{\left(n\right)}\nu^{\left(n\right)}~.
\]
Namely, if we find a $q$ that satisfies the conditions above we will have that gradient flow (for each weight $\tilde{w}_i$) satisfies the KKT conditions for minimizing this $q$.

We next turn to solving for this $q$, beginning with Eq.~\ref{eq:secord_cond_uv}:
\[
q''\left(\tilde{w}_{i}\right)=\frac{1}{\sqrt{\left(\delta_{+,i}-\delta_{-,i}\right)^{2}+4c_i^{2}+4\tilde{w}_{i}^{2}}}=\frac{1}{\sqrt{k_i+4\tilde{w}_{i}^{2}}}~,
\]
where $k_i\triangleq\left(\delta_{+,i}-\delta_{-,i}\right)^{2}+4c_i^{2}$.


Integrating the above, and using the constraint $q'\left(0\right)=0$ we get:
\[
q'\left(\tilde{w}_{i}\right)=\frac{\log\left(\sqrt{4\tilde{w}_{i}^{2}+k}+2\tilde{w}_{i}\right)-\log\left(\sqrt{k}\right)}{2}~.
\]
Simplifying the above we obtain:
\[
q'\left(\tilde{w}_{i}\right)=\frac{1}{2}\log\left(\frac{\sqrt{4\tilde{w}_{i}^{2}+k_i}+2\tilde{w}_{i}}{\sqrt{k_i}}\right)=\frac{1}{2}\log\left(\sqrt{1+\frac{4\tilde{w}_{i}^{2}}{k_i}}+\frac{2\tilde{w}_{i}}{\sqrt{k_i}}\right)=\frac{1}{2}\mathrm{arcsinh}\left(\frac{2\tilde{w}_{i}}{\sqrt{k_i}}\right)~.
\]
Finally, we integrate again to obtain the desired $q$:
\[
q_{k_i}\left(\tilde{w}_{i}\right)=\frac{1}{2}\int_{0}^{\tilde{w}_{i}}\mathrm{arcsinh}\left(\frac{2z}{\sqrt{k_i}}\right)dz=\frac{\sqrt{k_i}}{4}\left[1-\sqrt{1+\frac{4\tilde{w}_{i}^{2}}{k_i}}+\frac{2\tilde{w}_{i}}{\sqrt{k_i}}\mathrm{arcsinh}\left(\frac{2\tilde{w}_{i}}{\sqrt{k_i}}\right)\right]~,
\]
where
\[
k_i	=\left(\delta_{+,i}-\delta_{-,i}\right)^{2}+4c_i^{2}
	=\left(v_{+,i}^{2}\left(0\right)-u_{+,i}^{2}\left(0\right)-v_{-,i}^{2}\left(0\right)+u_{-,i}^{2}\left(0\right)\right)^{2}+4\left(u_{+,i}\left(0\right)u_{-,i}\left(0\right)+v_{+,i}\left(0\right)v_{-,i}\left(0\right)\right)^{2}~.
\]
For the case $u_{+,i}\left(0\right)=u_{-,i}\left(0\right),v_{+,i}\left(0\right)=v_{-,i}\left(0\right)$ (unbiased initialization of $\tilde{w}_{i}\left(0\right)=0$) we get
\[
k_i=4\left(u_{+,i}^{2}\left(0\right)+v_{+,i}^{2}\left(0\right)\right)^{2}
\]
\[
\Rightarrow \sqrt{k_i}=2\left(u_{+,i}^{2}\left(0\right)+v_{+,i}^{2}\left(0\right)\right)=\frac{4\alpha_{i}\left(1+s_{i}^{2}\right)}{1-s_{i}^{2}}~.
\]

Next, if we denote $Q_{\mathbf{k}}(\tilde{\mathbf{w}}) = \sum_{i=1}^d q_{k_i}\left(\tilde{w}_{i}\right)$,
we can write
\[
\nabla Q_{\mathbf{k}}(\tilde{\mathbf{w}}(\infty)) = \left(\nabla q\left(\tilde{w}_1(\infty)\right), ..., \nabla q\left(\tilde{w}_d(\infty)\right)\right)^\top = \sum_{n=1}^{N}\mathbf{x}^{\left(n\right)}\nu^{\left(n\right)}~.
\]
Therefore, we get that gradient flow satisfies the KKT conditions for minimizing this $Q$, which completes the proof.
\end{proof}

\section{Proof of Theorem ~\ref{theorem: Q for single-neuron linear network}}
\label{appendix: proof of Q for single-neuron linear network}

\begin{proof}

We start by examining a general multi-neuron fully connected linear network of depth $2$, reducing our claim at the end to the case of a network with a single hidden neuron ($m = 1$).

The fully connected linear network of depth $2$ is defined as
\[
f(\mathbf{x};\{a_i\},\{\vw_i\})	=\sum_{i=1}^m a_{i}\mathbf{w}_{i}^{\top}\mathbf{x} = \tilde{\mathbf{w}}^\top\mathbf{x}~,
\]
where $\tilde{\mathbf{w}}\triangleq \sum_{i=1}^m\tilde{\mathbf{w}}_{i}$, and $\tilde{\mathbf{w}}_{i}\triangleq a_{i}\mathbf{w}_{i}$.

The parameter gradient flow dynamics are given by:
\[\dot{a}_{i}=-\partial_{a_{i}}\mathcal{L}=\mathbf{w}_{i}^{\top}\left(\sum_{n=1}^{N}\mathbf{x}^{\left(n\right)}r^{\left(n\right)}\right)
\]
\begin{align*}
\dot{\mathbf{w}}_{i}=-\partial_{\mathbf{w}_{i}}\mathcal{L}=a_{i}\left(\sum_{n=1}^{N}\mathbf{x}^{\left(n\right)}r^{\left(n\right)}\right)
\label{eq: paramter dynamic fully connected}
\end{align*}

\[
\frac{d}{dt}\tilde{\mathbf{w}_{i}}=\dot{a}_{i}\mathbf{w}_{i}+a_{i}\dot{\mathbf{w}}_{i}=\left(a_{i}^{2}\mathbf{I}+\mathbf{w}_{i}\mathbf{w}_{i}^{\top}\right)\left(\sum_{n=1}^{N}\mathbf{x}^{\left(n\right)}r^{\left(n\right)}\right)~,
\]

where we denote the residual
\[
r^{\left(n\right)}(t) \triangleq y^{(n)} - \tilde{\mathbf{w}}^{\top}(t)\mathbf{x}^{(n)}~.
\]

Using Theorem 2.1 of \citet{Du2018AlgorithmicRI} (stated in Section~\ref{sec:multi_neuron}), we can write
\begin{align}
\label{eq: w dynamics for fully connected linear networks}
&\frac{d}{dt}\tilde{\mathbf{w}}_{i}(t)
=\left(\left(\delta_{i}+\left\Vert \mathbf{w}_{i}(t)\right\Vert ^{2}\right)\mathbf{I}+\mathbf{w}_{i}(t)\mathbf{w}_{i}^{\top}(t)\right)\left(\sum_{n=1}^{N}\mathbf{x}^{\left(n\right)}r^{\left(n\right)}\right)~,
\end{align}
or also
\[
\left(\left(\delta_i+\left\Vert \mathbf{w}_i(t)\right\Vert ^{2}\right)\mathbf{I}+\mathbf{w}_i(t)\mathbf{w}_i^{\top}(t)\right)^{-1}\frac{d}{dt}\tilde{\mathbf{w}}_i(t)=\left(\sum_{n=1}^{N}\mathbf{x}^{\left(n\right)}r^{\left(n\right)}\right)
\]
where assuming $\delta_i \geq 0$, a non-zero initialization $\tilde{\vw}(0)=a(0)\vw(0)\neq\mathbf{0}$ and that we converge to zero-loss solution, gives us that the expression $\left(\left(\delta_i+\left\Vert \mathbf{w}_i(t)\right\Vert ^{2}\right)\mathbf{I}+\mathbf{w}_i(t)\mathbf{w}_i^{\top}(t)\right)^{-1}$ exists. 
\\
\\
Using the Sherman-Morisson Lemma, we have
\[
\left(\delta_i+\left\Vert \mathbf{w}_i(t)\right\Vert ^{2}\right)^{-1}\left(\mathbf{I}-\frac{\mathbf{w}_i(t)\mathbf{w}_i^{\top}(t)}{\left(\delta_i+2\left\Vert \mathbf{w}_i(t)\right\Vert ^{2}\right)}\right)\frac{d}{dt}\tilde{\mathbf{w}}_i=\left(\sum_{n=1}^{N}\mathbf{x}^{\left(n\right)}r^{\left(n\right)}\right)~,
\]  
or
\begin{align}
\label{eq: Hessian equation for fully connected linear networks}
\left(\delta_i+\left\Vert \mathbf{w}_i(t)\right\Vert ^{2}\right)^{-1}\left(\mathbf{I}-\frac{\mathbf{\tilde{w}}_i(t)\mathbf{\tilde{w}}_i^{\top}(t)}{\left(\delta_{i}+\left\Vert \mathbf{w}_i(t)\right\Vert ^{2}\right)\left(\delta_{i}+2\left\Vert \mathbf{w}_i(t)\right\Vert ^{2}\right)}\right)\frac{d}{dt}\tilde{\mathbf{w}}_i=\left(\sum_{n=1}^{N}\mathbf{x}^{\left(n\right)}r^{\left(n\right)}\right)
\end{align}
where we again employed Theorem 2.1 of \citet{Du2018AlgorithmicRI}.
\\
Also, since
\[
\left\Vert \mathbf{\tilde{w}}_i(t)\right\Vert ^{2}=a_i^{2}(t)\left\Vert \mathbf{w}_i(t)\right\Vert ^{2}=\left\Vert \mathbf{w}_i(t)\right\Vert ^{2}\left(\delta_i+\left\Vert \mathbf{w}_i(t)\right\Vert ^{2}\right)~,
\]
we can express $\mathbf{w}$ as a function of  $\mathbf{\tilde{w}}$:
\[
\left\Vert \mathbf{w}_i(t)\right\Vert ^{2}=\frac{-\delta_i}{2}\pm\sqrt{\frac{\delta_i^{2}}{4}+\left\Vert \mathbf{\tilde{w}}_i(t)\right\Vert ^{2}}~.
\]
Since $\left\Vert \mathbf{w}_i(t)\right\Vert ^{2}\geq0$ we choose the (+) sign and obtain
\[
\left\Vert \mathbf{w}_i(t)\right\Vert =\sqrt{\frac{-\delta_i}{2}+\sqrt{\frac{\delta_i^{2}}{4}+\left\Vert \mathbf{\tilde{w}}_i(t)\right\Vert ^{2}}}~.
\]
Therefore, we can write Eq.~\ref{eq: Hessian equation for fully connected linear networks} as:
\begin{align}
\left(\frac{\delta_i}{2}+\sqrt{\frac{\delta_i^{2}}{4}+\left\Vert \mathbf{\tilde{w}}_i(t)\right\Vert ^{2}}\right)^{-1}\left(\mathbf{I}-\frac{\mathbf{\tilde{w}}_i(t)\mathbf{\tilde{w}}_i^{\top}(t)}{2\left(\frac{\delta_i}{2}+\sqrt{\frac{\delta_i^{2}}{4}+\left\Vert \mathbf{\tilde{w}}_i(t)\right\Vert ^{2}}\right)\sqrt{\frac{\delta_i^{2}}{4}+\left\Vert \mathbf{\tilde{w}}_i(t)\right\Vert ^{2}}}\right)\frac{d}{dt}\tilde{\mathbf{w}}_i(t)=\sum_{n=1}^{N}\mathbf{x}^{\left(n\right)}r^{\left(n\right)}(t)~.
\label{eq: final Hessian equation for fully connected linear networks}
\end{align}

We follow the "warped IMD" technique for deriving the implicit bias (presented in detail in Section~\ref{sec: A new technique for deriving the implicit bias} of the main text)
and multiply Eq. \ref{eq: final Hessian equation for fully connected linear networks} by some function $g\left(\mathbf{\tilde{w}}_i(t)\right)$ 
\begin{align*}
g\left(\mathbf{\tilde{w}}_i(t)\right)&\left(\frac{\delta_i}{2}+\sqrt{\frac{\delta_i^{2}}{4}+\left\Vert \mathbf{\tilde{w}}_i(t)\right\Vert ^{2}}\right)^{-1}\left(\mathbf{I}-\frac{\mathbf{\tilde{w}}_i(t)\mathbf{\tilde{w}}_i^{\top}(t)}{2\left(\frac{\delta_i}{2}+\sqrt{\frac{\delta_i^{2}}{4}+\left\Vert \mathbf{\tilde{w}}_i(t)\right\Vert ^{2}}\right)\sqrt{\frac{\delta_i^{2}}{4}+\left\Vert \mathbf{\tilde{w}}_i(t)\right\Vert ^{2}}}\right)\frac{d}{dt}\tilde{\mathbf{w}}_i(t)\\
&=\sum_{n=1}^{N}\mathbf{x}^{\left(n\right)}g\left(\mathbf{\tilde{w}}_i(t)\right)r^{\left(n\right)}(t)~.
\end{align*}

Following the approach in Section~\ref{sec: A new technique for deriving the implicit bias}, we then try and find $q\left(\mathbf{\tilde{w}}_i(t)\right)=\hat{q}\left(\left\Vert \mathbf{\tilde{w}}_i(t)\right\Vert \right)+\mathbf{z}^{\top}\mathbf{\tilde{w}}_i(t)$ and $g\left(\mathbf{\tilde{w}}_i(t)\right)$ such that 
\begin{align}
\nabla^{2} q\left(\mathbf{\tilde{w}}_i(t)\right)=g\left(\mathbf{\tilde{w}}_i(t)\right)\left(\frac{\delta_i}{2}+\sqrt{\frac{\delta_i^{2}}{4}+\left\Vert \mathbf{\tilde{w}}_i(t)\right\Vert ^{2}}\right)^{-1}
\left( \mathbf{I} -\frac{\mathbf{\tilde{w}}_i(t)\mathbf{\tilde{w}}_i^{\top}(t)}
{2\left(\frac{\delta_i}{2}+\sqrt{\frac{\delta_{i}^{2}}{4}+\left\Vert \mathbf{\tilde{w}}_i(t)\right\Vert ^{2}}\right)\sqrt{\frac{\delta_i^{2}}{4}+\left\Vert \mathbf{\tilde{w}}_i(t)\right\Vert^{2}}}\right) ~,
\label{eq: warped Hessian equation for fully connected linear networks}
\end{align}
so that then we'll have,
\[
\nabla^{2}q\left(\mathbf{\tilde{w}}_i(t)\right)\frac{d}{dt}\tilde{\mathbf{w}}_i(t)=\sum_{n=1}^{N}\mathbf{x}^{\left(n\right)}g\left(\mathbf{\tilde{w}}_i(t)\right)r^{\left(n\right)}(t)
\]
\[
\frac{d}{dt}\left(\nabla q\left(\mathbf{\tilde{w}}_i(t)\right)\right)=\sum_{n=1}^{N}\mathbf{x}^{\left(n\right)}g\left(\mathbf{\tilde{w}}_i(t)\right)r^{\left(n\right)}(t)
\]
\[
\nabla q\left(\mathbf{\tilde{w}}_i(t)\right)-\nabla q\left(\mathbf{\tilde{w}}_i(0)\right)=\sum_{n=1}^{N}\mathbf{x}^{\left(n\right)}\int_{0}^{t}g\left(\mathbf{\tilde{w}}_i(t')\right)r^{\left(n\right)}(t')dt'~.
\]
Requiring $\nabla q\left(\mathbf{\tilde{w}}_i(0)\right) = 0$, and denoting $\nu_i^{(n)} = \int_{0}^{\infty}g\left(\mathbf{\tilde{w}}_i(t')\right)r^{\left(n\right)}(t')dt'$, we get the condition:
\[
\nabla q\left(\mathbf{\tilde{w}}_i(\infty)\right)=\sum_{n=1}^{N}\mathbf{x}^{\left(n\right)}\nu_i^{\left(n\right)}~.
\]
To find $q$ we note that:
\[
\nabla q\left(\mathbf{\tilde{w}}_i(t)\right)=\hat{q}'\left(\left\Vert \mathbf{\tilde{w}}_i(t)\right\Vert \right)\frac{\mathbf{\tilde{w}}_i(t)}{\left\Vert \mathbf{\tilde{w}}_i(t)\right\Vert }+\mathbf{z}
\]
and
\begin{align*}
\nabla^{2}q\left(\mathbf{\tilde{w}}_i(t)\right)&=\left[\hat{q}''\left(\left\Vert \mathbf{\tilde{w}}_i(t)\right\Vert \right)-\hat{q}'\left(\left\Vert \mathbf{\tilde{w}}_i(t)\right\Vert \right)\frac{1}{\left\Vert \mathbf{\tilde{w}}_i(t)\right\Vert }\right]\frac{\mathbf{\tilde{w}}_i(t)\mathbf{\tilde{w}}_i^{\top}(t)}{\left\Vert \mathbf{\tilde{w}}_i(t)\right\Vert ^{2}}+\hat{q}'\left(\left\Vert \mathbf{\tilde{w}}_i(t)\right\Vert \right)\frac{1}{\left\Vert \mathbf{\tilde{w}}_i(t)\right\Vert }\mathbf{I}\\&=\frac{\hat{q}'\left(\left\Vert \mathbf{\tilde{w}}_i(t)\right\Vert \right)}{\left\Vert \mathbf{\tilde{w}}_i(t)\right\Vert }\left[\mathbf{I}-\left[1-\left\Vert \mathbf{\tilde{w}}_i(t)\right\Vert \frac{\hat{q}''\left(\left\Vert \mathbf{\tilde{w}}_i(t)\right\Vert \right)}{\hat{q}'\left(\left\Vert \mathbf{\tilde{w}}_i(t)\right\Vert \right)}\right]\frac{\mathbf{\tilde{w}}_i(t)\mathbf{\tilde{w}}_i^{\top}(t)}{\left\Vert \mathbf{\tilde{w}}_i(t)\right\Vert ^{2}}\right]~.
\end{align*}

Comparing the form above with the Hessian in Eq.~\ref{eq: warped Hessian equation for fully connected linear networks} we require
\[
g\left(\mathbf{\tilde{w}}_i(t)\right)=\frac{\hat{q}'\left(\left\Vert \mathbf{\tilde{w}}_i(t)\right\Vert \right)}{\left\Vert \mathbf{\tilde{w}}_i(t)\right\Vert }\left(\frac{\delta_i}{2}+\sqrt{\frac{\delta_i^{2}}{4}+\left\Vert \mathbf{\tilde{w}}_i(t)\right\Vert ^{2}}\right)
\]
and
\[
\frac{1}{2\left(\frac{\delta_i}{2}+\sqrt{\frac{\delta_i^{2}}{4}+\left\Vert \mathbf{\tilde{w}}_i(t)\right\Vert ^{2}}\right)\sqrt{\frac{\delta_i^{2}}{4}+\left\Vert \mathbf{\tilde{w}}_i(t)\right\Vert ^{2}}}=\frac{1-\left\Vert \mathbf{\tilde{w}}_i(t)\right\Vert \frac{\hat{q}''\left(\left\Vert \mathbf{\tilde{w}}_i(t)\right\Vert \right)}{\hat{q}'\left(\left\Vert \mathbf{\tilde{w}}_i(t)\right\Vert \right)}}{\left\Vert \mathbf{\tilde{w}}_i(t)\right\Vert ^{2}}
\]
\[
\Rightarrow\frac{\hat{q}''\left(\left\Vert \mathbf{\tilde{w}}_i(t)\right\Vert \right)}{\hat{q}'\left(\left\Vert \mathbf{\tilde{w}}_i(t)\right\Vert \right)}=\frac{1-\frac{\left\Vert \mathbf{\tilde{w}}_i(t)\right\Vert ^{2}}{\left(\frac{\delta_i}{2}+\sqrt{\frac{\delta_i^{2}}{4}+\left\Vert \mathbf{\tilde{w}}_i(t)\right\Vert ^{2}}\right)\sqrt{\delta_i^{2}+4\left\Vert \mathbf{\tilde{w}}_i(t)\right\Vert ^{2}}}}{\left\Vert \mathbf{\tilde{w}}_i(t)\right\Vert }
\]
\[
\Rightarrow\frac{\hat{q}''\left(x\right)}{\hat{q}'\left(x\right)}=\frac{1-\frac{x^{2}}{\left(\frac{\delta_i}{2}+\sqrt{\frac{\delta_i{2}}{4}+x^{2}}\right)\sqrt{\delta_i^{2}+4x^{2}}}}{x}~.
\]

Integrating that we get
\[
\log\hat{q}'\left(x\right)=\frac{1}{2}\log\left(\sqrt{x^{2}+\frac{\delta_i^{2}}{4}}-\frac{\delta_i}{2}\right)+C
\]
\[
\Rightarrow\hat{q}'\left(x\right)=C\sqrt{\sqrt{x^{2}+\frac{\delta_i^{2}}{4}}-\frac{\delta_i}{2}}
\]
\[
\Rightarrow\hat{q}\left(x\right)=C\frac{\left(x^{2}-\frac{\delta_i}{2}\left(\frac{\delta_i}{2}+\sqrt{x^{2}+\frac{\delta_i^{2}}{4}}\right)\right)\sqrt{\sqrt{x^{2}+\frac{\delta_i^{2}}{4}}-\frac{\delta_i}{2}}}{x}+C'~.
\]
Therefore,
\[
q\left(\mathbf{\tilde{w}}_i(t)\right)=C\frac{\left(\left\Vert \mathbf{\tilde{w}}_i(t)\right\Vert ^{2}-\frac{\delta_i}{2}\left(\frac{\delta_i}{2}+\sqrt{\left\Vert \mathbf{\tilde{w}}_i(t)\right\Vert ^{2}+\frac{\delta_i^{2}}{4}}\right)\right)\sqrt{\sqrt{\left\Vert \mathbf{\tilde{w}}_i(t)\right\Vert ^{2}+\frac{\delta_i^{2}}{4}}-\frac{\delta_i}{2}}}{\left\Vert \mathbf{\tilde{w}}_i(t)\right\Vert }+\mathbf{z}^{\top}\mathbf{\tilde{w}}_i(t)+C'~.
\]
Now, from the condition $\nabla q\left(\mathbf{\tilde{w}}_i(0)\right)=0$ we have
\[
\nabla q\left(\mathbf{\tilde{w}}_i(0)\right)=\frac{3}{2}C\frac{\mathbf{\tilde{w}}_i(0)}{\left\Vert \mathbf{\tilde{w}}_i(0)\right\Vert }\sqrt{\sqrt{\left\Vert \mathbf{\tilde{w}}_i(0)\right\Vert ^{2}+\frac{\delta_i^{2}}{4}}-\frac{\delta_i}{2}}+\mathbf{z}=0
\]
\[
\Rightarrow \mathbf{z}=-\frac{3}{2}C\frac{\mathbf{\tilde{w}}_i(0)}{\left\Vert \mathbf{\tilde{w}}_i(0)\right\Vert }\sqrt{\sqrt{\left\Vert \mathbf{\tilde{w}}_i(0)\right\Vert ^{2}+\frac{\delta_i^{2}}{4}}-\frac{\delta_i}{2}}~.
\]
We can set $C=1$, $C'=0$ and get
\begin{align*}
q\left(\mathbf{\tilde{w}}_i(t)\right)=&\frac{\left(\left\Vert \mathbf{\tilde{w}}_i(t)\right\Vert ^{2}-\frac{\delta_i}{2}\left(\frac{\delta_i}{2}+\sqrt{\left\Vert \mathbf{\tilde{w}}_i(t)\right\Vert ^{2}+\frac{\delta_i^{2}}{4}}\right)\right)\sqrt{\sqrt{\left\Vert \mathbf{\tilde{w}}_i(t)\right\Vert ^{2}+\frac{\delta_i^{2}}{4}}-\frac{\delta_i}{2}}}{\left\Vert \mathbf{\tilde{w}}_i(t)\right\Vert } 
 \\
 &-\frac{3}{2}\sqrt{\sqrt{\left\Vert \mathbf{\tilde{w}}_i(0)\right\Vert ^{2}+\frac{\delta_i^{2}}{4}}-\frac{\delta_i}{2}}\frac{\mathbf{\tilde{w}}_i^{\top}(0)}{\left\Vert \mathbf{\tilde{w}}_i(0)\right\Vert }\mathbf{\tilde{w}}_i(t)~.
\end{align*}

Finally, for the case of a fully connected network with a single hidden neuron ($m = 1$), the condition 
\[
\nabla q\left(\mathbf{\tilde{w}}_i(\infty)\right)=\sum_{n=1}^{N}\mathbf{x}^{\left(n\right)}\nu_i^{\left(n\right)}
\]
can be written as
\[
\nabla q\left(\mathbf{\tilde{w}}(\infty)\right)=\sum_{n=1}^{N}\mathbf{x}^{\left(n\right)}\nu^{\left(n\right)}
\]
which since $\nu^{(n)}$ has no dependency on the index $i$ is a valid KKT stationarity condition for the $q$ we found above.
Therefore, the gradient flow satisfies the KKT conditions for minimizing the $q$ we have found.
\end{proof}

\subsection{Validation of the use of the function $g$ as a ``Time-Warping''}
First, we show that Eq.~\ref{eq: final Hessian equation for fully connected linear networks} cannot take the form suggested by Eq.~\ref{Hw} (as in the standard IMD approach described in Section~\ref{section: general approach for deriving implicit bias}):
\[
\mathbf{H}(\tilde{\mathbf{w}}(t))\frac{d\tilde{\mathbf{w}}(t)}{dt}=\mathbf{X}\mathbf{r}(t)
\]
where $\mathbf{H}(\tilde{\mathbf{w}}(t))=\nabla^2 Q(\tilde{\mathbf{w}}(t))$ for some $Q$.

From Eq.~\ref{eq: final Hessian equation for fully connected linear networks} we get that $\mathbf{H}(\mathbf{w})$ takes the form

\begin{align*}
\mathbf{H}(\mathbf{w})=\left(\frac{\delta}{2}+\sqrt{\frac{\delta^{2}}{4}+\left\Vert \mathbf{w}\right\Vert ^{2}}\right)^{-1}
\left( \mathbf{I} -\frac{\mathbf{w}\mathbf{w}^{\top}}
{2\left(\frac{\delta}{2}+\sqrt{\frac{\delta^{2}}{4}+\left\Vert \mathbf{w}\right\Vert ^{2}}\right)\sqrt{\frac{\delta^{2}}{4}+\left\Vert \mathbf{w}\right\Vert^{2}}}\right) ~.
\end{align*}

Suppose $\mathbf{H}(\mathbf{w})$ is indeed the Hessian of some $Q(\mathbf{w})$, then is must respect the Hessian-map condition (see Eq.~\ref{eq: Hessian-map condition}) for any $\delta \geq 0$.
Specifically, for $\delta = 0$ we get
\begin{align*}
\mathbf{H}(\mathbf{w})=\frac{1}{\|\mathbf{w}\|}
\left( \mathbf{I} -\frac{\mathbf{w}\mathbf{w}^{\top}}
{2\|\mathbf{w}\|^2}\right) ~,
\end{align*}
which does not satisfy the Hessian-map condition
\[
\frac{\partial\mathbf{H}_{i,i}(\mathbf{w})}{\partial \mathbf{w}_j} = -\frac{w_j}{\|\mathbf{w}\|^{3}}+\frac{3}{2}\frac{w_i^2w_j}{\|\mathbf{w}\|^{5}} \neq -\frac{w_j}{2\|\mathbf{w}\|^{3}}+\frac{3}{2}\frac{w_i^2w_j}{\|\mathbf{w}\|^{5}} = \frac{\partial\mathbf{H}_{i,j}(\mathbf{w})}{\partial \mathbf{w}_i}~.
\]
Therefore, Eq.~\ref{eq: final Hessian equation for fully connected linear networks} cannot be solved using the standard IMD approach, and requires our suggested  “warped IMD” technique (see Section~\ref{sec: A new technique for deriving the implicit bias}).

Second, we write $g\left(\mathbf{\tilde{w}}_i(t)\right)$ explicitly and show it is positive, monotone and bounded.

From Eq.~\ref{eq: final Hessian equation for fully connected linear networks} we have
\[
g\left(\mathbf{\tilde{w}}_i(t)\right)=\frac{\hat{q}'\left(\left\Vert \mathbf{\tilde{w}}_i(t)\right\Vert \right)}{\left\Vert \mathbf{\tilde{w}}_i(t)\right\Vert }\left(\frac{\delta_i}{2}+\sqrt{\frac{\delta_i^{2}}{4}+\left\Vert \mathbf{\tilde{w}}_i(t)\right\Vert ^{2}}\right) = \frac{1}{\left\Vert \mathbf{\tilde{w}}_i(t)\right\Vert }\sqrt{\sqrt{\left\Vert \mathbf{\tilde{w}}_i(t)\right\Vert ^{2}+\frac{\delta_i^{2}}{4}}-\frac{\delta_i}{2}}\left(\frac{\delta_i}{2}+\sqrt{\frac{\delta_i^{2}}{4}+\left\Vert \mathbf{\tilde{w}}_i(t)\right\Vert ^{2}}\right)~.
\]
We can see that $g\left(\mathbf{\tilde{w}}_i(t)\right) = \hat{g}\left(\|\mathbf{\tilde{w}}_i(t)\|\right)$ where 
\[
\hat{g}(x) = \frac{\sqrt{\sqrt{x^{2}+\frac{\delta_{i}^{2}}{4}}-\frac{\delta_{i}}{2}}}{x }\left(\frac{\delta_{i}}{2}+\sqrt{\frac{\delta_{i}^{2}}{4}+x^{2}}\right)~.
\]
We notice that $\hat{g}(x)$ is smooth and positive for $\forall x > 0$, and since $\lim_{x \rightarrow 0^+} \hat{g}(x) = \sqrt{\delta_i}$ (see Lemma~\ref{lemma: lim g_hat at x=0}) it is also bounded for any finite $x$.

Also, using
\[
\hat{g}'(x) = \dfrac{2\sqrt{x^2+\frac{\delta_i^2}{4}}-\delta_i}{4\sqrt{x^2+\frac{\delta_i^2}{4}}\sqrt{\sqrt{x^2+\frac{\delta_i^2}{4}}-\frac{\delta_i}{2}}}
\]
we see that $\hat{g}'(x) > 0$, $\forall x > 0$ and so $\hat{g}(x)$ is monotonically increasing.

Further, we show that the KKT condition we got using the function $g$ is valid by showing that $\nu^{(n)} = \int_{0}^{\infty}g\left(\mathbf{\tilde{w}}(t')\right)r^{\left(n\right)}(t')dt'$ is finite.

Since we constructed $q(\tilde{\mathbf{w}}(t))$ s.t. $\nabla q(\tilde{\mathbf{w}}(t))-\nabla q(\tilde{\mathbf{w}}(0))=\int_0^tg(\tilde{\mathbf{w}}(t'))\mathbf{X}\mathbf{r}(t')dt'$, we get that if the RHS is infinite at $t \rightarrow \infty$ then $\nabla q(\tilde{\mathbf{w}}(\infty))$ is infinite. However, assuming we converge to a finite weight vector $\tilde{\mathbf{w}}(\infty)$, which is correct for the square loss, we get a contradiction since $\nabla q(\tilde{\mathbf{w}})$ is bounded for any finite input.

Finally, we show that $g(\tilde{\mathbf{w}}(t))\mathbf{H}(\tilde{\mathbf{w}}(t))$ does satisfy the Hessian-map condition.
We note that this is immediate from the construction of $q$, but provide it here for completeness.

\begin{align*}
g({\mathbf{w}})\mathbf{H}(\mathbf{w})=\frac{1}{\left\Vert \mathbf{w}\right\Vert }\sqrt{\sqrt{\left\Vert \mathbf{w}\right\Vert ^{2}+\frac{\delta^{2}}{4}}-\frac{\delta}{2}}
\left( \mathbf{I} -\frac{\mathbf{w}\mathbf{w}^{\top}}
{2\left(\frac{\delta}{2}+\sqrt{\frac{\delta^{2}}{4}+\left\Vert \mathbf{w}\right\Vert ^{2}}\right)\sqrt{\frac{\delta^{2}}{4}+\left\Vert \mathbf{w}\right\Vert^{2}}}\right) ~.
\end{align*}

We denote $f(x) = \frac{1}{x}\sqrt{\sqrt{x^{2}+\frac{\delta^{2}}{4}}-\frac{\delta}{2}}$ and $h(x) = \frac{f(x)}
{2\left(\frac{\delta}{2}+\sqrt{\frac{\delta^{2}}{4}+x ^{2}}\right)\sqrt{\frac{\delta^{2}}{4}+x^{2}}} $.

Without loss of generality it is enough to observe the following settings:

\textbf{$i \neq j \neq k$:}
\[
\frac{\partial\mathbf{H}_{i,j}(\mathbf{w})}{\partial \mathbf{w}_k} =  - w_i w_j h'(\|\mathbf{w}\|)\frac{w_k}{\|\mathbf{w}\|} = - w_i w_k h'(\|\mathbf{w}\|)\frac{w_j}{\|\mathbf{w}\|} = \frac{\partial\mathbf{H}_{i,k}(\mathbf{w})}{\partial \mathbf{w}_j}
\]
\textbf{$i = j \neq k$:}
\[
\frac{\partial\mathbf{H}_{i,i}(\mathbf{w})}{\partial \mathbf{w}_k} = f'(\|\mathbf{w}\|)\frac{w_k}{\|\mathbf{w}\|} - w_i^2 h'(\|\mathbf{w}\|)\frac{w_k}{\|\mathbf{w}\|} 
\]
\[
\frac{\partial\mathbf{H}_{i,k}(\mathbf{w})}{\partial \mathbf{w}_i} = - w_k  h(\|\mathbf{w}\|) - w_i w_k h'(\|\mathbf{w}\|)\frac{w_i}{\|\mathbf{w}\|}  = - w_k  h(\|\mathbf{w}\|)- w_i^2 h'(\|\mathbf{w}\|)\frac{w_k}{\|\mathbf{w}\|}
\]
Therefore, if $\forall x\,\,, \frac{f'(x)}{x} = -h(x)$ we get that $\frac{\partial\mathbf{H}_{i,i}(\mathbf{w})}{\partial \mathbf{w}_k} = \frac{\partial\mathbf{H}_{i,k}(\mathbf{w})}{\partial \mathbf{w}_i}$.

Using the derivative of $f(x)$ we can write:
\begin{align*}
f'(x) &= \dfrac{1}{2\sqrt{x^2+\frac{\delta^2}{4}}\sqrt{\sqrt{x^2+\frac{\delta^2}{4}}-\frac{\delta}{2}}}-\dfrac{\sqrt{\sqrt{x^2+\frac{\delta^2}{4}}-\frac{\delta}{2}}}{x^2} \\
&= \dfrac{1}{2\sqrt{x^2+\frac{\delta^2}{4}}\sqrt{\sqrt{x^2+\frac{\delta^2}{4}}-\frac{\delta}{2}}}-\dfrac{\sqrt{\sqrt{x^2+\frac{\delta^2}{4}}-\frac{\delta}{2}}}{\left(\sqrt{x^2+\frac{\delta^2}{4}}-\frac{\delta}{2}\right)\left(\sqrt{x^2+\frac{\delta^2}{4}}+\frac{\delta}{2}\right)}\\
&= \dfrac{1}{2\sqrt{x^2+\frac{\delta^2}{4}}\sqrt{\sqrt{x^2+\frac{\delta^2}{4}}-\frac{\delta}{2}}}-\dfrac{1}{\sqrt{\sqrt{x^2+\frac{\delta^2}{4}}-\frac{\delta}{2}}\left(\sqrt{x^2+\frac{\delta^2}{4}}+\frac{\delta}{2}\right)}\\
&= \dfrac{\left(\sqrt{x^2+\frac{\delta^2}{4}}+\frac{\delta}{2}\right) -2\sqrt{x^2+\frac{\delta^2}{4}}}{2\sqrt{x^2+\frac{\delta^2}{4}}\sqrt{\sqrt{x^2+\frac{\delta^2}{4}}-\frac{\delta}{2}}\left(\sqrt{x^2+\frac{\delta^2}{4}}+\frac{\delta}{2}\right)}\\
&= -\dfrac{\sqrt{\sqrt{x^2+\frac{\delta^2}{4}}-\frac{\delta}{2}}}{2\sqrt{x^2+\frac{\delta^2}{4}}\left(\sqrt{x^2+\frac{\delta^2}{4}}+\frac{\delta}{2}\right)}\\
& = - x\cdot h(x)~,
\end{align*}
and so $g({\mathbf{w}})\mathbf{H}(\mathbf{w})$ respects the Hessian-map condition.

\section{Proof of Proposition ~\ref{corollary: Q for strictly balanced multi-neuron linear network}}
\label{appendix: proof of Q for strictly balanced multi-neuron linear network}
\begin{proof}
We recall that the fully connected linear network of depth $2$ is defined as
\[
f(\mathbf{x};\{a_i\},\{\vw_i\})	=\sum_{i=1}^m a_{i}\mathbf{w}_{i}^{\top}\mathbf{x} = \tilde{\mathbf{w}}^\top\mathbf{x}~,
\]
where $\tilde{\mathbf{w}}\triangleq \sum_{i=1}^m\tilde{\mathbf{w}}_{i}$, and $\tilde{\mathbf{w}}_{i}\triangleq a_{i}\mathbf{w}_{i}$.

Returning to the dynamics of model parameters (Eq.~\ref{eq: w dynamics for fully connected linear networks}) we have
\[
\frac{d}{dt}\tilde{\mathbf{w}}_{i}(t)=\dot{a}_{i}\mathbf{w}_{i}+a_{i}\dot{\mathbf{w}}_{i}=\left(a_{i}^{2}\mathbf{I}+\mathbf{w}_{i}\mathbf{w}_{i}^{\top}\right)\left(\sum_{n=1}^{N}\mathbf{x}^{\left(n\right)}r^{\left(n\right)}\right)~.
\]
Therefore,
\[
\frac{d}{dt}\tilde{\mathbf{w}}(t)=\left(\sum_{i=1}^ma_{i}^{2}\mathbf{I}+\sum_{i=1}^m\mathbf{w}_{i}\mathbf{w}_{i}^{\top}\right)\left(\sum_{n=1}^{N}\mathbf{x}^{\left(n\right)}r^{\left(n\right)}\right)
\]
\[
\left(\sum_{i=1}^m a_i^2(t)\mathbf{I} + \sum_{i=1}^m\mathbf{w}_i(t) \mathbf{w}_i(t)^\top\right)^{-1}\frac{d}{dt}\tilde{\mathbf{w}}(t)=\left(\sum_{n=1}^{N}\mathbf{x}^{\left(n\right)}r^{\left(n\right)}\right)~.
\]

We can notice that we can express
\[
\sum_{i=1}^m a_i^2(t)\mathbf{I} + \sum_{i=1}^m\mathbf{w}_i(t) \mathbf{w}_i(t)^\top = \mathbf{A}(t) + \mathbf{U}(t)\mathbf{C}\mathbf{V}(t)
\]
where
\[
\mathbf{A}(t) = \left(\sum_{i=1}^m a_i^2(t)\right) \mathbf{I}_{d \times d}
\]
\[
\mathbf{C} = \mathbf{I}_{m \times m}
\]
\[
\mathbf{U}(t) = \mathbf{W}(t) \triangleq [\mathbf{w}_1(t),..., \mathbf{w}_m(t)] \in \mathbb{R}^{d \times m}
\]
\[
\mathbf{V}(t) = \mathbf{W}(t)^\top = [\mathbf{w}_1(t)^\top;...; \mathbf{w}_m(t)^\top] \in \mathbb{R}^{m \times d}~.
\]
By using the Woodbury matrix identity we can write
\begin{align*}
\left(\sum_{i=1}^m a_i^2(t)\mathbf{I} + \sum_{i=1}^m\mathbf{w}_i(t) \mathbf{w}_i(t)^\top\right)^{-1} &= \mathbf{A}^{-1} - \mathbf{A}^{-1}\mathbf{U}\left(\mathbf{I} +\mathbf{V}\mathbf{A}^{-1}\mathbf{U}\right)^{-1}\mathbf{V}\mathbf{A}^{-1} =\\
&= \mathbf{A}^{-1} - \mathbf{A}^{-1}\mathbf{U}\left(\left(\sum_{i=1}^m a_i^2(t)\right) \mathbf{I} + \mathbf{V}\mathbf{U}\right)^{-1}\mathbf{V}~.
\end{align*}
From Theorem 2.2 of \citet{Du2018AlgorithmicRI} (stated in Section~\ref{sec:multi_neuron}) we get that
\[
\mathbf{a}(t)\cdot\mathbf{a}(t)^{\top} =\mathbf{W}(t)^\top\mathbf{W}(t) + \mathbf{\Delta}~,
\]
where $\mathbf{\Delta} \in \mathbb{R}^{m \times m}$.

For the case of strict balanced initialization we have $\mathbf{\Delta} = 0$, and therefore
\begin{align*}
\left(\left(\sum_{i=1}^m a_i^2(t)\right) \mathbf{I} + \mathbf{V}(t)\mathbf{U}(t)\right)^{-1} &= \left(\left(\sum_{i=1}^m a_i^2(t)\right) \mathbf{I} + \mathbf{W}(t)^\top\mathbf{W}(t)\right)^{-1} =\\
&=\left(\left(\sum_{i=1}^m a_i^2(t)\right) \mathbf{I} + \mathbf{a}(t)\mathbf{a}(t)^{\top}\right)^{-1}\\
&= \left(\sum_{i=1}^m a_i^2(t)\right)^{-1}\mathbf{I} - \frac{\mathbf{a}(t)\mathbf{a}(t)^{\top}}{2\left(\sum_{i=1}^m a_i^2(t)\right)^2}~,
\end{align*}
where in the last transition we used the Sherman-Morrison lemma. 
It follows that
\begin{align*}
&\left(\sum_{i=1}^m a_i^2(t)\mathbf{I} + \sum_{i=1}^m\mathbf{w}_i(t) \mathbf{w}_i(t)^\top\right)^{-1} =\\
&=\left(\sum_{i=1}^m a_i^2(t)\right)^{-1}\left( \mathbf{I} -  \mathbf{W}(t)\left(\left(\sum_{i=1}^m a_i^2(t)\right)^{-1}\mathbf{I} - \frac{\mathbf{a}(t)\mathbf{a}(t)^{\top}}{2\left(\sum_{i=1}^m a_i^2(t)\right)^2}\right) \mathbf{W}(t)^\top\right)~.
\end{align*}

We continue and write
\begin{align*}
&\left(\sum_{i=1}^m a_i^2(t)\mathbf{I} + \sum_{i=1}^m\mathbf{w}_i(t) \mathbf{w}_i(t)^\top\right)^{-1} =\\
&= \left(\sum_{i=1}^m a_i^2(t)\right)^{-1}\left( \mathbf{I} - \left(\sum_{i=1}^m a_i^2(t)\right)^{-1} \mathbf{W}(t) \mathbf{W}(t)^\top + \frac12\left(\sum_{i=1}^m a_i^2(t)\right)^{-2}\left(\mathbf{W}(t)\mathbf{W}(t)^\top\right)^2\right)~.
\end{align*}
Using Theorem 2.1 of \citet{Du2018AlgorithmicRI} (stated in Section~\ref{sec:multi_neuron}), we know that
\[
a_i(t)^2 = \|\mathbf{w}_i(t)\|^2~.
\]
Therefore,
\[
\|\tilde{\mathbf{w}_{i}}(t)\| = |a_i(t)|\|\mathbf{w}_i(t)\| = a_i(t)^2
\]
and
\[
\sum_{i=1}^m a_i^2(t) = \sum_{i=1}^m \|\tilde{\mathbf{w}_{i}}(t)\|~.
\]
So, we can write
\begin{align*}
&\left(\sum_{i=1}^m a_i^2(t)\mathbf{I} + \sum_{i=1}^m\mathbf{w}_i(t) \mathbf{w}_i(t)^\top\right)^{-1} =\\
&=\left(\sum_{i=1}^m a_i^2(t)\right)^{-1}\left( \mathbf{I} - \left(\sum_{i=1}^m a_i^2(t)\right)^{-1} \left(\sum_{i=1}^m\mathbf{w}_i(t) \mathbf{w}_i(t)^\top\right) + \frac12 \left(\sum_{i=1}^m a_i^2(t)\right)^{-2} \left(\sum_{i=1}^m\mathbf{w}_i(t) \mathbf{w}_i(t)^\top\right)^2\right)=\\
&= \left(\sum_{i=1}^m \| \tilde{\mathbf{w}_i}(t)\|\right)^{-1}\left( \mathbf{I} - \left(\sum_{i=1}^m \| \tilde{\mathbf{w}_i}(t)\|\right)^{-1} \left(\sum_{i=1}^m\frac{\tilde{\mathbf{w}_i}(t) \tilde{\mathbf{w}_i}(t)^\top}{\| \tilde{\mathbf{w}_i}(t)\|}\right)  +\frac12\left(\sum_{i=1}^m \| \tilde{\mathbf{w}_i}(t)\|\right)^{-2} \left(\sum_{i=1}^m\frac{\tilde{\mathbf{w}_i}(t) \tilde{\mathbf{w}_i}(t)^\top}{\| \tilde{\mathbf{w}_i}(t)\|}\right)^2\right)~.
\end{align*}

Now, since
\[
\mathbf{a}(t)\mathbf{a}(t)^{\top} =\mathbf{W}(t)^\top\mathbf{W}(t)~,
\]
we can say that $\mathbf{W}(t)^\top\mathbf{W}(t)$ is a rank one matrix, and therefore also $\mathbf{W}(t)$, and also $\tilde{\mathbf{W}}(t)$.
\\
Therefore, all $\tilde{\mathbf{w}}_i$ are equal up to a multiplicative factor,
\[
\tilde{\mathbf{w}}_i(t) = c_i(t)\tilde{\mathbf{w}}(t)
\]
where from definition
\[
\sum_{i=1}^m c_i(t) = 1~.
\]
Therefore,
\[
\|\tilde{\mathbf{w}}_i(t)\| = |c_i(t)|\|\tilde{\mathbf{w}}(t)\|
\]
\[
\Rightarrow \sum_{i=1}^m \|\tilde{\mathbf{w}}_i(t)\| = \left(\sum_{i=1}^m |c_i(t)|\right)\| \tilde{\mathbf{w}}(t)\|
\]
\[
\Rightarrow \sum_{i=1}^m\frac{\tilde{\mathbf{w}}_i(t) \tilde{\mathbf{w}}_i(t)^\top}{\| \tilde{\mathbf{w}}_i(t)\|} = \left(\sum_{i=1}^m |c_i(t)|\right)\frac{\tilde{\mathbf{w}}(t)\tilde{\mathbf{w}}(t)^\top}{\|\tilde{\mathbf{w}}(t)\|}~,
\]
giving us
\begin{align*}
&\left(\sum_{i=1}^m a_i^2(t)\mathbf{I} + \sum_{i=1}^m\mathbf{w}_i(t) \mathbf{w}_i(t)^\top\right)^{-1}\frac{d}{dt}\tilde{\mathbf{w}}(t) =\\
&= \frac{1}{\left(\sum_{i=1}^m |c_i(t)|\right)}\frac{1}{|\|\tilde{\mathbf{w}}(t)\|}\left(\mathbf{I} - \frac{\tilde{\mathbf{w}}(t)\tilde{\mathbf{w}}(t)^\top}{\|\tilde{\mathbf{w}}(t)\|^2} + \frac12\left(\frac{\tilde{\mathbf{w}}(t)\tilde{\mathbf{w}}(t)^\top}{\|\tilde{\mathbf{w}}(t)\|^2}\right)^2\right)\frac{d}{dt}\tilde{\mathbf{w}}(t) = \sum_{n=1}^{N}\mathbf{x}^{\left(n\right)}r^{\left(n\right)}\\
&\Rightarrow \frac{1}{\left(\sum_{i=1}^m |c_i(t)|\right)}\frac{1}{|\|\tilde{\mathbf{w}}(t)\|}\left(\mathbf{I} - \frac12\frac{\tilde{\mathbf{w}}(t)\tilde{\mathbf{w}}(t)^\top}{\|\tilde{\mathbf{w}}(t)\|^2} \right)\frac{d}{dt}\tilde{\mathbf{w}}(t) = \sum_{n=1}^{N}\mathbf{x}^{\left(n\right)}r^{\left(n\right)}~,
\end{align*}

where in the last transition we used
\[
\left(\frac{\tilde{\mathbf{w}}(t)\tilde{\mathbf{w}}(t)^\top}{\|\tilde{\mathbf{w}}(t)\|^2}\right)^2 = \frac{\tilde{\mathbf{w}}(t)\tilde{\mathbf{w}}(t)^\top\tilde{\mathbf{w}}(t)\tilde{\mathbf{w}}(t)^\top}{\|\tilde{\mathbf{w}}(t)\|^4} = \frac{\tilde{\mathbf{w}}(t)\tilde{\mathbf{w}}(t)^\top}{\|\tilde{\mathbf{w}}(t)\|^2}~.
\]

We follow the "warped IMD" technique (presented in detail in Section~\ref{sec: A new technique for deriving the implicit bias})
and multiply the equation by some function $g\left(\mathbf{\tilde{w}}_i(t)\right)$ 
\begin{align*}
\frac{g\left(\mathbf{\tilde{w}}(t)\right)}{\left(\sum_{i=1}^m |c_i(t)|\right)}\frac{1}{|\|\tilde{\mathbf{w}}(t)\|}\left(\mathbf{I} - \frac12\frac{\tilde{\mathbf{w}}(t)\tilde{\mathbf{w}}(t)^\top}{\|\tilde{\mathbf{w}}(t)\|^2} \right)\frac{d}{dt}\tilde{\mathbf{w}}(t) = \left(\sum_{n=1}^{N}\mathbf{x}^{\left(n\right)}g\left(\mathbf{\tilde{w}}(t)\right)r^{\left(n\right)}\right)~.
\end{align*}

Following the approach in Section~\ref{sec: A new technique for deriving the implicit bias}, we then try and find $q\left(\mathbf{\tilde{w}}_i(t)\right)=\hat{q}\left(\left\Vert \mathbf{\tilde{w}}_i(t)\right\Vert \right)+\mathbf{z}^{\top}\mathbf{\tilde{w}}_i(t)$ and $g\left(\mathbf{\tilde{w}}_i(t)\right)$ such that 
\begin{align}
\nabla^{2}q\left(\mathbf{\tilde{w}}(t)\right) = \frac{g\left(\mathbf{\tilde{w}}(t)\right)}{\left(\sum_{i=1}^m |c_i(t)|\right)}\frac{1}{|\|\tilde{\mathbf{w}}(t)\|}\left(\mathbf{I} - \frac12\frac{\tilde{\mathbf{w}}(t)\tilde{\mathbf{w}}(t)^\top}{\|\tilde{\mathbf{w}}(t)\|^2} \right)~,
\label{eq: warped hessian strictly balanced fully connected network}
\end{align}
so that then we'll have
\[
\nabla^{2}q\left(\mathbf{\tilde{w}}(t)\right)\frac{d}{dt}\tilde{\mathbf{w}}(t)=\sum_{n=1}^{N}\mathbf{x}^{\left(n\right)}g\left(\mathbf{\tilde{w}}(t)\right)r^{\left(n\right)}(t)
\]
\[
\frac{d}{dt}\left(\nabla q\left(\mathbf{\tilde{w}}(t)\right)\right)=\sum_{n=1}^{N}\mathbf{x}^{\left(n\right)}g\left(\mathbf{\tilde{w}}(t)\right)r^{\left(n\right)}(t)
\]
\[
\nabla q\left(\mathbf{\tilde{w}}(t)\right)-\nabla q\left(\mathbf{\tilde{w}}(0)\right)=\sum_{n=1}^{N}\mathbf{x}^{\left(n\right)}\int_{0}^{t}g\left(\mathbf{\tilde{w}}(t')\right)r^{\left(n\right)}(t')dt'~.
\]
Assuming $\nabla q\left(\mathbf{\tilde{w}}(0)\right) = 0$, and denoting $\nu^{(n)} = \int_{0}^{\infty}g\left(\mathbf{\tilde{w}}(t')\right)r^{\left(n\right)}(t')dt'$, we get the KKT condition
\[
\nabla q\left(\mathbf{\tilde{w}}(\infty)\right)=\sum_{n=1}^{N}\mathbf{x}^{\left(n\right)}\nu^{\left(n\right)}~.
\]
To find $q$ we note that
\[
\nabla q\left(\mathbf{\tilde{w}}(t)\right)=\hat{q}'\left(\left\Vert \mathbf{\tilde{w}}(t)\right\Vert \right)\frac{\mathbf{\tilde{w}}(t)}{\left\Vert \mathbf{\tilde{w}}(t)\right\Vert }+\mathbf{z}
\]
and
\begin{align*}
\nabla^{2}q\left(\mathbf{\tilde{w}}(t)\right)&=\left[\hat{q}''\left(\left\Vert \mathbf{\tilde{w}}(t)\right\Vert \right)-\hat{q}'\left(\left\Vert \mathbf{\tilde{w}}(t)\right\Vert \right)\frac{1}{\left\Vert \mathbf{\tilde{w}}(t)\right\Vert }\right]\frac{\mathbf{\tilde{w}}(t)\mathbf{\tilde{w}}^{\top}(t)}{\left\Vert \mathbf{\tilde{w}}(t)\right\Vert ^{2}}+\hat{q}'\left(\left\Vert \mathbf{\tilde{w}}(t)\right\Vert \right)\frac{1}{\left\Vert \mathbf{\tilde{w}}(t)\right\Vert }\mathbf{I}\\&=\frac{\hat{q}'\left(\left\Vert \mathbf{\tilde{w}}(t)\right\Vert \right)}{\left\Vert \mathbf{\tilde{w}}(t)\right\Vert }\left[\mathbf{I}-\left[1-\left\Vert \mathbf{\tilde{w}}(t)\right\Vert \frac{\hat{q}''\left(\left\Vert \mathbf{\tilde{w}}(t)\right\Vert \right)}{\hat{q}'\left(\left\Vert \mathbf{\tilde{w}}(t)\right\Vert \right)}\right]\frac{\mathbf{\tilde{w}}(t)\mathbf{\tilde{w}}^{\top}(t)}{\left\Vert \mathbf{\tilde{w}}(t)\right\Vert ^{2}}\right]~.
\end{align*}
Comparing the form above with the Hessian in Eq.~\ref{eq: warped hessian strictly balanced fully connected network} we require
\[
\frac{g\left(\mathbf{\tilde{w}}(t)\right)}{\left(\sum_{i=1}^m |c_i(t)|\right)} = \hat{q}'\left(\left\Vert \mathbf{\tilde{w}}(t)\right\Vert \right)~,
\]
and
\[
1 - \left\Vert \mathbf{\tilde{w}}(t)\right\Vert \frac{\hat{q}''\left(\left\Vert \mathbf{\tilde{w}}(t)\right\Vert \right)}{\hat{q}'\left(\left\Vert \mathbf{\tilde{w}}(t)\right\Vert \right)} = \frac12
\]

\[
\Rightarrow\frac{\hat{q}''\left(x\right)}{\hat{q}'\left(x\right)} = \frac{1}{2x}
\]
\[
\log\hat{q}'\left(x\right) = \frac{1}{2}\ln x+C
\]
\[
\hat{q}'\left(x\right) = C\sqrt{x}~.
\]

Therefore,
\[
q\left(\mathbf{\tilde{w}}(t)\right)=C\left\Vert \mathbf{\tilde{w}}(t)\right\Vert ^{3/2}+\mathbf{z}^{\top}\mathbf{\tilde{w}}(t)+C'~,
\]
and using the condition $\nabla q\left(\mathbf{\tilde{w}}(0)\right) = 0$ we get
\[
q\left(\mathbf{\tilde{w}}(t)\right)=C\left\Vert \mathbf{\tilde{w}}(t)\right\Vert ^{3/2}-C\frac{3}{2}\left\Vert \mathbf{\tilde{w}}(0)\right\Vert ^{-1/2}\mathbf{\tilde{w}}(0)^{\top}\mathbf{\tilde{w}}(t)+C'~.
\]
We can set $C=1$, $C'=0$ and get
\[
q\left(\mathbf{\tilde{w}}(t)\right)=\left\Vert \mathbf{\tilde{w}}(t)\right\Vert ^{3/2}-\frac{3}{2}\left\Vert \mathbf{\tilde{w}}(0)\right\Vert ^{-1/2}\mathbf{\tilde{w}}(0)^{\top}\mathbf{\tilde{w}}(t)~.
\]

Therefore, gradient flow satisfies the KKT conditions for minimizing this $q$.
\end{proof}

\section{Proof of Theorem ~\ref{theorem: Q for small initialization multi-neuron linear network}}
\label{appendix: rich regime for fully connected multi-neuron network}
We recall the proof of Theorem~\ref{theorem: Q for single-neuron linear network} given in Appendix~\ref{appendix: proof of Q for single-neuron linear network}.

The form of the $q$ function described in the proof is $q\left(\mathbf{\tilde{w}}_i(t)\right)=\hat{q}\left(\left\Vert \mathbf{\tilde{w}}_i(t)\right\Vert \right)+\mathbf{z}^{\top}\mathbf{\tilde{w}}_i(t)$, where
\[
\mathbf{z} = -\frac{3}{2}\sqrt{\sqrt{\left\Vert \mathbf{\tilde{w}}(0)\right\Vert ^{2}+\frac{\delta^{2}}{4}}-\frac{\delta}{2}}\frac{\mathbf{\tilde{w}}(0)}{\left\Vert \mathbf{\tilde{w}}(0)\right\Vert }~.
\]

Under the limit $\|\tilde{\mathbf{w}}_i(0)\| \rightarrow 0$ we can see that $\|\mathbf{z}\| \rightarrow 0$.

When the linear term captured by $\mathbf{z}$ in the $q$ function is equal to zero, we have
\[
\nabla q\left(\mathbf{\tilde{w}}_{i}(\infty)\right)=\hat{q}'\left(\left\Vert \mathbf{\tilde{w}}_{i}(\infty)\right\Vert \right)\frac{\mathbf{\tilde{w}}_{i}(\infty)}{\left\Vert \mathbf{\tilde{w}}_{i}(\infty)\right\Vert }=\sum_{n}\mathbf{x}^{\left(n\right)}\nu_{i}^{\left(n\right)}~.
\]
Defining $\hat{\nu}_{i}^{\left(n\right)}=\frac{\nu_{i}^{\left(n\right)}\left\Vert \mathbf{\tilde{w}}_{i}(\infty)\right\Vert }{\hat{q}'\left(\left\Vert \mathbf{\tilde{w}}_{i}(\infty)\right\Vert \right)}$ we get
\[
\mathbf{\tilde{w}}_{i}(\infty)=\sum_{n}\mathbf{x}^{\left(n\right)}\hat{\nu}_{i}^{\left(n\right)}~.
\]
We notice that  
\[
\hat{q}'\left(x\right)=\sqrt{\sqrt{x^{2}+\frac{\delta_i^{2}}{4}}-\frac{\delta_i}{2}}
\]
and so $\hat{\nu}_{i}$ is finite assuming we converge to a finite-norm weight vector $\mathbf{\tilde{w}}_{i}(\infty)$, which is correct for the square loss.

Using the linear predictor definition of $\mathbf{\tilde{w}}(\infty)=\sum_{i}\mathbf{\tilde{w}}_{i}(\infty)$, denoting $\hat{\nu}^{\left(n\right)}=\sum_{i}\hat{\nu}_{i}^{\left(n\right)}$ and summing over $i$ gives

\[
\mathbf{\tilde{w}}(\infty)=\sum_{n}\mathbf{x}^{\left(n\right)}\hat{\nu}^{\left(n\right)}
\]

which is a valid KKT stationarity condition of the form $\nabla q\left(\mathbf{\tilde{w}}(\infty)\right) = \sum_{n}\mathbf{x}^{\left(n\right)}\hat{\nu}^{\left(n\right)}$ with $\nabla q\left(\mathbf{w}\right) = \mathbf{w}$.

Hence, gradient flow satisfies the KKT conditions for minimizing this $q$.

It follows that for a multi-neuron fully connected network with non-zero infinitesimal initialization, 
\[
\mathbf{\tilde{w}}(\infty)=\mathrm{argmin}_{\mathbf{w}}\left\Vert \mathbf{w}\right\Vert ^{2}\,\,\,\mathrm{s.t.}\,\,\mathbf{X}^{\top}\mathbf{w}=\mathbf{y}
\]
which is equivalent to
\[
\mathbf{\tilde{w}}(\infty)=\mathrm{argmin}_{\mathbf{w}}\left\Vert \mathbf{w}\right\Vert \,\,\,\mathrm{s.t.}\,\,\mathbf{X}^{\top}\mathbf{w}=\mathbf{y}~.
\]

\section{Characterization of the Implicit Bias Captured in Theorem~\ref{theorem: Q for unbiased u-v model}}
\label{appendix: charachtersation of q for unbiased u-v model}
In this Appendix we provide a detailed characterization of the implicit
bias for a diagonal linear network as described in Theorem~\ref{theorem: Q for unbiased u-v model},
\[
\tilde{\mathbf{w}}(\infty) = \argmin_\mathbf{w} Q_{\boldsymbol{k}}(\mathbf{w}) \quad \mathrm{s.t.\,\,}\mathbf{X}^\top\mathbf{w} = \mathbf{y}
\]
where
\begin{align*}
Q_{\boldsymbol{k}}\left(\mathbf{w}\right)=\sum_{i=1}^{d}q_{k_i}\left(w_{i}\right)~,
\end{align*}

\begin{align*}
q_k\left(x\right)&=\frac{\sqrt{k}}{4}\left[1-\sqrt{1+\frac{4x^{2}}{k}}+\frac{2x}{\sqrt{k}}\mathrm{arcsinh}\left(\frac{2x}{\sqrt{k}}\right)\right]
\end{align*}
and
\[
\sqrt{k_i}=\frac{4\alpha_{i}\left(1+s_{i}^{2}\right)}{1-s_{i}^{2}}~.
\]

For simplicity, we next assume $\alpha_i = \alpha, \,\,\ s_i = s\,\,\, \forall i \in [d]$.

We can notice that for $k \xrightarrow{} \infty$, i.e. $\frac{\alpha}{1 - s^2} \xrightarrow{} \infty$ we get that:

\begin{align*}
&q_k\left(w_{i}\right) \xrightarrow{k \rightarrow \infty}\frac{w_{i}^{2}}{\sqrt{k}}=\frac{1}{2\left(u_{+,i}^{2}\left(0\right)+v_{+,i}^{2}\left(0\right)\right)}w_{i}^{2}\\
&\Rightarrow Q_{\boldsymbol{k}}\left(\mathbf{w}\right)=\sum_{i=1}^{d}q_k\left(w_{i}\right)=\sum_{i=1}^{d}\frac{1}{2\left(u_{+,i}^{2}\left(0\right)+v_{+,i}^{2}\left(0\right)\right)}w_{i}^{2}~.
\end{align*}

Calculating the tangent kernel at the initialization we get
\begin{align*}
    K(\mathbf{x}_{1},\mathbf{x}_{2})=&\langle\nabla f\left(\mathbf{x}_{1}\right),\nabla f\left(\mathbf{x}_{2}\right)\rangle\\
    =&\langle\left[\mathbf{x}_{1}\circ \mathbf{u}_{+}\left(0\right),\mathbf{x}_{1}\circ \mathbf{v}{}_{+}\left(0\right),-\mathbf{x}_{1}\circ \mathbf{u}_{-}\left(0\right),-\mathbf{x}_{1}\circ \mathbf{v}_{-}\left(0\right)\right],\\  &\left[\mathbf{x}_{2}\circ \mathbf{u}_{+}\left(0\right),\mathbf{x}_{2}\circ \mathbf{v}{}_{+}\left(0\right),-\mathbf{x}_{2}\circ \mathbf{u}_{-}\left(0\right),-\mathbf{x}_{2}\circ \mathbf{v}_{-}\left(0\right)\right]\rangle
    \\=&\mathbf{x}_{1}^{\top}\textrm{diag}\left(\mathbf{u}_{+}^{2}\left(0\right)+\mathbf{v}_{+}^{2}\left(0\right)+\mathbf{u}_{-}^{2}\left(0\right)+\mathbf{v}_{-}^{2}\left(0\right)\right)\mathbf{x}_{2}~.
\end{align*}

For the case of unbiased initialization ($u_{+,i}\left(0\right)=u_{-,i}\left(0\right),v_{+,i}\left(0\right)=v_{-,i}\left(0\right)$) we have
\[
K(\mathbf{x}_{1},\mathbf{x}_{2}) = 2\mathbf{x}_{1}^{\top}\textrm{diag}\left(\mathbf{u}_{+}^{2}\left(0\right)+\mathbf{v}_{+}^{2}\left(0\right)\right)\mathbf{x}_{2}~.
\]

Therefore, using Lemma~\ref{lemma: rkhs_norm}, we can see that $Q_{\boldsymbol{k}}(\mathbf{w})$ is the RKHS norm with respect to the NTK at initialization. Therefore, $k \xrightarrow{} \infty$ indeed describes the NTK regime.
\\
\\
For $k \xrightarrow{} 0$, i.e. $\frac{\alpha}{1 - s^2} \xrightarrow{} 0$ we get that:
\begin{align*}
   q_k\left(w_{i}\right)	&=\frac{\sqrt{k}}{4}\left[1-\sqrt{1+\frac{4w_{i}^{2}}{k}}+\frac{2w_{i}}{\sqrt{k}}\textrm{arcsinh}\left(\frac{2w_{i}}{\sqrt{k}}\right)\right]\\
   &=\frac{\sqrt{k}}{4}-\sqrt{\frac{k}{16}+\frac{w_{i}^{2}}{4}}+\frac{w_{i}}{2}\textrm{arcsinh}\left(\frac{2w_{i}}{\sqrt{k}}\right)\\
   &\xrightarrow{k \rightarrow 0}\-\frac{\left|w_{i}\right|}{2}+\frac{\left|w_{i}\right|}{2}\log\left(\frac{4\left|w_{i}\right|}{\sqrt{k}}\right)\\
	&=\frac{1}{2}\left[-\left|w_{i}\right|+\left|w_{i}\right|\log\left(\frac{4\left|w_{i}\right|}{\sqrt{k}}\right)\right]\\
	&=\frac{1}{2}\left[\left|w_{i}\right|\log\left(\frac{1}{\sqrt{k}}\right)+\left|w_{i}\right|\left(\log\left(4\left|w_{i}\right|\right)-1\right)\right]\\
\end{align*}
\begin{align*}
    &\Rightarrow\frac{q_k\left(w_{i}\right)}{\frac{1}{2}\log\left(\frac{1}{\sqrt{k}}\right)}\rightarrow\left|w_{i}\right|+\frac{\left|w_{i}\right|\left(\log\left(4\left|w_{i}\right|\right)-1\right)}{\log\left(\frac{1}{\sqrt{k}}\right)}\\
	&=\left|w_{i}\right|+O\left(\frac{1}{\log\left(\frac{1}{\sqrt{k}}\right)}\right)\rightarrow\left|w_{i}\right|
\end{align*}

Therefore, 
\[
Q_{\boldsymbol{k}}\left(\mathbf{w}\right)=\sum_{i=1}^{d}\left|w_{i}\right|=\left\Vert \mathbf{w}\right\Vert _{1}
\]

and $k \xrightarrow{} 0$ describes the rich regime \cite{Woodworth2020KernelAR}.

\section{Characterization of the Implicit Bias Captured in Theorem~\ref{theorem: Q for single-neuron linear network}}
\label{appendix: charachtersation of q for fully connected single-neuron network}

In this Appendix we provide a detailed characterization of the implicit
bias for a two-layer fully connected neural network with a single hidden neuron ($m=1$) described in Theorem~\ref{theorem: Q for single-neuron linear network},
\[
\tilde{\mathbf{w}}(\infty) = \argmin_\mathbf{w} q(\mathbf{w}) \quad \mathrm{s.t.\,\,}\mathbf{X}^\top\mathbf{w} = \mathbf{y}
\]
\[
q\left(\mathbf{\tilde{w}}\right)=\hat{q}\left(\left\Vert \mathbf{\tilde{w}}\right\Vert \right)+\mathbf{z}^{\top}\mathbf{\tilde{w}}~,
\]
where
\[
\hat{q}\left(x\right)=\frac{\left(x^{2}-\frac{\delta}{2}\left(\frac{\delta}{2}+\sqrt{x^{2}+\frac{\delta^{2}}{4}}\right)\right)\sqrt{\sqrt{x^{2}+\frac{\delta^{2}}{4}}-\frac{\delta}{2}}}{x}
\]
\[
\mathbf{z}=-\frac{3}{2}\sqrt{\sqrt{\left\Vert \mathbf{\tilde{w}}\left(0\right)\right\Vert ^{2}+\frac{\delta^{2}}{4}}-\frac{\delta}{2}}\frac{\mathbf{\tilde{w}}\left(0\right)}{\left\Vert \mathbf{\tilde{w}}\left(0\right)\right\Vert }~.
\]

Note that for the sake of simplicity the notations above are an abbreviated version of those found Theorem~\ref{theorem: Q for single-neuron linear network}.

We will employ the initialization orientation, defined as $\mathbf{u}=\frac{\vw(0)}{\|\vw(0)\|}$, and the initialization scale, $\left\Vert \mathbf{\tilde{w}}\left(0\right)\right\Vert =\alpha$.

\subsection{The case $\alpha\rightarrow0$ for any $0\leq s<1$}

Note that from Lemma~\ref{lemma: alpha, s, delta relations} (part 2) we have
\[
\left\Vert \mathbf{z}\right\Vert =\frac{3}{2}\sqrt{\sqrt{\left\Vert \mathbf{\tilde{w}}\left(0\right)\right\Vert ^{2}+\frac{\delta^{2}}{4}}-\frac{\delta}{2}}=\frac{3}{2}\sqrt{\sqrt{\alpha^{2}+\frac{\delta^{2}}{4}}-\frac{\delta}{2}}=\frac{3}{2}\sqrt{\alpha\frac{1-s}{1+s}}~,
\]
 and thus for any $0\leq s<1$ when $\alpha\rightarrow0$ we get that
$\left\Vert \mathbf{z}\right\Vert \rightarrow0$. It follows that
$q_{\delta}\left(\mathbf{\tilde{w}}\right)=\hat{q}\left(\left\Vert \mathbf{\tilde{w}}\right\Vert \right)$
and since $\hat{q}\left(x\right)$ is a monotonically increasing function
(for any $\delta$) we get the $\ell_{2}$ implicit bias,
\[
\mathbf{\tilde{w}}\left(\infty\right)=\argmin_{\tilde{\vw}}\left(q_{\delta}\left(\mathbf{\tilde{w}}\right)\right)=\argmin_{\tilde{\vw}}\left(\hat{q}\left(\left\Vert \mathbf{\tilde{w}}\right\Vert \right)\right)=\argmin_{\tilde{\vw}}\left\Vert \mathbf{\tilde{w}}\right\Vert~. 
\]
We call this regime the \textit{Anti-NTK} regime.

\subsection{Other special cases}

Here we analyze the Taylor expansion of $q\left(\mathbf{\tilde{w}}\right)$
around $\mathbf{\tilde{w}}\left(0\right)$. To this end, we know that
\begin{align*}
\nabla^{2}q\left(\mathbf{\tilde{w}}\right) & =\frac{\hat{q}'\left(\left\Vert \mathbf{\tilde{w}}\right\Vert \right)}{\left\Vert \mathbf{\tilde{w}}\right\Vert }\left(\mathbf{I}-\frac{\mathbf{\tilde{w}}\mathbf{\tilde{w}}^{\top}}{2\left(\frac{\delta}{2}+\sqrt{\frac{\delta^{2}}{4}+\left\Vert \mathbf{\tilde{w}}\right\Vert ^{2}}\right)\sqrt{\frac{\delta^{2}}{4}+\left\Vert \mathbf{\tilde{w}}\right\Vert ^{2}}}\right)~,
\end{align*}
and thus the third-order term is order of $\frac{d}{dx}\frac{\hat{q}'\left(x\right)}{x}\left(\left\Vert \mathbf{\tilde{w}}(0)\right\Vert \right)$.
Since we know that $\nabla q\left(\mathbf{\tilde{w}}\left(0\right)\right)=0$
we can write the Taylor expansion as follows
\[
q\left(\mathbf{\tilde{w}}\right)=q\left(\mathbf{\tilde{w}}\left(0\right)\right)+\frac{1}{2}\left(\mathbf{\tilde{w}}-\mathbf{\tilde{w}}\left(0\right)\right)^{\top}\nabla^{2}q\left(\mathbf{\tilde{w}}\left(0\right)\right)\left(\mathbf{\tilde{w}}-\mathbf{\tilde{w}}\left(0\right)\right)+O\left(\frac{d}{dx}\frac{\hat{q}'\left(x\right)}{x}\left(\left\Vert \mathbf{\tilde{w}}(0)\right\Vert \right)\right)~.
\]
By using Lemma~\ref{lemma: alpha, s, delta relations} and 
\[
\hat{q}'\left(x\right)=\sqrt{\sqrt{x^{2}+\frac{\delta^{2}}{4}}-\frac{\delta}{2}}
\]
we calculate
\begin{align*}
\nabla^{2}q\left(\mathbf{\tilde{w}}\left(0\right)\right) & =\frac{\hat{q}'\left(\left\Vert \mathbf{\tilde{w}}\left(0\right)\right\Vert \right)}{\left\Vert \mathbf{\tilde{w}}\left(0\right)\right\Vert }\left(\mathbf{I}-\frac{\mathbf{\tilde{w}}\left(0\right)\mathbf{\tilde{w}}\left(0\right)^{\top}}{\left(\frac{\delta}{2}+\sqrt{\frac{\delta^{2}}{4}+\left\Vert \mathbf{\tilde{w}}\left(0\right)\right\Vert ^{2}}\right)\sqrt{\delta^{2}+4\left\Vert \mathbf{\tilde{w}}\left(0\right)\right\Vert ^{2}}}\right)\\
 & =\frac{\sqrt{\sqrt{\left\Vert \mathbf{\tilde{w}}\left(0\right)\right\Vert ^{2}+\frac{\delta^{2}}{4}}-\frac{\delta}{2}}}{\left\Vert \mathbf{\tilde{w}}\left(0\right)\right\Vert }\left(\mathbf{I}-\frac{\mathbf{\tilde{w}}\left(0\right)\mathbf{\tilde{w}}\left(0\right)^{\top}}{2\left(\frac{\delta}{2}+\sqrt{\frac{\delta^{2}}{4}+\left\Vert \mathbf{\tilde{w}}\left(0\right)\right\Vert ^{2}}\right)\sqrt{\frac{\delta^{2}}{4}+\left\Vert \mathbf{\tilde{w}}\left(0\right)\right\Vert ^{2}}}\right)\\
 & =\frac{\sqrt{\sqrt{\alpha^{2}+\frac{\delta^{2}}{4}}-\frac{\delta}{2}}}{\alpha}\left(\mathbf{I}-\frac{\alpha^{2}\mathbf{u}\mathbf{u}^{\top}}{2\left(\frac{\delta}{2}+\sqrt{\frac{\delta^{2}}{4}+\alpha^{2}}\right)\sqrt{\frac{\delta^{2}}{4}+\alpha^{2}}}\right)\\
 & =\sqrt{\frac{1-s}{\alpha}}\sqrt{\frac{1}{1+s}}\left(\mathbf{I}-\frac{\left(1-s\right)^{2}}{2\left(1+s^{2}\right)}\mathbf{u}\mathbf{u}^{\top}\right)~.
\end{align*}
 Also, by using
\begin{align*}
\hat{q}''\left(x\right) & =\frac{x}{2\sqrt{x^{2}+\frac{\delta^{2}}{4}}\sqrt{\sqrt{x^{2}+\frac{\delta^{2}}{4}}-\frac{\delta}{2}}}
\end{align*}
 we have that
\begin{align*}
\frac{d}{dx}\frac{\hat{q}'\left(x\right)}{x} & =\frac{\hat{q}''\left(x\right)x-\hat{q}'\left(x\right)}{x^{2}}\\
 & =\frac{\frac{x^{2}}{2\sqrt{x^{2}+\frac{\delta^{2}}{4}}\sqrt{\sqrt{x^{2}+\frac{\delta^{2}}{4}}-\frac{\delta}{2}}}-\sqrt{\sqrt{x^{2}+\frac{\delta^{2}}{4}}-\frac{\delta}{2}}}{x^{2}}\\
 & =\frac{1}{2\sqrt{x^{2}+\frac{\delta^{2}}{4}}\sqrt{\sqrt{x^{2}+\frac{\delta^{2}}{4}}-\frac{\delta}{2}}}-\frac{\sqrt{\sqrt{x^{2}+\frac{\delta^{2}}{4}}-\frac{\delta}{2}}}{x^{2}}~,
\end{align*}
 and thus, using Lemma~\ref{lemma: alpha, s, delta relations} we get
\begin{align*}
\frac{d}{dx}\frac{\hat{q}'\left(x\right)}{x}\left(\left\Vert \mathbf{\tilde{w}}(0)\right\Vert \right) & =\frac{1}{2\sqrt{\alpha^{2}+\frac{\delta^{2}}{4}}\sqrt{\sqrt{\alpha^{2}+\frac{\delta^{2}}{4}}-\frac{\delta}{2}}}-\frac{\sqrt{\sqrt{\alpha^{2}+\frac{\delta^{2}}{4}}-\frac{\delta}{2}}}{\alpha^{2}}\\
 & =-\frac{\left(1-s\right)^{2.5}}{\alpha^{1.5}}\left(\frac{1}{2\left(1+s^{2}\right)\sqrt{1+s}}\right)~.
\end{align*}
 Therefore, the Taylor expansion is
\[
q\left(\mathbf{\tilde{w}}\right)=q\left(\mathbf{\tilde{w}}\left(0\right)\right)+\frac{1}{2}\left(\mathbf{\tilde{w}}-\mathbf{\tilde{w}}\left(0\right)\right)^{\top}\left[\sqrt{\frac{1-s}{\alpha}}\sqrt{\frac{1}{1+s}}\left(\mathbf{I}-\frac{\left(1-s\right)^{2}}{2\left(1+s^{2}\right)}\mathbf{u}\mathbf{u}^{\top}\right)\right]\left(\mathbf{\tilde{w}}-\mathbf{\tilde{w}}\left(0\right)\right)+O\left(\frac{\left(1-s\right)^{2.5}}{\alpha^{1.5}}\left(\frac{1}{2\left(1+s^{2}\right)\sqrt{1+s}}\right)\right)~.
\]
 We are interested in cases where the higher order terms vanish. Since
$0\leq s<1$, we only need to require 
\[
\frac{\left(1-s\right)^{2.5}}{\alpha^{1.5}}\ll\sqrt{\frac{1-s}{\alpha}}
\]
\begin{align}
\Rightarrow\frac{\left(1-s\right)^{2}}{\alpha}\ll1~.
\label{taylor_cond}
\end{align}
In follows that when $\frac{\left(1-s\right)^{2}}{\alpha}\ll1$ we
can approximate
\[
q\left(\mathbf{\tilde{w}}\right)\approx q\left(\mathbf{\tilde{w}}\left(0\right)\right)+\frac{1}{2}\sqrt{\frac{1-s}{\alpha}}\sqrt{\frac{1}{1+s}}\left(\mathbf{\tilde{w}}-\mathbf{\tilde{w}}\left(0\right)\right)^{\top}\left(\mathbf{I}-\frac{\left(1-s\right)^{2}}{2\left(1+s^{2}\right)}\mathbf{u}\mathbf{u}^{\top}\right)\left(\mathbf{\tilde{w}}-\mathbf{\tilde{w}}\left(0\right)\right)~.
\]
 In this case, minimizing $q\left(\mathbf{\tilde{w}}\right)$ boils
down to minimizing the squared Mahalanobis norm 
\[
\left(\mathbf{\tilde{w}}-\mathbf{\tilde{w}}\left(0\right)\right)^{\top}\mathbf{B}\left(\mathbf{\tilde{w}}-\mathbf{\tilde{w}}\left(0\right)\right)
\]
 where
\begin{align}
\mathbf{B}=\mathbf{I}-\frac{\left(1-s\right)^{2}}{2\left(1+s^{2}\right)}\mathbf{u}\mathbf{u}^{\top}~.
\label{B_mat}
\end{align}

Note that $\mathbf{B}^{-1}$ is related to the NTK at initialization, since it is easy to verify that
\[
\mathbf{B}^{-1}=\frac{1}{a(0)^2}\left(a(0)^2\mathbf{I}+\vw(0)\vw(0)^{\top}\right)~,
\]
and the NTK at initialization is given by
\[
K(\mathbf{x},\mathbf{x}')=\mathbf{x}^{\top}\left(a(0)^2\mathbf{I}+\vw(0)\vw(0)^{\top}\right)\mathbf{x}'=a(0)^2\left(\mathbf{x}^{\top}\mathbf{B}^{-1}\mathbf{x}'\right)~.
\]
More specifically, using Lemma~\ref{lemma: rkhs_norm}, we can see that $q\left(\mathbf{\tilde{w}}\right)$ is the RKHS norm with respect to the NTK at initialization. 

Next, we discuss the cases when condition \eqref{taylor_cond} holds.

\subsubsection{The case $\alpha\rightarrow\infty$ for any $0\leq s<1$}

In this case \eqref{taylor_cond} holds and thus the implicit bias is given by 
\[
\mathbf{\tilde{w}}\left(\infty\right)=\argmin_{\tilde{\vw}}\left(q_{\delta}\left(\mathbf{\tilde{w}}\right)\right)=\argmin_{\tilde{\vw}}\left(\left(\mathbf{\tilde{w}}-\mathbf{\tilde{w}}\left(0\right)\right)^{\top}\mathbf{B}\left(\mathbf{\tilde{w}}-\mathbf{\tilde{w}}\left(0\right)\right)\right)~,
\]
 where $\mathbf{B}$ defined in \eqref{B_mat}.

\subsubsection{The case $s\rightarrow1$ for any $\alpha>0$}
In this case \eqref{taylor_cond} also holds and thus the implicit bias is given by 
\[
\mathbf{\tilde{w}}\left(\infty\right)=\argmin_{\tilde{\vw}}\left(q_{\delta}\left(\mathbf{\tilde{w}}\right)\right)=\argmin_{\tilde{\vw}}\left(\left(\mathbf{\tilde{w}}-\mathbf{\tilde{w}}\left(0\right)\right)^{\top}\mathbf{B}\left(\mathbf{\tilde{w}}-\mathbf{\tilde{w}}\left(0\right)\right)\right)~,
\]
 where $\mathbf{B}$ defined in \eqref{B_mat}. Since $s\rightarrow1$ we get that
$\mathbf{B}\rightarrow\mathbf{I}$ and thus 
\[
\mathbf{\tilde{w}}\left(\infty\right)=\argmin_{\tilde{\vw}}\left(\left\Vert \mathbf{\tilde{w}}-\mathbf{\tilde{w}}\left(0\right)\right\Vert \right)~.
\]

\section{Proof of Theorem~\ref{theorem: Q for single-neuron with leaky relu activation}}
\label{appendix: proof of Q for single-neuron with leaky relu activation}

\begin{definition} (KKT point) \cite{Dutta2013ApproximateKP}
\label{def: KKT point}
Consider the following optimization problem (P) for $\mathbf{x} \in \mathbb{R}^d$
\begin{align*}
\min f(\mathbf{x})  \,\,\,\,
\, \mathrm{s.t.}\,\,\,\, g_n(\mathbf{x}) \leq 0\,\,\, \forall n \in [N]
\end{align*}
where $f, g_n: \mathbb{R}^d \xrightarrow{} \mathbb{R}$ are locally Lipschitz functions. We say $\mathbf{x} \in \mathbb{R}^d$ is a feasible point of (P) if $g_n(\mathbf{x}) \leq 0\,\,\, \forall n \in [N]$.
Further, a feasible point $\mathbf{x} \in \mathbb{R}^d$ is a KKT point if $\mathbf{x}$ satisfies the KKT conditions:
\begin{align*}
    \exists\, \nu^{(1)},..., \nu^{(N)} &\geq 0 \,\,\, \mathrm{s.t.}\\
    &1.\,\, 0 \in \partial^o f(\mathbf{x}) + \sum_{n \in [N]}\nu^{(n)}\partial^o g_n(\mathbf{x}) \\
    &2.\,\, \forall n \in [N]: \,\, \nu^{(n)} g_n(\mathbf{x}) = 0
\end{align*}
where $\partial^o$ is the local (Clarke’s) sub-differential.
\end{definition}

We follow the lines of the proof for Theorem~\ref{theorem: Q for single-neuron linear network} given in Appendix~\ref{appendix: proof of Q for single-neuron linear network}.

As we do in Appendix~\ref{appendix: proof of Q for single-neuron linear network}, we start by examining a general multi-neuron fully connected network of depth $2$, reducing our claim at the end to the case of a network with a single hidden neuron ($m = 1$).

The fully connected depth $2$ network with Leaky ReLU activations is defined as
\[
f(\mathbf{x}^{(n)};\{a_i\},\{\vw_i\}) = \sum_{i}a_{i}\sigma\left(\mathbf{w}_{i}^{\top}\mathbf{x}^{\left(n\right)}\right)~,
\]
where $\sigma$ is a leaky ReLU with parameter $\rho$,
\[
\sigma\left(\mathbf{w}_{i}^{\top}\mathbf{x}^{\left(n\right)}\right)=\left(\left(1-\rho\right)I\left[\mathbf{w}_{i}^{\top}\mathbf{x}^{\left(n\right)}>0\right]+\rho\right)\mathbf{w}_{i}^{\top}\mathbf{x}^{\left(n\right)}.
\]
The sub-gradient of $\sigma$ is
\[
c_{i}^{\left(n\right)}\left(t\right)=\begin{cases}
1 & \mathbf{w}_{i}^{\top}\mathbf{x}^{\left(n\right)}>0\\
\left[\rho,1\right] & \mathbf{w}_{i}^{\top}\mathbf{x}^{\left(n\right)}=0\\
\rho & \mathbf{w}_{i}^{\top}\mathbf{x}^{\left(n\right)}<0
\end{cases}
\]

The gradient inclusion parameter dynamics are
\[
\dot{a}_{i}\in-\partial_{a_{i}}\mathcal{L}=\mathbf{w}_{i}^{\top}\left(\sum_{n=1}^{N}\mathbf{x}^{\left(n\right)}c_{i}^{\left(n\right)}r^{\left(n\right)}\right)
\]
\begin{equation}
\label{eq: w dynamics for single-neuron with leaky relu}
\dot{\mathbf{w}}_{i}\in-\partial_{\mathbf{w}_{i}}\mathcal{L}=a_{i}\left(\sum_{n=1}^{N}\mathbf{x}^{\left(n\right)}c_{i}^{\left(n\right)}r^{\left(n\right)}\right)
\end{equation}

where we denote the residual
\[
r^{\left(n\right)}(t) \triangleq y^{(n)} - \sum_{i}a_{i}\sigma\left(\mathbf{w}_{i}^{\top}\mathbf{x}^{\left(n\right)}\right)~.
\]

Defining $\tilde{\mathbf{w}}_{i}\triangleq a_{i}\mathbf{w}_{i}$ we have
\[
\frac{d}{dt}\tilde{\mathbf{w}}_{i}\in\dot{a}_{i}\mathbf{w}_{i}+a_{i}\dot{\mathbf{w}}_{i}=\left(a_{i}^{2}\boldsymbol{I}+\mathbf{w}_{i}\mathbf{w}_{i}^{\top}\right)\left(\sum_{n=1}^{N}\mathbf{x}^{\left(n\right)}c_{i}^{\left(n\right)}\left(t\right)r^{\left(n\right)}(t)\right)~.
\]

Using Theorem 2.1 of \citet{Du2018AlgorithmicRI} (stated in Section~\ref{sec:multi_neuron}), we can write
\[
\frac{d}{dt}\tilde{\mathbf{w}}_{i}(t)\in\left(\left(\delta_{i}+\left\Vert \mathbf{w}_{i}(t)\right\Vert ^{2}\right)\mathbf{I}+\mathbf{w}_{i}(t)\mathbf{w}_{i}^{\top}(t)\right)\left(\sum_{n=1}^{N}\mathbf{x}^{\left(n\right)}c_{i}^{\left(n\right)}\left(t\right)r^{\left(n\right)}(t)\right)
\]

or
\[
\left(\left(\delta_{i}+\left\Vert \mathbf{w}_{i}(t)\right\Vert ^{2}\right)\mathbf{I}+\mathbf{w}_{i}(t)\mathbf{w}_{i}^{\top}(t)\right)^{-1}\frac{d}{dt}\tilde{\mathbf{w}}_{i}(t)\in\sum_{n=1}^{N}\mathbf{x}^{\left(n\right)}c_{i}^{\left(n\right)}\left(t\right)r^{\left(n\right)}(t)~,
\]
where assuming $\delta_i \geq 0$, a non-zero initialization $\tilde{\vw}(0)=a(0)\vw(0)\neq\mathbf{0}$ and that we converge to zero-loss solution, gives us that the expression $\left(\left(\delta_i+\left\Vert \mathbf{w}_i(t)\right\Vert ^{2}\right)\mathbf{I}+\mathbf{w}_i(t)\mathbf{w}_i^{\top}(t)\right)^{-1}$ exists. 

Using the Sherman Morisson Lemma, we have
\[
\left(\delta_{i}+\left\Vert \mathbf{w}_{i}(t)\right\Vert ^{2}\right)^{-1}\left(\mathbf{I}-\frac{\mathbf{w}_{i}(t)\mathbf{w}_{i}^{\top}(t)}{\left(\delta_{i}+2\left\Vert \mathbf{w}_{i}(t)\right\Vert ^{2}\right)}\right)\frac{d}{dt}\tilde{\mathbf{w}}_{i}(t)\in\sum_{n=1}^{N}\mathbf{x}^{\left(n\right)}c_{i}^{\left(n\right)}\left(t\right)r^{\left(n\right)}(t)
\]
or
\begin{equation}
\left(\delta_{i}+\left\Vert \mathbf{w}_{i}(t)\right\Vert ^{2}\right)^{-1}\left(\mathbf{I}-\frac{\mathbf{\tilde{w}}_{i}(t)\mathbf{\tilde{w}}_{i}^{\top}(t)}{\left(\delta_{i}+\left\Vert \mathbf{w}_{i}(t)\right\Vert ^{2}\right)\left(\delta_{i}+2\left\Vert \mathbf{w}_{i}(t)\right\Vert ^{2}\right)}\right)\frac{d}{dt}\tilde{\mathbf{w}}_{i}(t)\in\sum_{n=1}^{N}\mathbf{x}^{\left(n\right)}c_{i}^{\left(n\right)}\left(t\right)r^{\left(n\right)}(t)~.
\label{eq: hessian dynamics for non-linear fully connected network}
\end{equation}

Next we use the relation proven in Appendix~\ref{appendix: proof of Q for single-neuron linear network},
\[
\left\Vert \mathbf{w}_{i}(t)\right\Vert =\sqrt{\frac{-\delta_{i}}{2}+\sqrt{\frac{\delta_{i}^{2}}{4}+\left\Vert \mathbf{\tilde{w}}_{i}(t)\right\Vert ^{2}}}
\]
and together with Eq.~\ref{eq: hessian dynamics for non-linear fully connected network} we can write
\[
\left(\frac{\delta_{i}}{2}+\sqrt{\frac{\delta_{i}^{2}}{4}+\left\Vert \mathbf{\tilde{w}}_{i}(t)\right\Vert ^{2}}\right)^{-1}\left(\mathbf{I}-\frac{\mathbf{\tilde{w}}_{i}(t)\mathbf{\tilde{w}}_{i}^{\top}(t)}{\left(\frac{\delta_{i}}{2}+\sqrt{\frac{\delta_{i}^{2}}{4}+\left\Vert \mathbf{\tilde{w}}_{i}(t)\right\Vert ^{2}}\right)\sqrt{\delta_{i}^{2}+4\left\Vert \mathbf{\tilde{w}}_{i}(t)\right\Vert ^{2}}}\right)\frac{d}{dt}\tilde{\mathbf{w}}_{i}(t)\in\sum_{n=1}^{N}\mathbf{x}^{\left(n\right)}c_{i}^{\left(n\right)}\left(t\right)r^{\left(n\right)}(t)~.
\]

We follow the "warped IMD" technique for deriving the implicit bias (presented in detail in Section~\ref{sec: A new technique for deriving the implicit bias})
and multiply the equation by some function $g\left(\mathbf{\tilde{w}}_i(t)\right)$ 
\begin{align*}
    g\left(\mathbf{\tilde{w}}_{i}(t)\right)&\left(\frac{\delta_{i}}{2}+\sqrt{\frac{\delta_{i}^{2}}{4}+\left\Vert \mathbf{\tilde{w}}_{i}(t)\right\Vert ^{2}}\right)^{-1}\left(\mathbf{I}-\frac{\mathbf{\tilde{w}}_{i}(t)\mathbf{\tilde{w}}_{i}^{\top}(t)}{\left(\frac{\delta_{i}}{2}+\sqrt{\frac{\delta_{i}^{2}}{4}+\left\Vert \mathbf{\tilde{w}}_{i}(t)\right\Vert ^{2}}\right)\sqrt{\delta_{i}^{2}+4\left\Vert \mathbf{\tilde{w}}_{i}(t)\right\Vert ^{2}}}\right)\frac{d}{dt}\tilde{\mathbf{w}}_{i}(t)\\
    &\in\sum_{n=1}^{N}\mathbf{x}^{\left(n\right)}c_{i}^{\left(n\right)}\left(t\right)g\left(\mathbf{\tilde{w}}_{i}(t)\right)r^{\left(n\right)}(t)~.
\end{align*}
Following the approach in Section~\ref{sec: A new technique for deriving the implicit bias}, we then try and find $q\left(\mathbf{\tilde{w}}_i(t)\right)=\hat{q}\left(\left\Vert \mathbf{\tilde{w}}_i(t)\right\Vert \right)+\mathbf{z}^{\top}\mathbf{\tilde{w}}_i(t)$ and $g\left(\mathbf{\tilde{w}}_i(t)\right)$ such that 
\begin{align}
\partial^{2} q\left(\mathbf{\tilde{w}}_i(t)\right)=g\left(\mathbf{\tilde{w}}_i(t)\right)\left(\frac{\delta_i}{2}+\sqrt{\frac{\delta_i^{2}}{4}+\left\Vert \mathbf{\tilde{w}}_i(t)\right\Vert ^{2}}\right)^{-1}
\left( \mathbf{I} -\frac{\mathbf{\tilde{w}}_i(t)\mathbf{\tilde{w}}_i^{\top}(t)}
{2\left(\frac{\delta_i}{2}+\sqrt{\frac{\delta_{i}^{2}}{4}+\left\Vert \mathbf{\tilde{w}}_i(t)\right\Vert ^{2}}\right)\sqrt{\frac{\delta_i^{2}}{4}+\left\Vert \mathbf{\tilde{w}}_i(t)\right\Vert^{2}}}\right) 
\label{eq: gauged hessian dynamics for non-linear fully connected network}
\end{align}
and $\partial q\left(\mathbf{\tilde{w}}_{i}(0)\right)=0$. We therefore get 
\[
\partial^{2}q\left(\mathbf{\tilde{w}}_{i}(t)\right)\frac{d}{dt}\tilde{\mathbf{w}_{i}}(t)\in\sum_{n=1}^{N}\mathbf{x}^{\left(n\right)}c_{i}^{\left(n\right)}\left(t\right)g\left(\mathbf{\tilde{w}}_{i}(t)\right)r^{\left(n\right)}(t)
\]
so
\[
\frac{d}{dt}\left(\partial q\left(\mathbf{\tilde{w}}_{i}(t)\right)\right)\in\sum_{n=1}^{N}\mathbf{x}^{\left(n\right)}c_{i}^{\left(n\right)}\left(t\right)g\left(\mathbf{\tilde{w}}_{i}(t)\right)r^{\left(n\right)}(t)~.
\]
Integrating this equation, and recalling $\partial q\left(\mathbf{\tilde{w}}_{i}(0)\right)=0\,$, we obtain
\[
\partial q\left(\mathbf{\tilde{w}}_{i}(t)\right)\in\sum_{n=1}^{N}\mathbf{x}^{\left(n\right)}c_{i}^{\left(n\right)}\left(\infty\right)\nu_{i}^{\left(n\right)}\, ,
\]

where we denoted $\nu_{i}^{\left(n\right)}=\int_{0}^{\infty}dt\frac{c_{i}^{\left(n\right)}\left(t\right)}{c_{i}^{\left(n\right)}\left(\infty\right)}g\left(\mathbf{\tilde{w}}_{i}(t)\right)r^{\left(n\right)}(t)$.

We take notice that this is made possible since for a Leaky ReLU slope $\rho > 0$, we have that $c_{i}^{\left(n\right)}\left(\infty\right) > 0$.

Since Eq. \ref{eq: gauged hessian dynamics for non-linear fully connected network} is identical to the Hessian we got in the proof of Theorem~\ref{theorem: Q for single-neuron linear network} (Eq.~\ref{eq: warped Hessian equation for fully connected linear networks}), we end up with the same $q(\tilde{\mathbf{w}}_i)$ function as we describe there.

We define the linear model 
\[
\mathbf{\tilde{w}}^{\top}\left(\infty\right)\tilde{\mathbf{x}}^{\left(n\right)}=\sum_{i}a_{i}\sigma\left(\mathbf{w}_{i}^{\top}\left(\infty\right)\mathbf{x}^{\left(n\right)}\right)=y^{\left(n\right)}~,
\]
where
\[
\tilde{\mathbf{x}}^{\left(n\right)}=\left[c_{1}^{\left(n\right)}\left(\infty\right)\mathbf{x}^{\left(n\right)};c_{2}^{\left(n\right)}\left(\infty\right)\mathbf{x}^{\left(n\right)};\dots;c_{m}^{\left(n\right)}\left(\infty\right)\mathbf{x}^{\left(n\right)}\right]
\]
\[
\mathbf{\tilde{w}}\left(\infty\right)=\left[\mathbf{\tilde{w}}_{1}\left(\infty\right);\dots;\mathbf{\tilde{w}}_{m}\left(\infty\right)\right]~.
\]
So we have
\[
\partial q\left(\mathbf{\tilde{w}}_{i}(t)\right)\in\sum_{n=1}^{N}\tilde{\mathbf{x}}^{\left(n\right)}\nu_{i}^{\left(n\right)}~.
\]
Finally, for the case of a fully connected network with a single hidden neuron ($m = 1$), the condition
\[
\partial q\left(\mathbf{\tilde{w}}_{i}(t)\right)\in\sum_{n=1}^{N}\tilde{\mathbf{x}}^{\left(n\right)}\nu_{i}^{\left(n\right)}
\]
can be written as
\begin{equation} \label{eq: kkt-nonlinear}
   \partial q\left(\mathbf{\tilde{w}}(t)\right)\in\sum_{n=1}^{N}\tilde{\mathbf{x}}^{\left(n\right)}\nu^{\left(n\right)} 
\end{equation}

which since $\nu^{(n)}$ has no dependency on the index $i$ is a valid KKT stationarity condition for the $q$ we found above (according to definition~\ref{def: KKT point}, where we notice that the second KKT condition of complementary slackness is not needed for regression since we use an equality constraint).

Therefore, the gradient flow satisfies the KKT conditions for minimizing the $q$ we have found.

It follows that we can write
\[
\mathbf{\tilde{w}}\left(\infty\right) = \argmin_{\mathbf{\tilde{w}}}q(\mathbf{\tilde{w}})\,\,\, \text{s.t.} \,\,\, \mathbf{X}^\top\mathbf{\tilde{w}} = \mathbf{y}~.
\]

Additionally, from Eq. \ref{eq: kkt-nonlinear}, using the chain rule, we get

\[
\partial_{\mathbf{w}_{i}(\infty)}q_{\delta_{i}}\left(a_{i}(\infty)\mathbf{{w}}_{i}(\infty)\right)=\left(\partial_{\mathbf{w}_{i}(\infty)}\mathbf{\tilde{w}}_{i}(\infty)\right)\partial_{\mathbf{\tilde{w}}_{i}(\infty))}q_{\delta_{i}}\left(\mathbf{\tilde{w}}_{i}(\infty)\right)\in a_{i}(\infty)\sum_{n=1}^{N}\mathbf{x}^{\left(n\right)}c_{i}^{\left(n\right)}\left(\infty\right)\nu^{\left(n\right)}
\]

\[
\partial_{a_{i}(\infty)}q_{\delta_{i}}\left(a_{i}(\infty)\mathbf{{w}}_{i}(\infty)\right)=\left(\partial_{a_{i}(\infty)}\mathbf{\tilde{w}}_{i}(\infty)\right)^{\top}\partial_{\mathbf{\tilde{w}}_{i}(\infty))}q_{\delta_{i}}\left(\mathbf{\tilde{w}}_{i}(\infty)\right)\in\mathbf{w}_{i}(\infty)^{\top}\sum_{n=1}^{N}\mathbf{x}^{\left(n\right)}c_{i}^{\left(n\right)}\left(\infty\right)\nu^{\left(n\right)}~,
\]
which, together with the feasability of the solution are exactly the KKT conditions of this (non-convex, non-smooth) optimization problem
\[
 \left(a(\infty), \mathbf{w}(\infty)\right) = \argmin_{a,\mathbf{w}} q_{\delta}(a\mathbf{w}) \quad \mathrm{s.t.\,\,\,\,}
 a\sigma(\mathbf{X}^{\top}\mathbf{w}) = \mathbf{y} \,.
\]

\section{Auxiliary Lemmas}
\begin{lemma}
$\delta = a^{2}\left(0\right) - \left\Vert \mathbf{w}\left(0\right)\right\Vert ^{2} = \frac{4\alpha s}{1 - s^2}$ .
\label{lemma: delta_i relation}
\end{lemma}
\begin{proof}
By the notation 
\[
\alpha = |a(0)| \cdot \|\mathbf{w}(0)\| 
\]
\[
s = \frac{|a(0)| - \|\mathbf{w}(0)\|}{|a(0)| + \|\mathbf{w}(0)\|}
\]
we get
\[
1 - s^2 = \frac{4|a(0)|\|\mathbf{w}(0)\|}{(|a(0)| + \|\mathbf{w}(0)\|)^2} 
\]
and
\begin{align*}
 \frac{4\alpha s}{1 - s^2} &=  4\alpha \frac{|a(0)| - \|\mathbf{w}(0)\|}{|a(0)| + \|\mathbf{w}(0)\|}\frac{(|a(0)| + \|\mathbf{w}(0)\|)^2}{4\alpha} \\
 &= a^{2}\left(0\right) - \left\Vert \mathbf{w}\left(0\right)\right\Vert ^{2}  = \delta~.
\end{align*}
\end{proof}

\begin{lemma}
\label{lemma: alpha, s, delta relations}
The initialization scale $\alpha$, initialization shape $s$ and the balancedness factor $\delta$ satisfy:
\begin{enumerate}
\item
\[
\sqrt{\alpha^{2}+\frac{\delta^{2}}{4}}=\frac{\alpha\left(1+s^{2}\right)}{1-s^{2}}
\]
\item
\[
\sqrt{\alpha^{2}+\frac{\delta^{2}}{4}}-\frac{\delta}{2}=\alpha\frac{1-s}{1+s}
\]
\item
\[
\sqrt{\alpha^{2}+\frac{\delta^{2}}{4}}+\frac{\delta}{2}=\alpha\frac{1+s}{1-s}
\]
\end{enumerate}
\end{lemma}

\begin{proof}

\begin{enumerate}
\item
Using Lemma~\ref{lemma: delta_i relation} we get 
\begin{align*}
\sqrt{\alpha^{2}+\frac{\delta^{2}}{4}}&=\sqrt{\alpha^{2}+\frac{4\alpha^2s^2}{(1 - s^2)^2}}
 =\frac{\alpha}{1 - s^2}\sqrt{(1 - s^2)^2 + 4s^2}
 =\frac{\alpha\left(1+s^{2}\right)}{1-s^{2}}~.
\end{align*}
\item
Using part 1 and Lemma~\ref{lemma: delta_i relation} we get
\[
\sqrt{\alpha^{2}+\frac{\delta^{2}}{4}}-\frac{\delta}{2}=\frac{\alpha\left(1+s^{2}\right)}{1-s^{2}}-\frac{2\alpha s}{1-s^{2}}=\alpha\frac{\left(1-s\right)^{2}}{1-s^{2}}=\alpha\frac{1-s}{1+s}~.
\]
\item
Using part 1 and Lemma~\ref{lemma: delta_i relation} we get
\[
\sqrt{\alpha^{2}+\frac{\delta^{2}}{4}}+\frac{\delta}{2}=\frac{\alpha\left(1+s^{2}\right)}{1-s^{2}}+\frac{2\alpha s}{1-s^{2}}=\alpha\frac{\left(1+s\right)^{2}}{1-s^{2}}=\alpha\frac{1+s}{1-s}~.
\]
\end{enumerate}
\end{proof}

\begin{lemma}
\label{lemma: lim g_hat at x=0}
Let 
\[
\hat{g}(x) = \frac{\sqrt{\sqrt{x^{2}+\frac{\delta^{2}}{4}}-\frac{\delta}{2}}}{x }\left(\frac{\delta}{2}+\sqrt{\frac{\delta^{2}}{4}+x^{2}}\right)
\]
be defined $\forall x > 0$, and $\forall \delta \geq 0$.
Then:
\[
\lim_{x \rightarrow 0^+} \hat{g}(x) = 0 ~.
\]
\end{lemma}
\begin{proof}
\[
\lim_{x \rightarrow 0^+} \hat{g}(x) = \lim_{x \rightarrow 0^+}\delta \frac{\sqrt{\sqrt{x^{2}+\frac{\delta^{2}}{4}}-\frac{\delta}{2}}}{x } = \lim_{x \rightarrow 0^+}\delta \sqrt{\frac{\sqrt{x^{2}+\frac{\delta^{2}}{4}}-\frac{\delta}{2}}{x^2}}~.
\]
Using L'Hopital's rule we have
\[
\lim_{x \rightarrow 0^+} \frac{\sqrt{x^{2}+\frac{\delta^{2}}{4}}-\frac{\delta}{2}}{x^2} = \lim_{x \rightarrow 0^+} \frac{\frac{x}{\sqrt{x^{2}+\frac{\delta^{2}}{4}}}}{2x} = \lim_{x \rightarrow 0^+} \frac{1}{2\sqrt{x^{2}+\frac{\delta^{2}}{4}}} = \frac{1}{\delta}~,
\]

and so
\[
\lim_{x \rightarrow 0^+} \hat{g}(x)  = \lim_{x \rightarrow 0^+}\delta \sqrt{\frac{1}{\delta}} = \sqrt{\delta}~.
\]
\end{proof}

\begin{lemma}
\label{lemma: rkhs_norm}
Let $\mathbf{A}$ be a positive definite matrix and $f\left(\mathbf{x}\right)$
a kernel predictor corresponding to a linear kernel $K\left(\mathbf{x},\mathbf{x}'\right)=\mathbf{x}^{\top}\mathbf{A}\mathbf{x}'$.
Then
\[
\left\Vert f\right\Vert _{K}^{2}=\mathbf{w}^{\top}\mathbf{A}^{-1}\mathbf{w}~,
\]
 where $f\left(\mathbf{x}\right)=\mathbf{w}^{\top}\mathbf{x}$. 
\end{lemma}

\begin{proof}
Write $K\left(\mathbf{x},\mathbf{x}'\right)=\mathbf{x}^{\top}\mathbf{A}\mathbf{x}'=\mathbf{x}^{\top}\mathbf{A}^{\frac{1}{2}}\mathbf{A}^{\frac{1}{2}}\mathbf{x}'$,
then $\phi\left(\mathbf{x}\right)=\mathbf{A}^{\frac{1}{2}}\mathbf{x}$
is the corresponding feature mapping and 
\[
f\left(\mathbf{x}\right)=\tilde{\mathbf{w}}^{\top}\phi\left(\mathbf{x}\right)=\tilde{\mathbf{w}}^{\top}\mathbf{A}^{\frac{1}{2}}\mathbf{x}=\mathbf{w}^{\top}\mathbf{x}
\]

for $\mathbf{w}=\mathbf{A}^{\frac{1}{2}}\tilde{\mathbf{w}}$. Therefore
\[
\left\Vert f\right\Vert _{K}^{2}=\left\Vert \tilde{\mathbf{w}}\right\Vert ^{2}=\left\Vert \mathbf{A}^{-\frac{1}{2}}\mathbf{w}\right\Vert ^{2}=\mathbf{w}^{\top}\mathbf{A}^{-1}\mathbf{w}~.
\]
\end{proof}


\end{document}